\documentclass[12pt]{article}

\usepackage[margin=2.5cm]{geometry}
\usepackage{parskip}

\usepackage{xcolor, calc, blindtext}
\usepackage{wrapfig}
\usepackage{hyphenat}
\usepackage{booktabs}

\usepackage{amsmath, amssymb, amsfonts, mathtools, bm, bbm, dsfont, physics, tensor, mathrsfs, mathdots, stmaryrd}
\usepackage{mathtools}
\usepackage{tensor}
\usepackage{algorithm}

\usepackage{amsthm}
\theoremstyle{definition}
\newtheorem{definition}{Definition}[section]
\newtheorem*{remark}{Remark}
\newtheorem{theorem}{Theorem}[section]
\newtheorem{corollary}{Corollary}[theorem]

\newtheorem{proposition}{Proposition}[section]
\newtheorem{example}{Example}[section]

\usepackage{tcolorbox}
\colorlet{shadecolor}{gray!15}

\renewenvironment{abstract}
  {\begin{center}\bfseries Abstract\end{center}%
   \leftskip=0pt\rightskip=0pt\noindent}
  {\par\vspace{1em}}

\numberwithin{equation}{section}

\usepackage{graphicx}
\usepackage{wrapfig}
\usepackage{float}
\usepackage{booktabs}
\usepackage{pdflscape}
\usepackage{lscape}
\usepackage{tikz}
\usepackage{tikz-3dplot}
\usepackage{tikz-cd}
\usetikzlibrary{positioning, fit}
\usepackage{pgfplots}
\pgfplotsset{compat=newest}
\tdplotsetmaincoords{60}{115}
\usetikzlibrary{arrows.meta, positioning, fit, backgrounds}

\usepackage{array, tabularx, multirow, makecell, colortbl, longtable}
\usepackage[table]{xcolor}
\usepackage{enumitem}

\usepackage[sort&compress,numbers]{natbib}
\usepackage{hyperref}
\usepackage{url}

\usepackage{authblk}
\usepackage{caption, subcaption}
\usepackage{comment}
\usepackage{placeins}
\usepackage{setspace}
\usepackage[all]{xy}
\usepackage{framed} %

\DeclareFontFamily{T1}{calligra}{}
\DeclareFontShape{T1}{calligra}{m}{n}{<->s*[1.44]callig15}{}
\DeclareMathAlphabet\mathcalligra   {T1}{calligra} {m} {n}
\DeclareMathAlphabet\mathzapf       {T1}{pzc} {mb} {it}
\DeclareMathAlphabet\mathchorus     {T1}{qzc} {m} {n}
\DeclareMathAlphabet\mathrsfso      {U}{rsfso}{m}{n}

\usepackage{algpseudocode}

\definecolor{myLightgreen}{RGB}{231, 239, 228} %

\definecolor{myLightgrey}{RGB}{237, 237, 237} %

\begin{document}
\begin{center}  

\vspace*{0.2cm}
{\LARGE \bfseries A Machine Learning Approach \\ to the Nirenberg Problem\par}
\vspace{1cm}

{
\textbf{Gianfranco Cort\'es}$^{1},$ \hspace{0.3pt}
 \textbf{Maria Esteban-Casadevall}$^{2},$ \hspace{0.3pt} \textbf{Yueqing Feng}$^{3},$ \hspace{0.3pt} 
\textbf{Jonas Henkel}$^{4},$ \hspace{0.3pt} 
\textbf{Edward Hirst}$^{5},$ \hspace{0.3pt}
\textbf{Tancredi Schettini Gherardini}$^{\, 6,7},$ \hspace{0.3pt}
\textbf{Alexander G. Stapleton}$^{8}$\par
}
\vspace{1cm} 
{\small
\parbox{\textwidth}{%
\raggedright
\qquad\qquad $^{1}$ University of Florida, USA\\
\qquad\qquad $^{2}$ University of Amsterdam, Amsterdam, Netherlands\\
\qquad\qquad $^{3}$ University of California, Berkeley, USA\\
\qquad\qquad $^{4}$ Philipps-Universit\"at Marburg, Marburg, Germany\\
\qquad\qquad $^{5}$ Universidade Estadual de Campinas (Unicamp), Campinas, Brazil\\
\qquad\qquad $^{6}$ Mathematisches Institut (MI), University of Bonn, Bonn, Germany\\
\qquad\qquad $^{7}$ Max Planck Institute for Mathematics, Bonn, Germany \\
\qquad\qquad $^{8}$ Centre for Theoretical Physics, Queen Mary University of London, London, UK
}}

\vspace{1cm} 

\end{center}

\begingroup
\renewcommand{\thefootnote}{}  %
\footnotetext{%
\parbox{\textwidth}{%
\raggedright
$^{1}$\texttt{gcortes@ufl.edu}, 
$^{2}$\texttt{m.estebancasadevall@uva.nl},
$^{3}$\texttt{fyq@berkeley.edu},
$^{4}$\texttt{jonas.henkel@uni-marburg.de},
$^{5}$\texttt{ehirst@unicamp.br},
$^{6}$\texttt{tsg@math.uni-bonn.de},
$^{7}$\texttt{a.g.stapleton@qmul.ac.uk}.\\
Authors are listed in alphabetical order by surname.
}}
\endgroup

\begin{abstract}

This work introduces the Nirenberg Neural Network: a numerical approach to the Nirenberg problem of prescribing Gaussian curvature on $S^2$ for metrics that are pointwise conformal to the round metric. Our mesh-free physics-informed neural network (PINN) approach directly parametrises the conformal factor globally and is trained with a geometry-aware loss enforcing the curvature equation. Additional consistency checks were performed via the Gauss–Bonnet theorem, and spherical-harmonic expansions were fit to the learnt models to provide interpretability.

For prescribed curvatures with known realisability, the neural network achieves very low losses ($10^{-7} - 10^{-10}$), while unrealisable curvatures yield significantly higher losses. This distinction enables the assessment of unknown cases, separating likely realisable functions from non-realisable ones. The current capabilities of the Nirenberg Neural Network demonstrate that neural solvers can serve as exploratory tools in geometric analysis, offering a quantitative computational perspective on longstanding existence questions.

\end{abstract}

\newpage
\tableofcontents
\newpage

\section{Introduction}

A long-standing problem posed by L. Nirenberg \cite{Nirenberg1953} asks the following: \textit{Given a smooth function $K : S^2 \to \mathbb{R}$, does there exist a Riemannian metric on $S^2$, pointwise conformal to the round metric, whose Gaussian curvature is $K$?}

Let $(S^2, g_0)$ denote the 2-sphere with the round metric $g_0$, whose Gaussian curvature is normalised to $K_{g_0} \equiv 1$. Let $g$ be a metric that is pointwise conformal to $g_0$, written as $g=e^{2u}g_0$, for some $u \in C^\infty(S^2)$. Given $K \in C^\infty(S^2)$, the Nirenberg problem is equivalent to solving the following second-order partial differential equation (PDE)
\begin{equation} \label{eq:main_pde}
     1-\Delta_{g_0} u  = K e^{2u}, 
\end{equation}
where $\Delta_{g_0}$ denotes the Laplace-Beltrami operator with respect to $g_0$. Throughout, $K$ is referred to as a prescribed Gaussian curvature or simply as a prescriber, and Equation~\eqref{eq:main_pde} is referred to as the Nirenberg equation.

Regarding the problems of prescribing Gaussian curvature on $S^2$, there are three natural questions to ask (as discussed in \cite{KW1974curvature} and \cite{KW1975scalar}), namely whether $K$ is the Gaussian curvature of (a) a Riemannian metric, (b) a Riemannian metric conformally equivalent to a prescribed metric $g_0$, (c) a Riemannian metric pointwise conformal to a prescribed metric $g_0$. Two metrics $g$ and $g'$ are said to be conformally equivalent if there is a diffeomorphism $\varphi$ of $S^2$ and a function $u\in C^{\infty}(S^2)$ such that $\varphi(g')=e^{2u}g$. 
Note that (c) is a special case of (b) where $\varphi$ is the identity map, and (b) is a special case of (a) where there is restriction to only one conformal class, such that $(c)\ \Rightarrow\ (b)\ \Rightarrow\ (a)$. 
On $S^2$, by the Uniformization Theorem, $(a)\iff(b)$, with a necessary and sufficient condition provided by \cite{kazdan1975existence}.
However the Nirenberg problem (c) remains incomplete, with lots of research towards its resolution, including \cite{KW1974curvature, KW1975scalar, ChenLi1995KW, Moser1973, https://doi.org/10.1002/cpa.3160480606, ChangYang1987, chang1988conformal, ji2009scalar, struwe2005flow, struwe2023bubbling, xu1993remarks, li2022sigma2, kazdan1975existence, anderson2021nirenberg, han1990prescribing, changliu1993nirenberg, chang1993scalar,ding1995note}. 

Beyond classical obstructions, such as the Gauss-Bonnet formula and the Kazdan-Warner identity, many works impose additional symmetry hypotheses on $K$ to derive explicit sufficient criteria for solvability. 
In particular, for rotationally symmetric functions $K$, both necessary and sufficient conditions are found. 
For general cases with no additional symmetry imposed, most of the work is based on the variational approach proposed by Moser \cite{Moser1973}. 
Equation~\eqref{eq:main_pde} is viewed as the Euler-Lagrange equation of the functional
$$J(u) = \int_{S^2}(|\nabla u|^2+2u)\;d\nu_0 -\log\int_{S^2}Ke^{2u}\;d\nu_0, \quad u\in H^1(S^2).$$
Note that the Palais-Smale condition fails in the case where $K$ is non-constant. Chang-Yang \cite{ChangYang1987, chang1988conformal} provided a set of sufficient conditions through the minimax scheme, followed by the Morse and degree approach being further developed in \cite{han1990prescribing, changliu1993nirenberg, chang1993scalar}, and later Struwe proposing a technique with the prescribed curvature flow \cite{struwe2005flow}. 
Beyond variational methods, an alternative perspective \cite{anderson2021nirenberg} is to study the problem as the image of the curvature map from the space of potential solutions $u$ to the space of potential curvature functions $K_{g_u}$. 
In that setting, the structure of regular and singular points of the map is studied using Cheeger-Gromov theory. Despite these sufficient criteria, a complete understanding is still lacking, and the problem remains open.

A difficulty in solving Equation~\eqref{eq:main_pde} stems from the loss of compactness caused by the non-compact M\"obius group $\mathrm{Conf}(S^{2})$, where variational sequences may concentrate and develop bubbling phenomena \cite{ChangYang1987,han1990prescribing,ding1995note,struwe2023bubbling,anderson2021nirenberg}. 
Existence proofs obtained by variational methods provide limited information on the explicit behaviour of solutions. It is in contexts of this kind, where PDE solutions escape analytic control, that (traditional) numerical methods have been developed and refined \cite{TrefethenSpectralMATLAB, HungriaLessardMireles, DayLessardMischaikow, MooreInterval, NakaoPlumWatanabe, RUUTH20081943,Dziuk2013FiniteEM}. 
Coupling these numerical schemes with computer assisted proofs has led to significant and rigorous new results in the study of PDEs, many of which are discussed in the review \cite{GomezSerranoSurvey} and references therein. 

Modern day numerical methods now include developed statistical techniques arising from the field of \textit{Machine Learning} (ML). 
The ubiquitous successes of the Neural Network (NN) architecture, aided by their various universal approximation theorems, inspired their application in tackling difficult PDE problems \cite{raissi2017physicsinformeddeeplearning, E_2017}, giving birth to Physics Informed Neural Networks (PINNs)\footnote{This is, of course, rooted by the pioneering works \cite{LeeKang1990, DissanayakePhanThien1994, LagarisLikasFotiadis1998}, twenty years ahead of their time.}.
These powerful tools have been applied in a variety of contexts, such as minimal surfaces \cite{DebSanghavi2025HolographicPINN, ZhouYe2023MinimalSurfacePINN}, Calabi-Yau metrics \cite{Ashmore:2019wzb, Douglas:2020hpv, Anderson:2020hux, Jejjala:2020wcc, Larfors:2021pbb, Gerdes:2022nzr, Berglund:2022gvm}, embedded surfaces \cite{fang2021physicsinformedneuralnetworkframework}, and blow-up solutions \cite{wang2023asymptoticselfsimilarblowupprofile, wang2025discoveryunstablesingularities}. 
While most of the aforementioned works lie in the realm of \textit{ML for discovery}, the last two references in particular hint at the possible integration of PINNs with computer-assisted proof techniques. 
This aims to couple the modelling capacity of the neural networks with the rigour required by mathematical proofs, providing a very exciting and novel research direction (see \cite{gomez2019computer, tanaka2026learnverifyframeworkrigorous}, for instance), which is intended to be incorporated into future developments of this work. 
On a methodological level, this aligns with the framework proposed in \cite{Henkel2025assistant}, exploring how AI tools can be responsibly integrated into the daily practice of mathematical research. 

This paper presents the first investigation of the Nirenberg problem through the use of PINNs, with the following results: 
\begin{itemize}
    \item Numerical solutions to Equation~\eqref{eq:main_pde} are obtained, some of which are only a few orders of magnitude higher than the precision needed for rigorous mathematical validation via computer-assisted proofs\footnote{For the trivial case of the round metric, numerical precision is achieved.}.
    \item For the set of realisable and non-realisable prescribed curvatures considered, one is able to observe an empirical separation in both the attained loss and the Gauss–Bonnet consistency check, which corroborates the reliability of the method.
    \item For functions whose realisability is still an open question, the results provide quantitative evidence towards a conjectured answer.
    \item For some prescribers, one is able to recover accurate closed-form approximations by fitting a spherical-harmonic expansion, yielding an explicit analytic expression derived from the trained network.
\end{itemize}

The structure of the paper is as follows. In Section~\ref{sec: background}, an introduction to the problem is provided, together with a review of the progress that has been made toward its solution. In Section~\ref{sec:Nirenberg_impl}, a description of the neural network architecture and the numerical tools employed in the investigation is given.
The results are presented in detail in Section~\ref{sec: results}, and Section~\ref{sec: conclusions} discusses the conclusions and some natural directions for future work. Finally, further insights into the architecture of the Nirenberg Neural Network and additional results and visualisations are provided in the appendices.

\section{Background}\label{sec: background}

\subsection{Geometric Constraints on Curvature}
\label{sub:geom_constr_on_curv}

There are two well-known obstructions for the existence of solutions to Equation~\eqref{eq:main_pde}, the first comes from integrating Equation~\eqref{eq:main_pde} over the sphere, which immediately yields the Gauss-Bonnet constraint. 
From the fact that $\int_{S^2} \Delta_{g_0} u \, dA_{g_0} = 0$,
\begin{equation} \label{eq:gauss_bonnet}
    \int_{S^2} K e^{2u} \, dA_{g_0} = \int_{S^2} 1 \, dA_{g_0} = 4\pi.
\end{equation}
From this integral constraint, a fundamental necessary condition follows immediately:
\begin{itemize}
    \item \textbf{Positivity:} The function $K$ must be positive somewhere on $S^2$ (i.e. $\max_{x \in S^2} K(x) > 0$). If $K$ were non-positive everywhere, the integral on the left-hand side of Equation~\eqref{eq:gauss_bonnet} would be non-positive, contradicting the value of $4\pi$.
\end{itemize}

Another obstruction was identified by Kazdan and Warner \cite{KW1974curvature}. 
By considering the interaction of the curvature gradient with conformal Killing vector fields, they derived the following necessary condition:
\begin{theorem}[Kazdan-Warner Identity]\label{theo: Kazdan-Warner Identity}
    If $u$ is a solution to Equation~\eqref{eq:main_pde}, then for any first spherical harmonic $F$ (an eigenfunction satisfying $\Delta_{g_0} F = 2F$), the following identity holds:
    \begin{equation} \label{eq:KW_identity}
        \int_{S^2} \langle \nabla K, \nabla F \rangle_{g_0} \, e^{2u} \, dA_{g_0} = 0.
    \end{equation}
\end{theorem}
This condition imposes strong restrictions on the geometry of $K$. 
For instance, let $K(x)$ be a rotationally symmetric function depending only on the polar angle $\theta$ (a zonal function), if $K(\theta)$ is strictly monotonic (e.g. $K(\theta) = 2 + \cos\theta$), $F(\theta)$ can be chosen as $F(\theta) = \cos\theta$. 
Then the scalar product $\langle \nabla K, \nabla F \rangle$ has a fixed sign everywhere on $S^2$ (excluding poles where it vanishes), making the integral in Equation~\eqref{eq:KW_identity} non-zero. 
Thus, strictly monotonic rotationally symmetric functions cannot be realised as Gaussian curvatures in the conformal class of the standard metric \cite{https://doi.org/10.1002/cpa.3160480606}.

\medskip

\subsection{Existence Results}\label{section: existence results}

This section reviews explicit sufficient conditions on $K$ that guarantee the solvability of the Nirenberg equation. These conditions prevent the bubble behaviour and the resulting class of $K$ is used to generate training and validation examples for the package. Unknown cases are also outlined, which are used to construct testing examples.

\paragraph{Antipodally Symmetric Functions}\mbox{}\\
Moser proved \cite{Moser1973} that if $K$ is an even function, i.e. $K(x) = K(-x)$ for all $x \in S^2$, and $K$ is positive somewhere, then a solution to Equation~\eqref{eq:main_pde} always exists. 

\begin{example}[Spherical Harmonics of even degree]\label{example: even spherical ahrmonics}
    The real spherical harmonics $Y_{\ell, m}$ satisfy the symmetry property $Y_{\ell, m}(-\theta) = (-1)^\ell Y_{\ell, m}(\theta)$. Consequently, any harmonic with \textit{even degree} $\ell$ is an even function. Therefore, any linear combination $K(\theta) = \sum_{\ell,m} c_{2\ell, m} Y_{2\ell, m}(\theta)$ which is positive in at least one point automatically admits a solution.
\end{example}

\paragraph{Rotationally Symmetric Functions}\mbox{}\\
For zonal functions $K(\theta)$, Xu and Yang \cite{xu1993remarks} provided a refined criterion.

\begin{theorem}[\cite{xu1993remarks}]\label{theo: rotationally symmetric solutions}
    Let $K$ be rotationally symmetric and non-degenerate (in the sense that $K'' \neq 0$ whenever $K'=0$). The necessary and sufficient condition for solvability is that $K$ is positive somewhere and that $K'(\theta)$ changes sign in the region where $K(\theta) > 0$.
\end{theorem}

This implies that while $K(\theta) = 2 + \cos\theta$ (monotone) is unsolvable, a "camel hump" function or a perturbed harmonic like $K(\theta) = 2 + \cos(2\theta)$ is solvable.

Based on these criteria, the solvability for the specific class of zonal spherical harmonics can be completely classified. They thus serve as benchmark functions in this work.

\begin{proposition}[Classification of Zonal Spherical Harmonics]\label{prop: Classification zonal spherical harmonics}
    Let $K(\theta) = Y_{\ell, 0}(\theta)$ be the zonal spherical harmonic of degree $\ell \ge 1$ with $m=0$.
    \begin{enumerate}
        \item If $\ell = 1$, the function $K$ is strictly monotonic on $(0, \pi)$. Consequently, the Nirenberg problem admits no solution.
        \item If $\ell \ge 2$, the function $K$ admits a solution.
    \end{enumerate}
\end{proposition}

\begin{proof}
    For $\ell=1$, $K(\theta) = c \cos(\theta)$. This function is strictly monotonic in the domain where $K>0$. According to the obstruction results by Chen and Li \cite{https://doi.org/10.1002/cpa.3160480606}, no solution exists.
    
    For $\ell \ge 2$, the sufficient conditions given by Xu and Yang \cite{xu1993remarks} are verified: $K$ must be positive somewhere, $K'$ must change sign in the positive region, and critical points must be non-degenerate.
    Since $Y_{\ell, 0}$ is an eigenfunction of the Laplacian, it satisfies the ODE on $(0, \pi)$:
    \begin{equation}
        K''(\theta) + \cot(\theta) K'(\theta) + \ell(\ell+1)K(\theta) = 0.
    \end{equation}
    The non-degeneracy ($K'' \neq 0$ whenever $K'=0$) is verified in two steps:
    \begin{itemize}
        \item \textbf{Interior:} Let $\theta_0 \in (0, \pi)$ be an interior critical point ($K'(\theta_0) = 0$). The ODE simplifies to $K''(\theta_0) = -\ell(\ell+1)K(\theta_0)$. Since the coefficients of the ODE are smooth in $(0, \pi)$, uniqueness of solutions implies that $K$ cannot vanish simultaneously with its derivative unless $K \equiv 0$. Since $Y_{\ell, 0}$ is not trivial, $K(\theta_0) \neq 0$, and thus $K''(\theta_0) \neq 0$.
        \item \textbf{Boundary:} At the poles (e.g. $\theta \to 0$), the limit of the term $\cot(\theta)K'(\theta)$ is considered. Using l'H\^opital's rule and the smoothness of $K$, $\lim_{\theta \to 0} \frac{K'(\theta)}{\tan(\theta)} = K''(0)$. The ODE at the limit thus yields $2K''(0) = -\ell(\ell+1)K(0)$. Since $Y_{\ell,0}(\theta) \propto P_\ell(\cos \theta)$ and Legendre polynomials satisfy $P_\ell(1)=1$, $K(0) \neq 0$, and consequently $K''(0) \neq 0$.
    \end{itemize}
    Since for $\ell \ge 2$, $Y_{\ell, 0}$ possesses internal extrema, the derivative changes sign, satisfying all existence conditions.
\end{proof}

\paragraph{General Functions without additional symmetry assumptions}\mbox{}\\
The distribution of critical points plays an essential role in the search for sufficient conditions. For the approach rooted in Morse and degree theory, the discussion is restricted to the case where $K$ does not contain degenerate critical points.

\begin{theorem} (Chang-Yang, \cite{ChangYang1987})\label{thm: condition critical points}
    Suppose $K$ has only isolated non-degenerate critical points and
    \begin{equation}\label{eq:non-degenerate}
        |\nabla K| + |\Delta K| \ne 0.
    \end{equation}

    If $\Sigma_{x\in S_-}(-1)^{ind(x)}\ne 1$, where $S_-=\{x\in S^2: \nabla K(x)=0, \Delta K(x) < 0\}$, then Equation  \eqref{eq:main_pde} has a solution. Here $ind(x)$ denotes the Morse index of the critical point $x$.
\end{theorem}

An equivalent form of the theorem is the following (after aligning sign conventions of the Laplacian).

\begin{theorem}[\cite{xu1993remarks}]\label{theo:xu}
    Suppose $K$ satisfies the conditions in the above theorem. Let 
    $$\Omega = \{p: K(p)>0\}.$$
    Let $p+1$ be the number of local maxima in $\Omega$ of $K$, and $q$ the number of saddle points $x$ in $\Omega$ with $\Delta K(x) < 0$. If $p\ne q$, then Equation \eqref{eq:main_pde} has a solution.
\end{theorem}

In particular, a sufficient condition for the above to hold is the following.

\begin{corollary}[\cite{ji2009scalar}]
    Suppose $K$ satisfies the conditions in the above theorem. If 
    \[\Delta K(x)\Delta K(-x) - \nabla K(x)\cdot\nabla K(-x) \ge 0, \forall x\in S^2,\]
    then Equation \eqref{eq:main_pde} has a solution.
\end{corollary}

Other than the above mentioned explicit criterion, structural theorems are developed \cite{anderson2021nirenberg} from the geometric convergence theorems of Cheeger-Gromov. Consider the map 
\[\pi: C^{m, \alpha}(S^2)\rightarrow C^{m-2,\alpha}_{+}(S^2),\]
defined by $u\mapsto K_{g_u}$. Here $C^{m,\alpha}_+(S^2)$ denotes the functions $f$ in $C^{m, \alpha}(S^2)$ such that $f(x)>0$ for some $x\in S^2$. Standard elliptic theory imposes that $\pi$ is smooth, non-linear, Fredholm with index $0$. When restricted in the non-degenerate region, $\pi$ is proper (c.f. Theorem 1.1, \cite{anderson2021nirenberg}), and the notion of degree makes sense. The author provides explicit formulae for computing the degree (c.f. Theorem 1.2 and Theorem 1.3, \cite{anderson2021nirenberg}). When degree is non-zero, this gives a signed count on the number of generic solutions, hence providing solutions to Equation~\eqref{eq:main_pde}. For the degree zero region $\Omega_0$, it's shown to be connected (c.f. Remark 4.2) and analysis of regular and singular points of $\pi$ is conducted.

\begin{theorem}[\cite{anderson2021nirenberg}, Theorem 1.4]
    The regular points are regular values of $\pi$, are open and dense in the domain and range of $\pi$, and the set of singular points of $\pi$ is a stratified space with strata of codimension $s\ge 1$ in $C^{m,\alpha}(S^2)$. 
\end{theorem}

\subsection{Generating Synthetic Train and Test Data}
\label{sec:spectralpair}
For the purpose of training and validating machine learning models, suitable data may be generated considering the inversion of the problem. Instead of prescribing $K$ and then solving for $u$, synthetic pairs $(u, K)$ can be directly generated.

\begin{proposition}[Spectral Pairs]\label{prop: spectral pairs}
    Let $\{Y_{\ell, m}\}$ denote the real-valued spherical harmonics on $S^2$. For any finite set of coefficients $\{c_{\ell, m}\} \subset \mathbb{R}$ with $1 \le \ell \le L$, define the conformal factor $u$ by
    \begin{equation}
        u(x) = \sum_{\ell=1}^{L} \sum_{m=-\ell}^{\ell} \frac{c_{\ell, m}}{\ell(\ell+1)} Y_{\ell, m}(x).
    \end{equation}
    Then $u$ is an exact solution to the Nirenberg Equation~\eqref{eq:main_pde} with prescribed Gaussian curvature $K$ given by
    \begin{equation}
        K(x) = e^{-2u(x)} \left( 1 + \sum_{\ell=1}^{L} \sum_{m=-\ell}^{\ell}  c_{\ell, m} Y_{\ell, m}(x) \right).
    \end{equation}
\end{proposition}%

\begin{remark}
    The Kazdan-Warner identity (Theorem \ref{eq:KW_identity}) provides a powerful tool to construct counterexamples. Specifically, any function $K$ for which the gradient $\nabla K$ has a non-vanishing projection onto a first spherical harmonic $\nabla F$ (e.g. strictly monotonic functions like $K = 2 + x$) cannot be realised as Gaussian curvature. Such functions will be used, along with more subtle cases in Section \ref{sub:results_general_funcs}.
\end{remark}

From the discussion in Section~\ref{section: existence results}, potential unknown examples lie either in the region $\Omega_0$, where $\pi$ has degree zero, or in settings where degenerate points are allowed. For testing purposes, non-rotationally symmetric examples that violate the hypotheses of the currently available existence theorems are constructed. These begin with low-degree polynomial curvatures $K$, and the data for which Theorem~\ref{theo:xu} applies are discarded. Some remaining examples with $p=q$ are listed in Table~\ref{tab:unknown_exist_general}. Test functions for the prescribed curvature $K$ built from spherical harmonics are also considered. While the even-degree cases are understood (c.f. Example~\ref{example: even spherical ahrmonics}), odd-degree cases yield additional examples for testing.

\begin{example}[Spherical harmonics of odd degree]\label{ex:unknown_sh}
Let $Y_{\ell,m}$ be a spherical harmonic of odd degree and $|m|>1$ or $\ell>1$ and $|m|= 1$. As degenerate points are precisely given by the self-intersections of the nodal sets \cite{berardhelffer2014nodal}, it has critical zeros $p\in S^2$, i.e. $\nabla Y_{\ell,m}(p)=\Delta Y_{\ell,m}(p)=0$. The non-degeneracy condition in Theorem \ref{thm: condition critical points} is thus violated.
\end{example}

\section{The Nirenberg Neural Network}
\label{sec:Nirenberg_impl}

This section introduces the Nirenberg Neural Network (NNN) architecture. 
The NNN is a physics-informed neural network (PINN) designed to approximate the solution to the Nirenberg equation, where PINNs are a class of neural architectures designed to learn approximate solutions to PDEs \cite{raissi2017physicsinformeddeeplearning}. 
A neural network can be viewed as a parametrized function \( NN(x; \theta) \). The goal is then to optimize the parameters so that \( \theta \) approaches \( \theta^\star \) where \( NN(x; \theta^\star) \) approximates a solution of the PDE. The network is trained by minimising a loss function that includes boundary conditions, consistency conditions, and other constraints of the PDE. 
This loss minimisation process optimises the network's parameters, conditional on the data, using gradient-based methods such as Adam \cite{kingma2014adam}.

Essential to their successes, neural networks are \emph{universal approximators} \cite{HORNIK1989359, hanin2018approximatingcontinuousfunctionsrelu, kidger2020universalapproximationdeepnarrow}, meaning that, under mild assumptions on the activation function, they can approximate any continuous function on a compact domain to arbitrary accuracy. 
This makes them a natural tool for approximating solutions to PDEs. 
Consequently, this approach has recently attracted significant interest and has been successfully applied in a wide range of settings, including geometric analysis contexts such as the construction of Calabi-Yau, G2, and Einstein metrics, with a lot of other research carried out in this context \cite{Ashmore:2019wzb, larfors2022numerical, larfors2021learning, gerdes2023cyjax, jejjala2022neural, douglas2020numerical, douglas2024harmonic, hirst2024ainstein, deluca2025gravity, haydys2025}.

To motivate the choice of using NNs for studying the Nirenberg problem, recall that the key Equation~\eqref{eq:main_pde} is a nonlinear elliptic PDE \cite{KW1974curvature}. 
Even in the rotationally symmetric case, determining the conformal factor $u$ for a given curvature $K$ involves solving a nonlinear ODE which rarely admits closed-form solutions. For a general non-symmetric prescribed curvature $K$, a key analytical difficulty is the potential loss of compactness. 
Explicit formulae are not expected in general, and solvability is subtle. 
The curvature map $u \mapsto K_{e^{2u}g_{0}}$ modulo the conformal group $\mathrm{Conf}(S^{2})$ is usually studied using non-linear elliptic theory, degree-theoretic ideas, and variational methods. For these reasons, a data-driven approach approximating the solution map (finding $u$ for a given $K$) provides a novel and promising perspective. 
Although this is a natural direction, numerical approaches to the Nirenberg problem remain relatively underutilised.

The description of the Nirenberg architecture begins with the input sample generation method (Subsection \ref{sec: input}); architecture details (Subsection \ref{sec:architecture}); the training objective (Subsection \ref{sec:loss}); and ends with output diagnosis (Subsection \ref{sec: visualisation}).

\subsection{Input Structure and Sampling Strategy} \label{sec: input}
The sphere is represented using standard \textit{stereographic projection}, which provides a smooth atlas with two coordinate patches (each one covering the whole sphere except a point). In the Nirenberg architecture, both patches are used to encourage geometric consistency and stabilize training. In this setting, the northern hemisphere (points with positive $z$-coordinate) is projected from the south pole onto a disc, while the southern hemisphere (points with negative $z$-coordinate) is projected from the north pole onto a separate disc (see Figure \ref{Fig: stereographic_maps}). This construction ensures that each hemisphere, including an overlap region, is covered by a single chart. By making the overlap far from maximal, singularities in the coordinate representation (as observed in \cite{hirst2024ainstein}) are avoided.

The stereographic projection maps points on the sphere to the Euclidean plane via a conformal transformation:
\begin{equation} \label{eq: stereographic projection}
\pi_0 : S^2 \setminus \{S\} \to \mathbb{R}^2, \quad
(x, y, z) \mapsto (u, v) = \left( \frac{x}{1+z}, \frac{y}{1+z} \right),
\end{equation}
where $S=(0,0,-1)$ is the South pole. The inverse map
\begin{equation}
\pi_0^{-1} : \mathbb{R}^2 \to S^2, \quad
(u, v) \mapsto (x, y, z) = \left( \frac{2u}{1+r^2}, \frac{2v}{1+r^2}, \frac{1-r^2}{1+r^2} \right),
\end{equation}
with \(r^2 = u^2 + v^2\).  An equivalent definition applies for maps in the northern hemisphere.  

In these coordinates, the round metric takes the explicit form
\begin{equation}
g_0 = \frac{4}{(1+r^2)^2} I,
\end{equation}
where $I$ is the identity matrix. This conformal metric has the crucial property that its scalar curvature is constant and equal to 2 everywhere on the sphere.

\begin{figure} [t!]
  \centering
  \includegraphics[scale=0.27, clip]{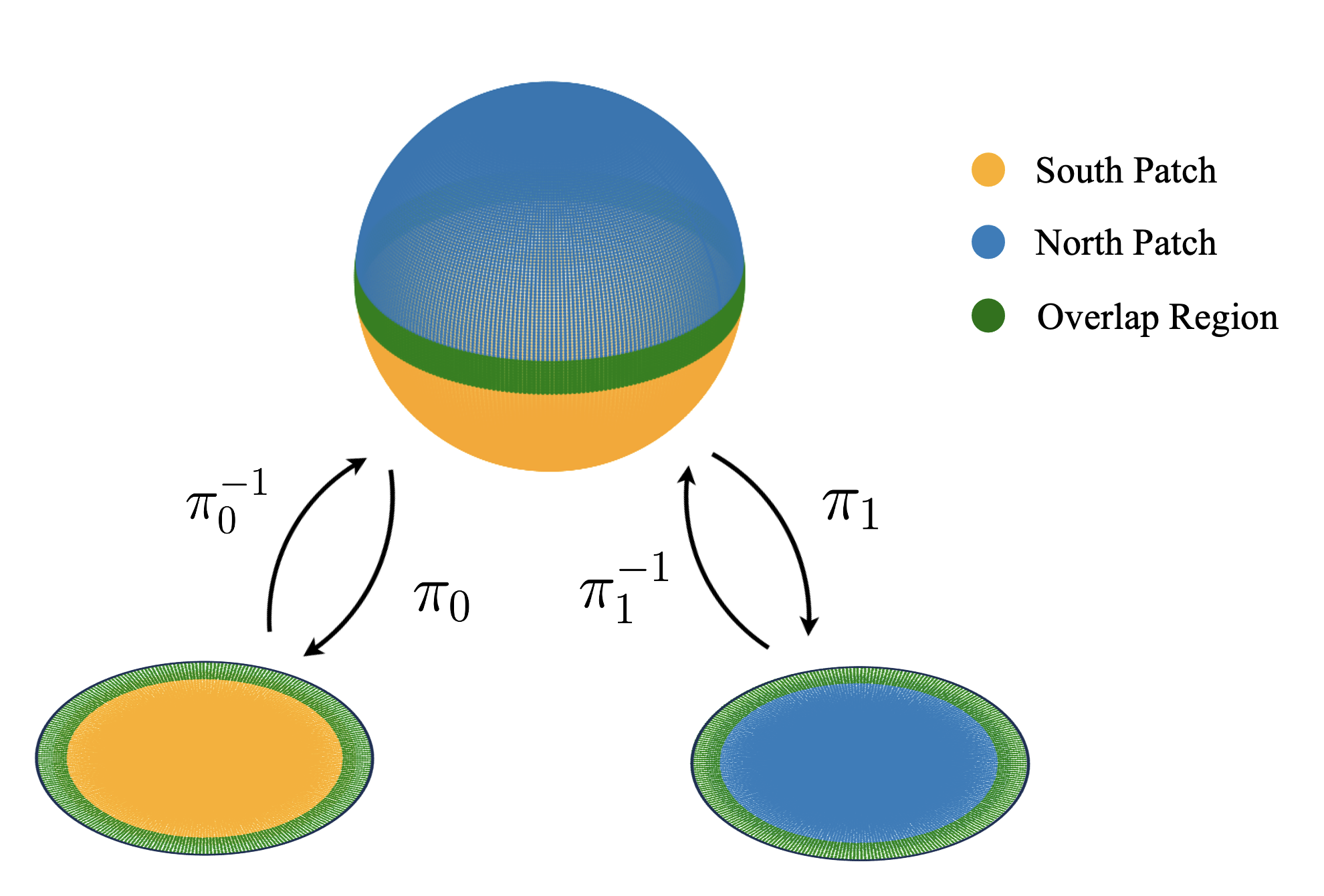}
  \caption{Stereographic map and inverse between a sphere and disks}
  \label{Fig: stereographic_maps}
\end{figure}

Training data consists of points uniformly distributed on the 2-sphere according to the standard surface measure. Uniform sampling is achieved by sampling the polar angle $\theta$ such that $\cos(\theta)$ is uniform on $[-1, 1]$ (corresponding to uniform distribution of the $z$-coordinate), and sampling the azimuthal angle $\phi$ uniformly on $[0, 2\pi]$. This ensures that the sampling density is proportional to the surface element $\sin(\theta) \, d\theta \, d\phi$, yielding \textit{uniform} coverage of the sphere. Once points are sampled in spherical coordinates, they are converted to Cartesian coordinates $(x, y, z) = (\sin\theta\cos\phi, \sin\theta\sin\phi, \cos\theta)$. These 3D coordinates are then projected into the appropriate patch coordinates using the stereographic projection described in Equation \eqref{eq: stereographic projection}. 

The sampling procedure includes a configurable radial offset parameter that allows points to extend slightly beyond the equator into the opposite hemisphere, giving rise to an overlap region. Training on points in this region carries one clear advantage, and one clear disadvantage. The network receives signals from both patches simultaneously, encouraging global smoothness of the conformal factor which is desirable. Conversely, the overlap region is weighted twice as much as the rest of the sphere, introducing a bias. It is found that due to the architecture of the network (described in the next section), it is not necessary to include any overlap region to achieve a solution that is smooth across the equator; therefore, all investigations are obtained with the following sampling strategy: two hemispheres sampled separately, with no overlap.

At each sampled point, the prescribed curvature function $K$ is evaluated to provide the target label for supervised learning. Additionally, a normalisation constant is computed once at the beginning of training by numerically integrating the squared curvature function over the sphere using either Monte Carlo sampling or Gauss-Legendre quadrature. This normalisation constant is attached to each training batch and used to scale the loss function as described below. 

\subsection{Architecture Details: the Nirenberg Neural Network} \label{sec:architecture}

The core innovation lies in using a neural network to directly parameterise the conformal factor $u$ as a function on the sphere thought of as a hypersurface in $\mathbb{R}^3$. Rather than working \textit{just} with patch coordinates, which would involve dealing with potential discontinuities at chart boundaries, the network operates on three-dimensional Cartesian coordinates $(x, y, z)$ lying on the unit sphere. This ensures that the conformal factor is globally well-defined and smooth across all patches.

The network begins with an optional encoding layer that transforms the input coordinates using random Fourier features (RFFs). This technique, based on kernel approximation theory, maps the input through randomly sampled trigonometric functions to capture high-frequency variations in the conformal factor. The encoding takes the form
\begin{equation}
\phi(x) = \sqrt{\frac{2}{m}}\left[\cos(Fx + \omega), \sin(Fx + \omega)\right],
\end{equation}
where the frequency matrix $F$ is sampled from a Gaussian distribution with unit standard deviation, and the phase vector $\omega$ is sampled uniformly. This preprocessing enables the network to represent complex, even periodic, features that might be difficult to capture with small numbers of linear layers. 

Figure~\ref{Fig:rff_loss_as_fn_l} shows the training loss for spherical harmonics of degree $\ell$, namely $Y_{\ell,0}(\theta,\phi)$. As $\ell$ increases, the functions become more oscillatory, and the neural network struggles to approximate them accurately. As explained above, it is expected that the use of RFFs will alleviate this behaviour. 

\begin{figure} [t!]
  \centering
  \includegraphics[scale=0.67, clip]{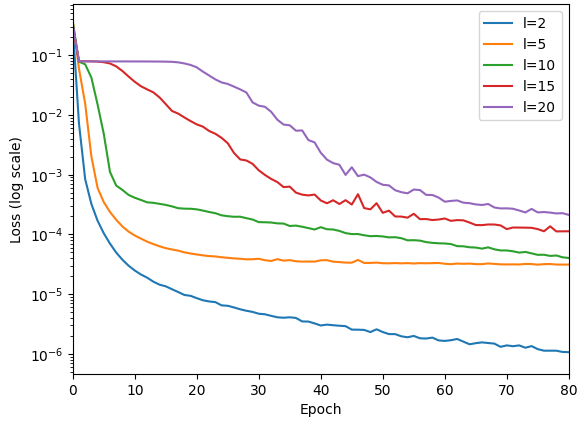}
  \caption{Training dynamics of the network without RFFs as a function of spherical harmonic order $l$.}
  \label{Fig:rff_loss_as_fn_l}
\end{figure}

Following the encoding, the architecture consists of several residual blocks, each containing fully connected layers with smooth non-linear activations. The residual connections allow the network to learn perturbations from the identity mapping, which is particularly effective when the solution is expected to be close to the trivial case (where $u = 0$, corresponding to the round metric). Each block processes the input through a pair of dense layers with activations such as GELU or SiLU, and adds the result back to the input via skip connections. This design helps avoid vanishing gradients in deeper networks.

The final layer produces a single scalar output representing the conformal factor $u$ at the input point on the sphere, and hence producing a $u$ value for all sphere points. This output layer uses small random initialisation with standard deviation of order $10^{-3}$, ensuring that the network begins near the trivial solution. The bias is initialised to zero, providing a neutral starting point that allows the optimization to explore the solution space without strong initial biases. The architecture just outlined is represented in Figure \ref{fig:stereograhic}.

\newpage
\begin{figure}[t!] 
    \centering
    \includegraphics[width=0.97\linewidth]{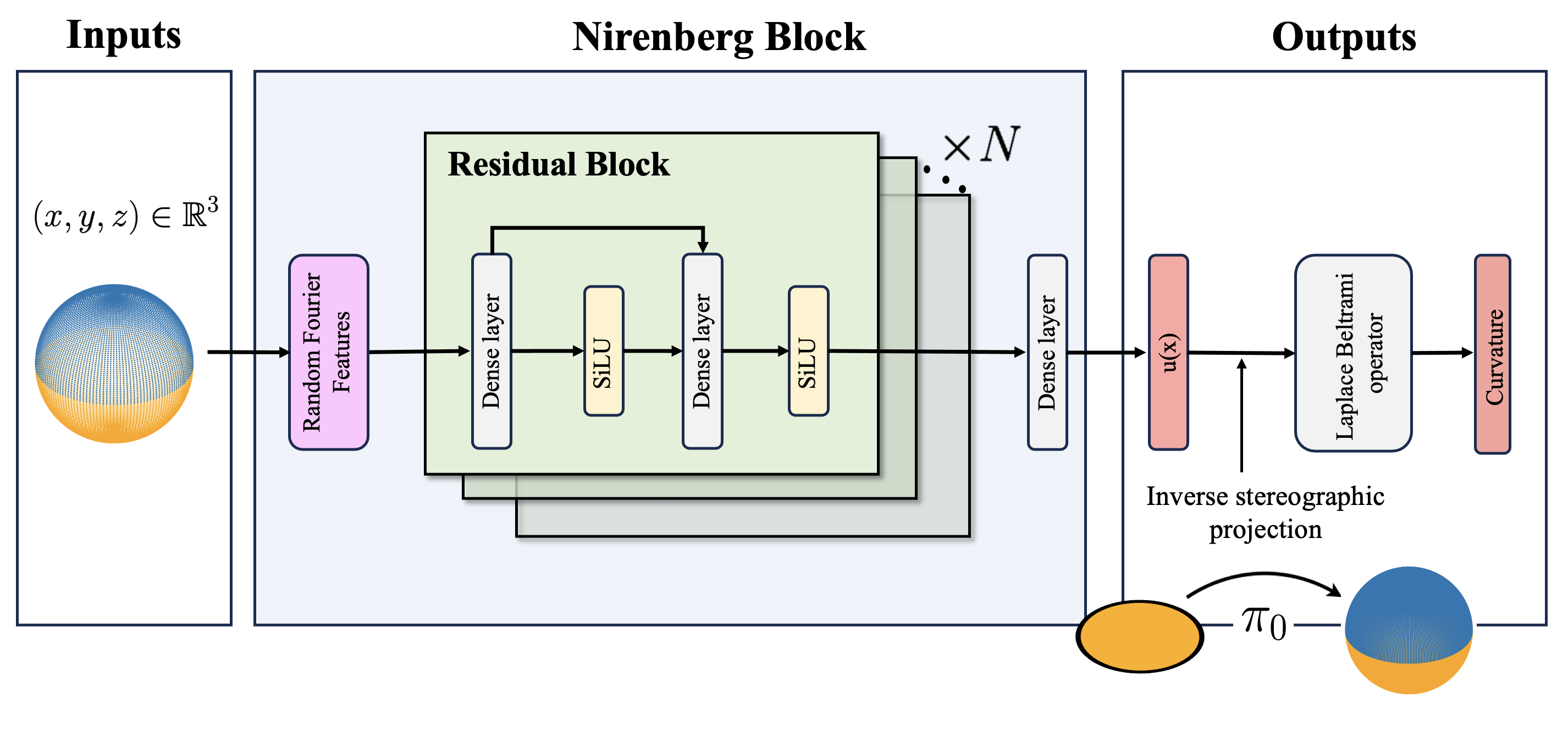}
    \caption{Topology of the Nirenberg Neural Network architecture.}
    \label{fig:stereograhic}
\end{figure}

A key technical challenge is computing the Laplace-Beltrami operator $\Delta_{g_0} u$ appearing in the Nirenberg PDE and the associated loss. Traditional numerical methods would require finite-difference approximations on a discretised grid, introducing discretisation errors and boundary artifacts. Instead, the Nirenberg Neural Network leverages automatic differentiation provided by modern deep learning frameworks to compute derivatives exactly (up to numerical precision) by constructing a computational graph and applying the chain rule.

The computation proceeds in two stages using nested gradient tracking. The first stage computes the gradient of the conformal factor with respect to patch coordinates, yielding $\nabla u = (\partial_x u, \partial_y u)$. The second stage computes the divergence of a weighted gradient, implementing the formula for the Laplace-Beltrami operator on a Riemannian manifold:
\begin{equation}
\Delta_{g_0} u = \frac{1}{\sqrt{|g_0|}} \partial_i\left(\sqrt{|g_0|} g_0^{ij} \partial_j u\right).
\end{equation}

\subsection{Loss Function}
\label{sec:loss}

The Nirenberg Neural Network is trained in the semi-supervised regime to approximate the metric $g$ which pointwise solves the Nirenberg partial differential equation. The loss function measures the mean squared error (MSE) between the predicted scalar curvature $R_g(x)$, and the rescaled prescribed Gaussian curvature $K(x)$\footnote{On a surface $R_g = 2K_g$, where $R_g$ and $K_g$ are the Ricci and Gaussian curvatures, respectively. This is accounted for by the factor of 2 in the definition of $R_g$ in terms of the Laplace-Beltrami operator.}, evaluated across points sampled on the sphere. In practice, the predicted curvature is computed from the network's conformal factor $R_g = e^{-2u}(2 - 2\Delta_{g_0} u)$, where $\Delta_{g_0}$ denotes the Laplace-Beltrami operator. The loss $\mathcal L$ is defined as
\begin{equation}\label{eq:nn_loss}
\mathcal{L} = \mathbb{E}_{x \sim S^2}\left[\frac{(R_g(x) - 2 K(x))^2}{\mathcal N}\right],
\end{equation}
where $\mathcal N$ is a normalisation factor given by the integral of the squared prescribed curvature over the sphere such that $\mathcal N = \int_{S^2} (2K)^2 \, d\Omega$. This normalisation serves two purposes: it balances the loss contributions across different prescribed functions with varying magnitudes, and it provides scale-invariant training that helps stabilise optimisation for curvature functions with extreme values. 

\subsection{Visualisation and Analytic Checks} \label{sec: visualisation}

The code comes equipped with comprehensive visualisation and diagnostics tools, enabling detailed inspection of the learnt solution. The utility of these tools will be presented in the next section, for a chosen set of prescribed functions. These plots reveal the spatial structure of the solution and help identify regions where the network may show pathological behaviours which were not captured by the loss value alone. These include numerical artifacts which can often be quite elusive, unless an explicit inspection is performed. 

A crucial verification step involves checking the Gauss-Bonnet theorem, c.f. Equation~\eqref{eq:gauss_bonnet}. This integral is computed numerically by sampling points uniformly on the sphere, evaluating the predicted scalar curvature $K_g$ and the conformal volume element $dA_g = e^{2u} dA_{g_0}$, and averaging over samples weighted by the surface measure. The result should equal $4\pi$ regardless of the conformal factor, as this is a topological invariant. Significant deviation from this value indicates numerical instabilities or failure to satisfy the underlying geometric constraints.

\section{Results}
\label{sec: results}

This section presents the training and testing results of the Nirenberg Neural Network over a series of prescribed curvature functions (prescribers). Section~\ref{sub:results_spectralpairs} studies the stable benchmark of Spectral Pairs generated via Proposition~\ref{prop: spectral pairs}. Section~\ref{sub:results_spherical_harmonics} then considers prescribers built from spherical harmonics, including known solvable and unsolvable cases as well as candidates whose solvability is currently unknown. Section~\ref{sub:results_general_funcs} extends the experiments to more general prescribers. Finally, Section~\ref{sub:sh_expansion} fits spherical harmonic expansion ansatz to the trained models, yielding interpretable closed-form approximations of the learned conformal factors $u$.

\subsection{Spectral Pairs}
\label{sub:results_spectralpairs}
Proposition \ref{prop: spectral pairs} offers a controlled evaluation setting to directly probe the correctness of the network. For a given spectral pair $(u_\text{true}, K)$, the ground truth conformal factor $u_\text{true}$ is compared with the network's predicted conformal factor $u_\text{pred}$. Over a fixed set of evaluation points $\{x_i\}_{i=1}^N \subset S^2$, the following metrics are reported:
\begin{enumerate}
    \item Mean absolute error (MAE),
    \[\operatorname{MAE}(u_\text{pred}, u_\text{true}) := \frac{1}{N}\sum_{i=1}^N\big| u_\text{pred}(x_i)-u_\text{true}(x_i) \big|.\]
    \item Bias,
    \[\operatorname{Bias}(u_\text{pred}, u_\text{true})
    := \frac{1}{N}\sum_{i=1}^N \Big( u_\text{pred}(x_i)-u_\text{true}(x_i) \Big).\]
    \item Pearson correlation coefficient (PCC)
    \[
    \operatorname{PCC}(u_\text{pred}, u_\text{true})
    := \frac{\sum_{i=1}^N \big(u_\text{pred}(x_i)-\overline{u_\text{pred}}\big)\big(u_\text{true}(x_i)-\overline{u_\text{true}}\big)}
    {\sqrt{\sum_{i=1}^N \big(u_\text{pred}(x_i)-\overline{u_\text{pred}}\big)^2}
    \sqrt{\sum_{i=1}^N \big(u_\text{true}(x_i)-\overline{u_\text{true}}\big)^2}},
    \]
    where $\overline{u_\text{pred}} := \frac{1}{N}\sum_{i=1}^N u_\text{pred}(x_i)$ and
    $\overline{u_\text{true}} := \frac{1}{N}\sum_{i=1}^N u_\text{true}(x_i)$.
\end{enumerate}

\begin{table}[ht!]
\centering
\resizebox{\textwidth}{!}{%
\renewcommand{\arraystretch}{1.2}
\setlength{\tabcolsep}{10pt}
\begin{tabular}{l|ccc}
\toprule
\textbf{Prescriber} & MAE & PCC & Bias \\
\midrule
\texttt{prop\_a}  & $7.95\times10^{-6}\,{\color{gray}\scriptstyle \pm\,3.1\times10^{-6}}$ & $1.0\,{\color{gray}\scriptstyle \pm\,6.6\times10^{-9}}$ & $1.98\times10^{-6}\,{\color{gray}\scriptstyle \pm\,6.3\times10^{-6}}$ \\
\midrule
\texttt{prop\_b}  & $1.21\times10^{-2}\,{\color{gray}\scriptstyle \pm\,2.5\times10^{-3}}$ & $0.93\,{\color{gray}\scriptstyle \pm\,5.7\times10^{-2}}$ & $4.22\times10^{-4}\,{\color{gray}\scriptstyle \pm\,4.7\times10^{-5}}$ \\
\midrule
\texttt{prop\_c}  & $2.12\times10^{-2}\,{\color{gray}\scriptstyle \pm\,1.6\times10^{-2}}$ & $0.98\,{\color{gray}\scriptstyle \pm\,3.1\times10^{-3}}$ & $-1.17\times10^{-3}\,{\color{gray}\scriptstyle \pm\,1.5\times10^{-3}}$ \\
\midrule
\texttt{round}  & $7.75\times10^{-7}\,{\color{gray}\scriptstyle \pm\,7.4\times10^{-7}}$ & -- & $8.68\times10^{-9}\,{\color{gray}\scriptstyle \pm\,1.1\times10^{-8}}$ \\
\bottomrule
\end{tabular}%
}
\caption{Spectral pair evaluation metrics (mean ${\color{gray}\scriptstyle \pm \text{std}}$ over 5 random training seeds).}
\label{tab:spectral_pair_metrics}
\end{table}

Table \ref{tab:spectral_pair_metrics} reports these metrics for three spectral pairs (\texttt{prop\_a}, \texttt{prop\_b}, \texttt{prop\_c}) and a \texttt{round} baseline, for which $u_\text{true}\equiv 0$. For \texttt{prop\_a}, the MAE is comparable to the \texttt{round} baseline, indicating near-perfect recovery of the reference conformal factor. The MAE is larger for \texttt{prop\_b} and \texttt{prop\_c}, and their bias is non-negligible, revealing a significant global offset component in the error. At the same time, the PCC remains consistently high, indicating that the \emph{spatial variation} of the conformal factor is accurately captured: after mean-centering, the extrema and oscillatory structure of $u_\text{pred}$ aligns closely with those of $u_\text{true}$ across the sphere (equivalently, $u_\text{pred}-\overline{u_\text{pred}}$ agrees strongly with $u_\text{true}-\overline{u_\text{true}}$ up to a positive scaling). 

It is worth noting that this distinction is related to non-uniqueness in the Nirenberg problem. If $u$ solves the prescribed curvature equation for $K$, then composing with a Möbius transformation yields another solution corresponding to the pullback curvature. Hence, to compare learnt solutions with a prescribed target, a gauge-fixing to quotient out the action of $\mathrm{Conf}(S^2)$ shall be implemented. We leave a systematic gauge-fixing procedure for future work.

\subsection{Spherical Harmonics}
\label{sub:results_spherical_harmonics}

Before employing the Nirenberg Neural Network to infer solvability for general prescribers, performance is benchmarked on the class of spherical harmonics $Y_{\ell, m}(\theta, \phi)$. This class provides a rigorous testing ground as the analytical results sharply divide the parameter space $(\ell, m)$ into regions of guaranteed existence, guaranteed non-existence, and analytical ambiguity.

Based on the theoretical background presented in Section \ref{section: existence results}, the spherical harmonics are categorised as presented in Table \ref{tab:spherical_harmonics}.
\medskip
\begin{table}[h]
\centering
\begin{tabularx}{\textwidth}{|p{0.24\textwidth}|p{0.25\textwidth}|X|}
\hline
\rowcolor{blue!30}
\textbf{Formula} & \textbf{Name} & \textbf{Comments/Reason} \\
\Xhline{2pt}

\multicolumn{3}{|c|}{\textbf{Known to exist}} \\
\Xhline{2pt}

\(Y_{2\ell,m}\) & \texttt{sh\_2l\_m} & Example \ref{example: even spherical ahrmonics} \\
\hline
$Y_{2,0}$ & \texttt{sh\_2\_0} & Proposition \ref{prop: Classification zonal spherical harmonics}\\
\hline
$Y_{3,0}$ & \texttt{sh\_3\_0} &  Proposition \ref{prop: Classification zonal spherical harmonics}\\
\Xhline{2pt}

\rowcolor{blue!15}
\multicolumn{3}{|c|}{\textbf{Known not to exist}} \\
\Xhline{2pt}

$Y_{1,0}$ & \texttt{sh\_1\_0} & Proposition \ref{prop: Classification zonal spherical harmonics} \\
\hline
$Y_{1,1}$ & \texttt{sh\_1\_1} & Theorem \ref{theo: Kazdan-Warner Identity}\\
\hline
\Xhline{2pt}

\rowcolor{blue!15}
\multicolumn{3}{|c|}{\textbf{Existence unknown}} \\
\Xhline{2pt}

$Y_{2\ell +1,m}$, $\ell\geq 1$ & \texttt{sh\_2l+1\_m} & Unknown \\
\hline
$Y_{3,1}$ & \texttt{sh\_3\_1} & Unknown \\
\hline
$Y_{3,2}$ & \texttt{sh\_3\_2} & Unknown \\
\hline
$Y_{3,3}$ & \texttt{sh\_3\_3} & Unknown \\
\hline

\end{tabularx}
\caption{Overview of spherical harmonics under test. }
\label{tab:spherical_harmonics}
\end{table}

\subsubsection{Motivating Existence}
From Figure \ref{fig:sh_loss_scatter}, is is observed that the spherical harmonics $Y_{2,0}, Y_{3,0}$ (which are known to be realisable) exhibit significantly lower loss than those for $Y_{1,0}$ and $Y_{1,1}$ (which are known to be unrealisable). 
More precisely, the former differ by loss approximately two orders of magnitude lower than the latter. All experimental results in this section are obtained from 50 independent experiments run with unique seeds, and over two discrete hyperparameter configurations specified in Section \ref{sub:training_parameters}.

From the level of loss it is possible to infer the solvability of a given prescriber. Examining a series of known cases allows a threshold to be drawn: if the loss obtained from training with an unknown prescriber is of similar order to those which are known to be solvable, then it is increasingly likely that, to an accuracy attainable with the network, the prescriber admits a solution to the Nirenberg equation. It can be seen that the losses for the $Y_{3,1}, Y_{3,2},$ and $Y_{3,3}$ spherical harmonics are of similar order to the realisable prescribers, suggesting that these are likely to be solvable. 

As is suggested in Section \ref{sub:geom_constr_on_curv}, the Gauss-Bonnet theorem may be utilised to validate the necessary condition for the solvability. Figure \ref{fig:sh_gauss_bonnet} shows the Gauss-Bonnet error for the targeted spherical harmonic prescribers; the Gauss-Bonnet errors of $Y_{3,1}, Y_{3,2},$ and $Y_{3,3}$ are of comparable order to that of $Y_{3,0}$, whilst the errors of $Y_{1,0}$ and $Y_{1,1}$ are at least two orders of magnitude higher. These findings are further supported by the visualisations in Figure \ref{fig: spherical harmonic plots}, which show that the predicted curvature for $Y_{1,1}$ and $Y_{3,1}$ closely matches the target curvature. 

\begin{figure} [ht!]
    \centering
    \includegraphics[width=0.8\linewidth]{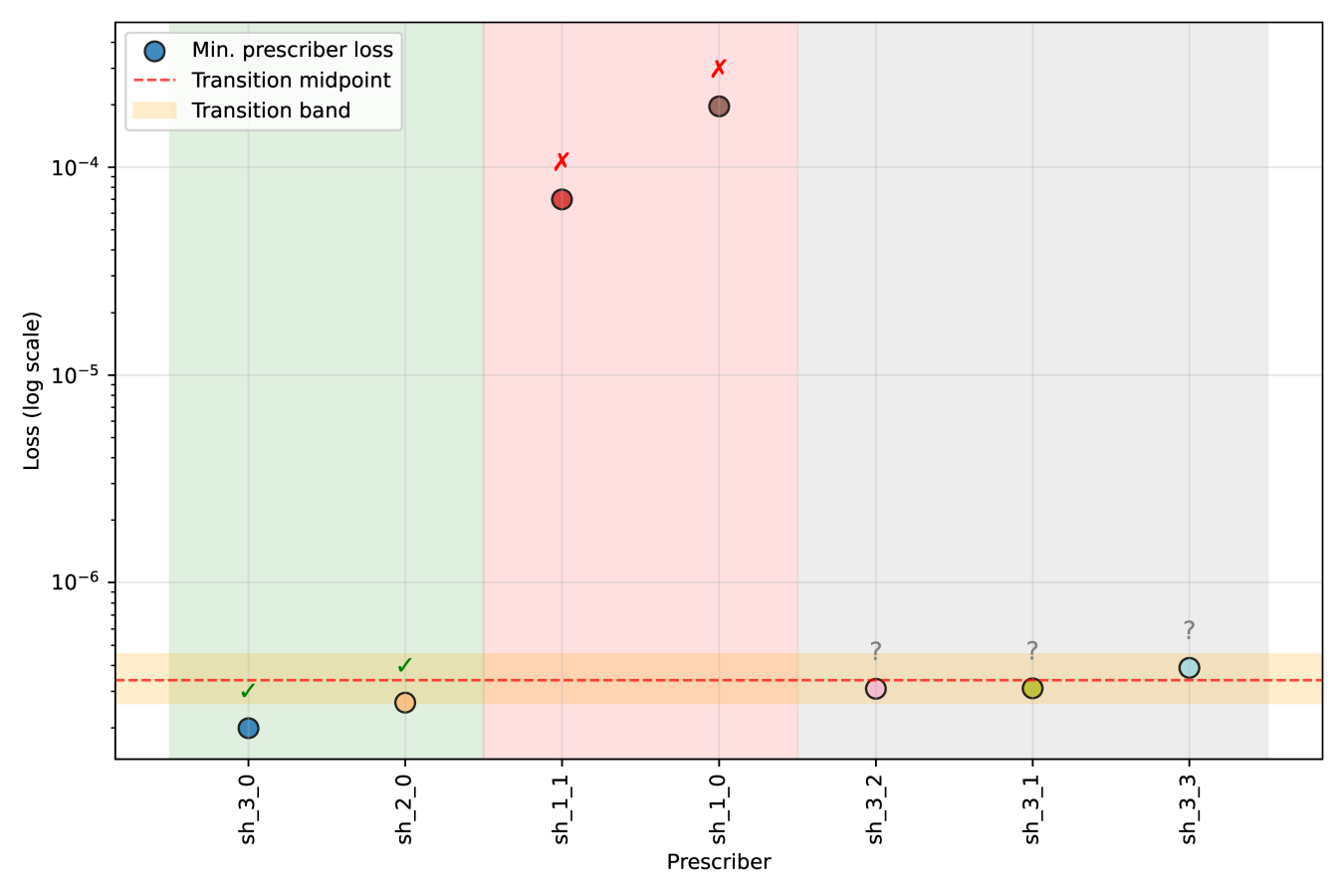}
    \caption{Minimum losses for spherical harmonic prescribers. Known realisable and unrealisable solutions are shown as ticks and crosses respectively. Unknown solutions are represented by `?'. Band represents the space between highest loss of known realisable and lowest loss of known unrealisable solutions from ensemble of all runs in this work (presented in Figure \ref{fig:general_loss_scatter}).}
    \label{fig:sh_loss_scatter}
\end{figure}

\begin{figure} [ht!]
    \centering
    \includegraphics[width=0.8\linewidth]{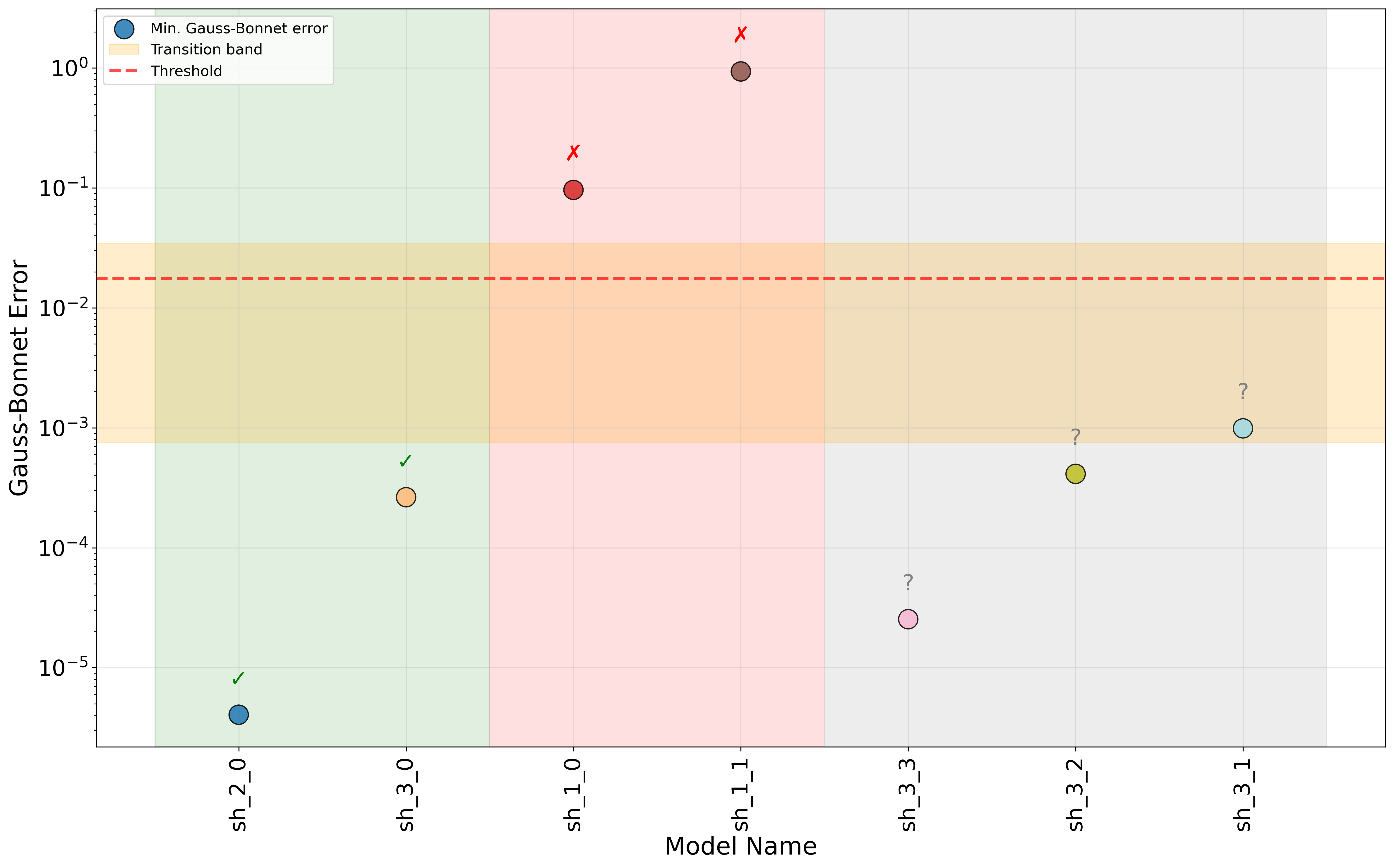}
    \caption{Gauss-Bonnet violations for the spherical harmonic prescribers. The integral of the predicted scalar curvature over the sphere is performed numerically, using the conformal volume element. The plot shows the relative deviation of this quantity, $\int_{S^2} K_g e^{2u} \, dA_{g_0}$, from the expected value of $ 2 \pi$. The notation and stylistic choices are the same as in Figure \ref{fig:sh_loss_scatter}; accordingly, the band represents the space between highest loss of known realisable and lowest loss of known unrealisable solutions from ensemble of all runs in this work (presented in Figure \ref{fig:general_gauss_bonnet}).}
    \label{fig:sh_gauss_bonnet}
\end{figure}

\begin{figure*}[ht!]
    \centering

    \noindent
    \colorbox{myLightgreen}{%
        \parbox{\dimexpr\textwidth-2\fboxsep}{%
            \centering\bfseries \texttt{sh\_2\_0}
        }%
    }    
    \vspace{1.0em}
        \begin{subfigure}[t]{0.22\textwidth}
        \centering
        \includegraphics[width=\linewidth]{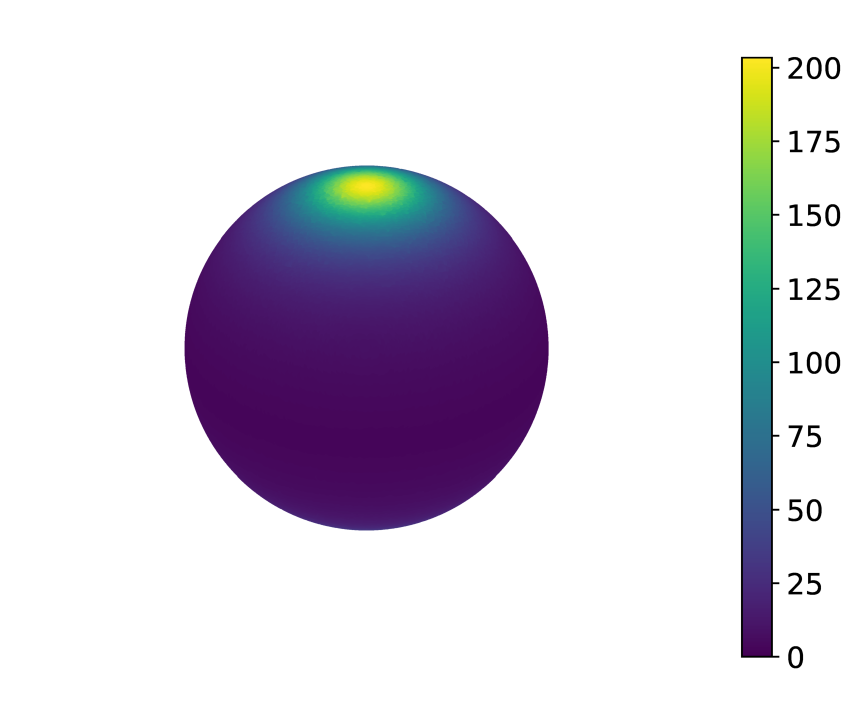}
        \caption{\texttt{Metric $g_{00}$}}
    \end{subfigure}
    \hfill
    \begin{subfigure}[t]{0.22\textwidth}
        \centering
        \includegraphics[width=\linewidth]{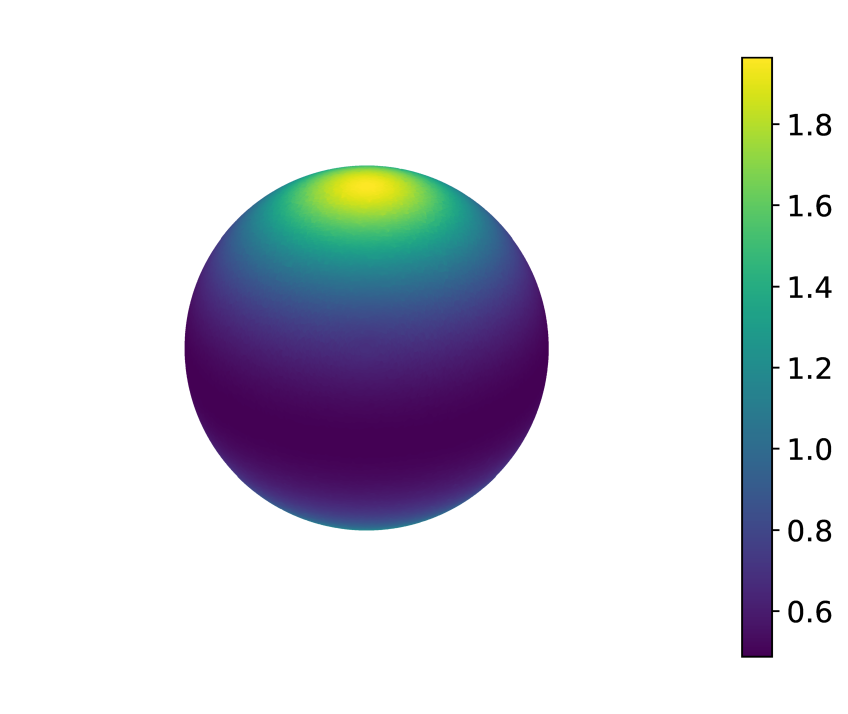}
        \caption{\texttt{u}}
    \end{subfigure}
    \hfill
    \begin{subfigure}[t]{0.22\textwidth}
        \centering
        \includegraphics[width=\linewidth]{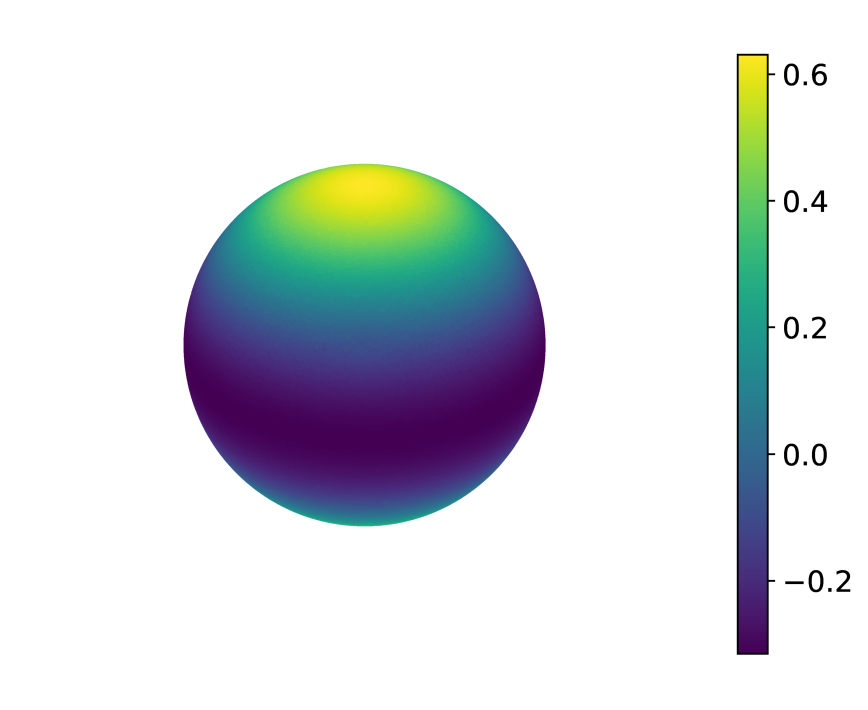}
        \caption{\texttt{Predicted R}}
    \end{subfigure}
        \hfill
    \begin{subfigure}[t]{0.22\textwidth}
        \centering
        \includegraphics[width=\linewidth]{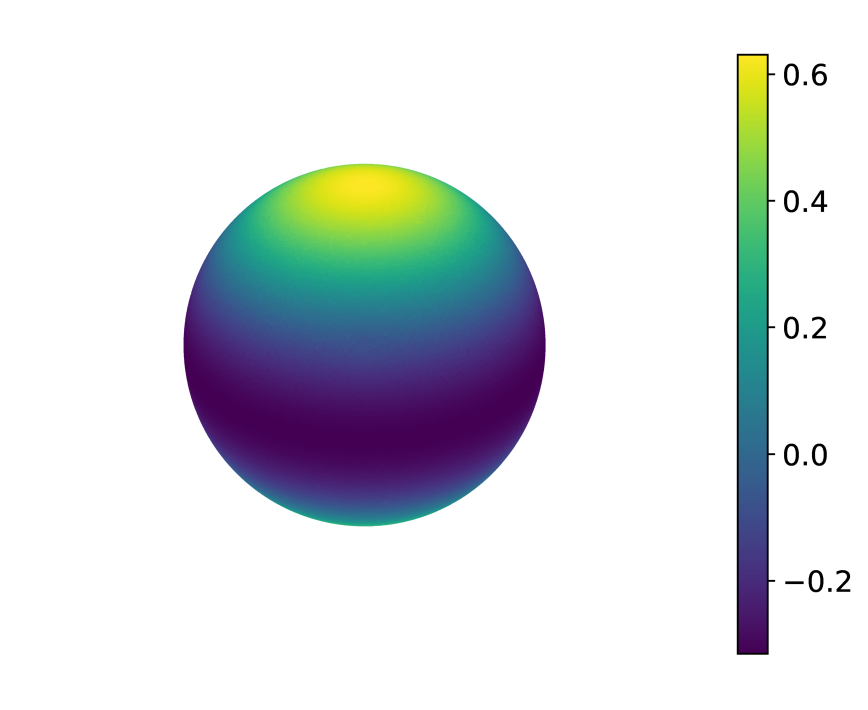}
        \caption{\texttt{Target R}}
    \end{subfigure}
   
      \vspace{1em}

    \noindent
    \colorbox{red!15}{%
        \parbox{\dimexpr\textwidth-2\fboxsep}{%
            \centering\bfseries \texttt{sh\_1\_1}
        }%
    }    
    \vspace{1.0em}
        \begin{subfigure}[t]{0.22\textwidth}
        \centering
        \includegraphics[width=\linewidth]{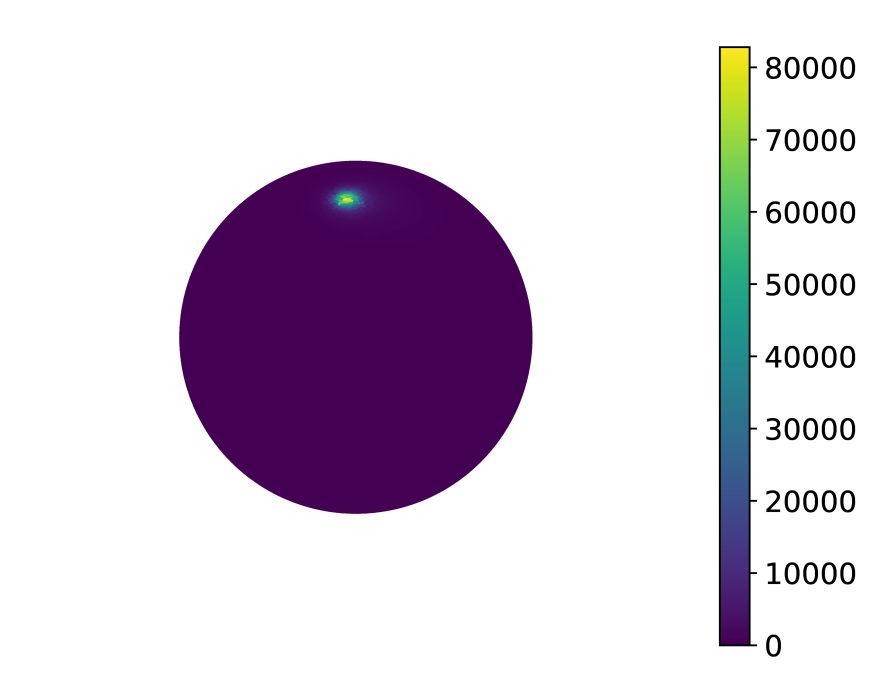}
        \caption{\texttt{Metric $g_{00}$}}
    \end{subfigure}
    \hfill
    \begin{subfigure}[t]{0.22\textwidth}
        \centering
        \includegraphics[width=\linewidth]{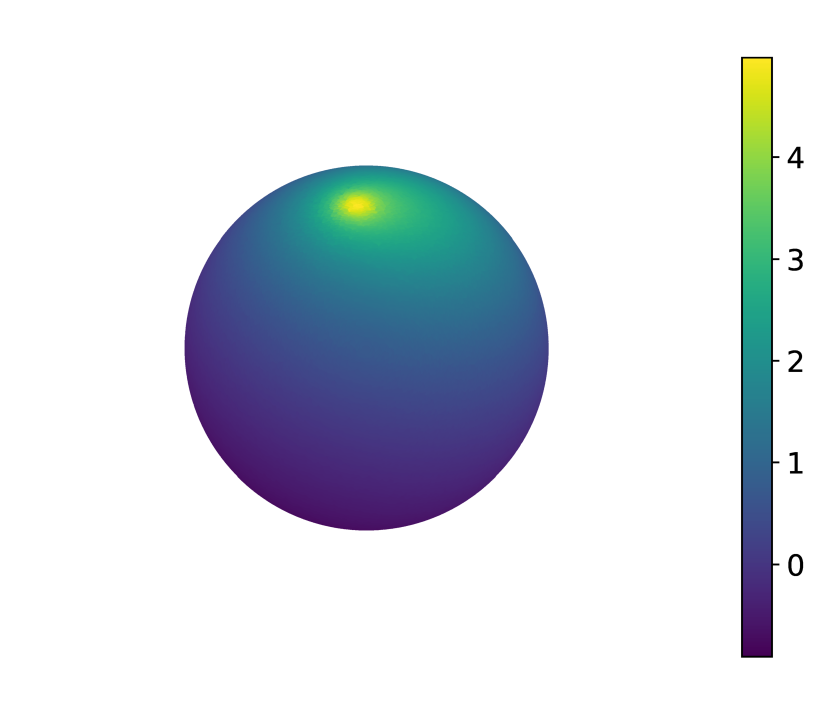}
        \caption{\texttt{u}}
    \end{subfigure}
    \hfill
    \begin{subfigure}[t]{0.22\textwidth}
        \centering
        \includegraphics[width=\linewidth]{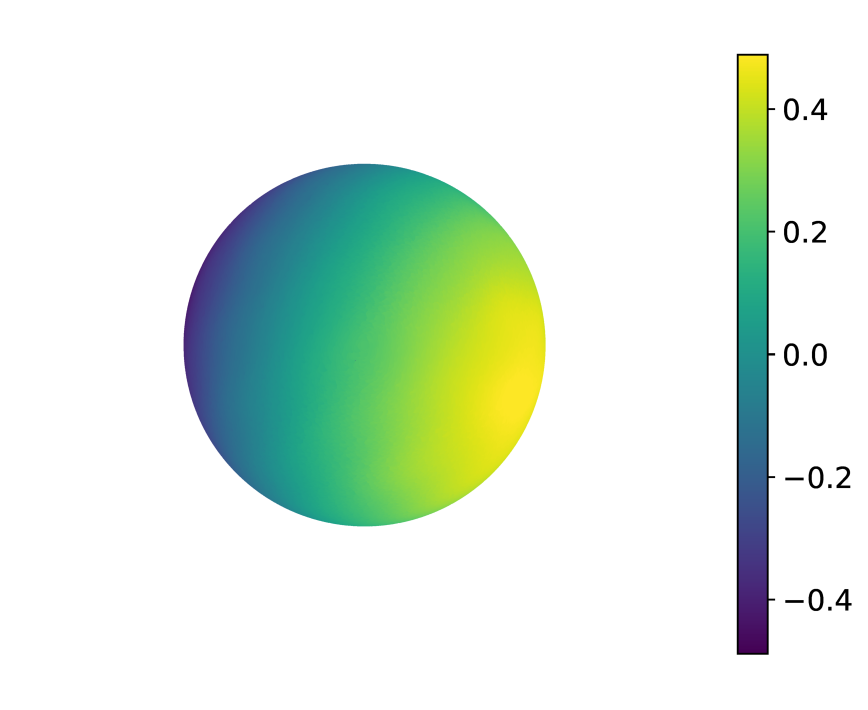}
        \caption{\texttt{Predicted R}}
    \end{subfigure}
        \hfill
    \begin{subfigure}[t]{0.22\textwidth}
        \centering
        \includegraphics[width=\linewidth]{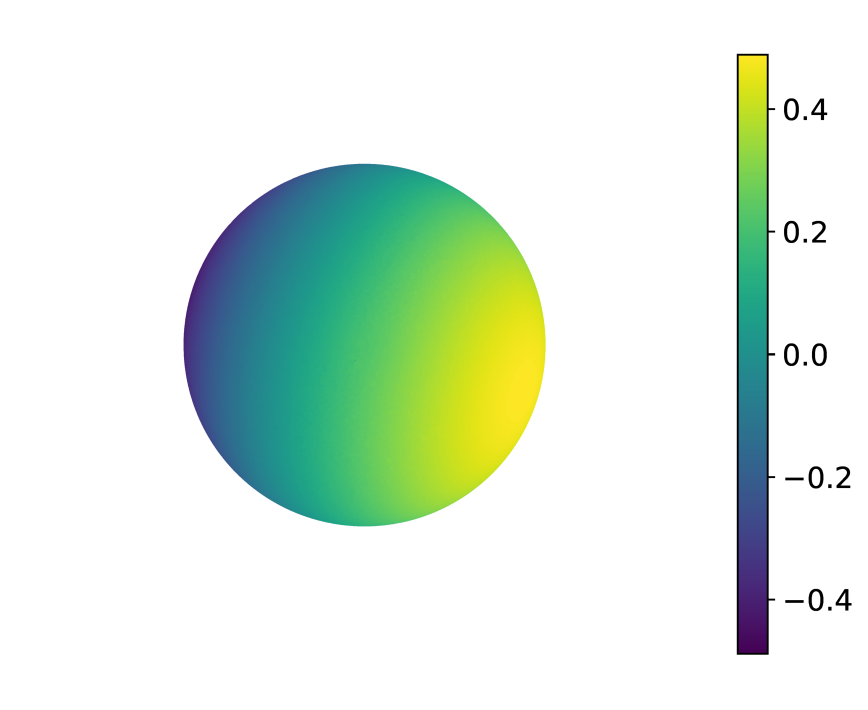}
        \caption{\texttt{Target R}}
    \end{subfigure}

          \vspace{1em}

    \noindent
    \colorbox{myLightgrey}{%
        \parbox{\dimexpr\textwidth-2\fboxsep}{%
            \centering\bfseries \texttt{sh\_3\_1}
        }%
    }    
    \vspace{1.0em}
        \begin{subfigure}[t]{0.22\textwidth}
        \centering
        \includegraphics[width=\linewidth]{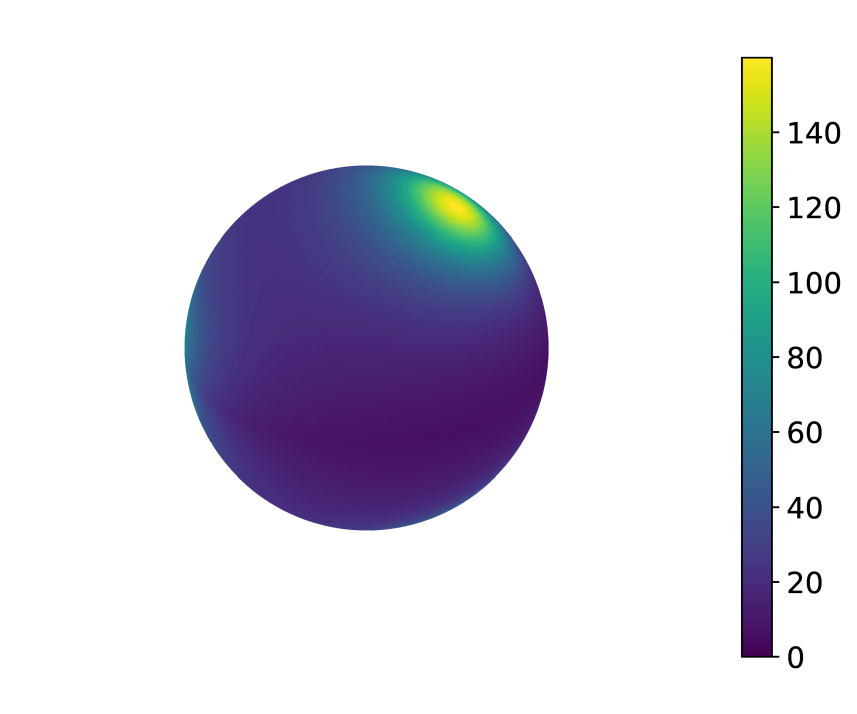}
        \caption{\texttt{Metric $g_{00}$}}
    \end{subfigure}
    \hfill
    \begin{subfigure}[t]{0.22\textwidth}
        \centering
        \includegraphics[width=\linewidth]{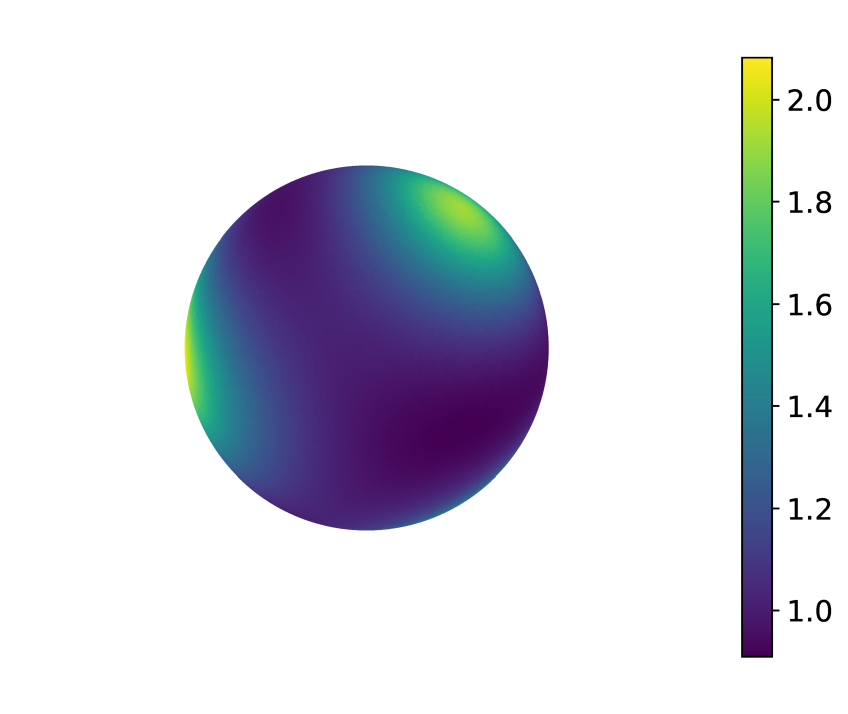}
        \caption{\texttt{u}}
    \end{subfigure}
    \hfill
    \begin{subfigure}[t]{0.22\textwidth}
        \centering
        \includegraphics[width=\linewidth]{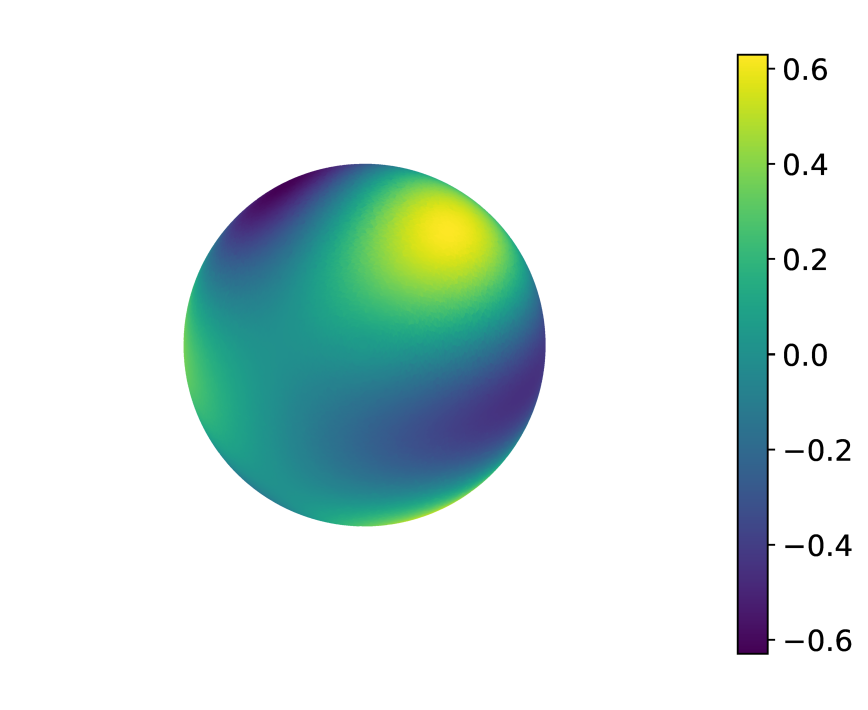}
        \caption{\texttt{Predicted R}}
    \end{subfigure}
        \hfill
    \begin{subfigure}[t]{0.22\textwidth}
        \centering
        \includegraphics[width=\linewidth]{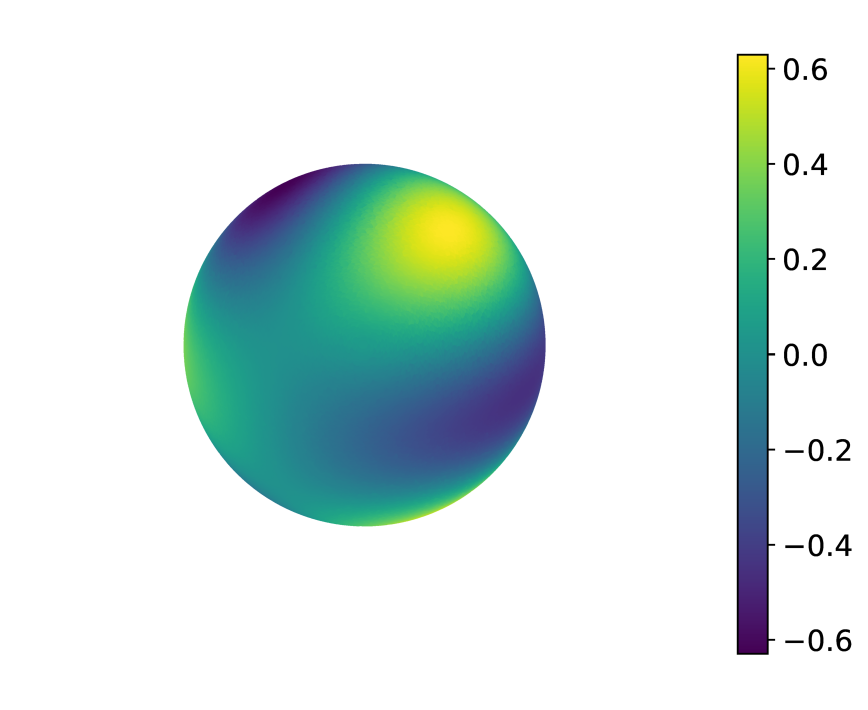}
        \caption{\texttt{Target R}}
    \end{subfigure}
    \caption{Metric component $g_{00}$, conformal factor $u$, and predicted and target curvature $R$.} 
    \label{fig: spherical harmonic plots}
\end{figure*}

\subsubsection{Motivating Lack of Existence}

In addition to validating realisability, the Nirenberg Neural Network provides corroborating evidence \textit{against} the existence of solutions. That being said, drawing a lack of existence is perhaps more challenging: the inability to obtain solution does not prohibit its existence. Additional steps are taken to remediate this issue. A potential reason for the network’s failure to converge to a solution could be a lack of expressivity in the architecture. This is particularly relevant for spherical harmonics. Typically neural networks struggle to encode highly oscillatory behaviour such as that of the trigonometric functions which act as the building blocks of spherical harmonics. In order to ameliorate this risk, a large number of random Fourier features (RFFs), described in Section \ref{sec:architecture}, are used to ensure sufficient expressibility. The identification of solutions for prescribers exhibiting similar complexity to those which do not yield results provides first-order reassurance that the network has the ability to encode the periodic behaviour of the trigonometric functions. To mitigate biases induced by the choice of network properties, a selection of identical hyperparameter configurations are used for each prescriber. The lowest losses are collated in Figure \ref{fig:sh_loss_scatter}. 

Note that $Y_{1,0}$ and $Y_{1,0}$ (i.e. \texttt{sh\_1\_0} and \texttt{sh\_1\_1}) are unrealisable. The high loss in Figure \ref{fig:sh_loss_scatter} and the pathological behaviour of the metric shown in Figure \ref{fig: spherical harmonic plots} confirms that the network is unable to construct such a metric. The lowest losses for the $Y_{1,0}$ and $Y_{1,1}$ prescribers are of a magnitude greater than for the known realisable prescribers. The lack of existence is further corroborated by the Gauss-Bonnet verification, as shown in Figure \ref{fig:sh_gauss_bonnet}; we find this test to be a very good detector for numerical artifacts and pathologies. The results of the analytic test outlined in the previous section associated with the terminal metric for the aforementioned spherical harmonics are 5 orders of magnitude larger than that of the lowest-loss known cases, and 3 orders of magnitude higher than the largest. This provides reassurance that even \textit{if} the `metrics' realised the desired curvature, they would not belong to the same topological class as the sphere.

There is a question as to why spherical harmonics over degree 3 were not studied. Empirically, it was found that for the hyperparameters set in \ref{sub:training_parameters}, the lack of expressibility of the network caused $l=5$ and $l=7$ losses to be misleadingly large. A small set of runs with more complex networks resulted in losses which were of similar order to the $l=3$ cases for some configurations of $m$, although a systematic exploration was not completed. Whilst running with a greater number of Fourier features allows greater expressivity and thus access to higher $l$, the decision to set the upper limit to $l = l_\text{upper}$ is nonetheless arbitrary, and future work with larger architectures and more compute power has the potential to address this. 
\subsection{General Functions}
\label{sub:results_general_funcs}
The collection of general prescribers studied is summarised in Tables \ref{tab:known_exist_general}, \ref{tab:known_noexist_general}, and \ref{tab:unknown_exist_general}. This set includes: (i) functions that are known to be solvable, either by explicit construction or via Theorem ~\ref{theo: rotationally symmetric solutions}; (ii) functions that are known to be unsolvable; and (iii) functions for which the solvability is unknown. In each case, the same training protocol and ensemble of hyperparameter choices described in the previous subsection is employed, with 50 random seed runs with two network complexity levels (see Appendix \ref{sub:training_parameters}), allowing for a direct comparison of achieved loss values across prescribers. All models are considered at the point where the training losses have converged. Figure~\ref{fig:general_loss_scatter} provides a global summary of these results by displaying, on a logarithmic scale, the minimum loss attained across all hyperparameter configurations for each general prescriber. Known realisable and unrealisable cases are marked by ticks and crosses respectively, while prescribers of unknown validity are denoted by question marks. The dashed horizontal line indicates the empirical loss threshold inferred from the ensemble of spherical harmonic experiments. Figure $\ref{fig: general function plots}$ visualizes the metric component $g_{00}$, conformal factor $u$, and predicted and target curvature $R$ for some of the considered functions. 

This section explores the application of the Nirenberg Neural Network to prescriber functions beyond spherical harmonics. Unlike the spherical harmonic case, where symmetry and spectral structure provide strong analytical guidance, the functions considered here are chosen to probe a broader range of geometric and analytic features, including monotonicity, sign-changing behaviour, degeneracy of critical points, and the validity conditions outlined in Section \ref{section: existence results}. 

\begin{table}[h]
\begin{tabularx}{\textwidth}{|p{0.28\textwidth}|p{0.21\textwidth}|X|}
\hline
\rowcolor{blue!30}
\textbf{Formula} & \textbf{Name} & \textbf{Comments/Reason} \\
\Xhline{2pt}

\rowcolor{blue!15}
\multicolumn{3}{|c|}{\textbf{Known to exist}} \\
\Xhline{2pt}

$3\cos^2\theta \; e^{(\cos^2\theta-\frac{1}{3})}$ & \texttt{prop\_a} & Proposition \ref{prop: spectral pairs}, with $c_{2,0}=1$. \\
\hline
$2\cosh(\cos\theta)$ & \texttt{cosh\_profile} & Theorem \ref{theo: rotationally symmetric solutions}\\
\hline
\ldots & \texttt{prop\_b}  & Proposition \ref{prop: spectral pairs}, with $c_{1,1}=1/6$, $c_{2,2}=1/18$, $c_{3,1}=3/11$. \\
\hline
\small $\tanh(2\cos^2\theta) + \cos^2\theta + 0.5$ & \texttt{tanh\_wave} & Theorem \ref{theo: rotationally symmetric solutions} \\
\hline
$(1.5 + \cos\theta)^2(1.2 - \cos\theta)$ & \texttt{egg} &  Theorem \ref{theo: rotationally symmetric solutions}. \\
\hline
$5\cos^3(\theta) - 3\cos(\theta)+1$ & \texttt{5cos3\_1} & Theorem \ref{theo: rotationally symmetric solutions}\\
\hline
$\cos^2(\theta)+0.2$ & \texttt{cos2\_theta\_plus\_offset} &Theorem \ref{theo: rotationally symmetric solutions}\\
\hline
\ldots & \texttt{prop\_c}  & Proposition \ref{prop: spectral pairs}, with $c_{1,1}=2$, $c_{2,2}=6$, $c_{3,1}=12$. \\
\hline
$1+0.5\cos(3\theta)$ & \texttt{sinusoidal} &  Theorem \ref{theo: rotationally symmetric solutions} \\
\hline
\end{tabularx}
\caption{Table of prescribers which are solvable.}
\label{tab:known_exist_general}
\end{table}
\medskip

\begin{table}[ht!]
\begin{tabularx}{\textwidth}{|p{0.24\textwidth}|p{0.25\textwidth}|X|}
\hline
\rowcolor{blue!30}
\textbf{Formula} & \textbf{Name} & \textbf{Comments/Reason} \\
\Xhline{2pt}

\rowcolor{blue!15}
\multicolumn{3}{|c|}{\textbf{Known not to exist}} \\
\Xhline{2pt}
\hline
$\exp(2 \cos \theta)$ & \texttt{monotonic\_exp} & Theorem \ref{theo: rotationally symmetric solutions} (No sign change of $f'$) \\
\hline
$\cos^2 \theta - 0.9$ & \texttt{nodal\_crossing} & Theorem \ref{theo: rotationally symmetric solutions} ($f'$ does only change sign for $f<0$) \\
\hline
$2 + y$ & \texttt{two\_plus\_y} & $\nabla K\cdot\nabla y = |\nabla y|^2$, violating \eqref{eq:KW_identity}. \\
\hline
$2 + x$ & \texttt{two\_plus\_x} & $\nabla K\cdot\nabla x = |\nabla x|^2$, violating \eqref{eq:KW_identity}. \\
\hline
$z + \frac{1}{4}x^2$ & \texttt{z\_plus\_x\_squared\_over\_4} \; & $\nabla_{\partial\theta}K = \sin\theta(\cos\theta\cos^2\phi-1)\le 0$, violating \eqref{eq:KW_identity}. \\
\hline
$z + xy$ & \texttt{z\_plus\_xy} & $\nabla_{\partial\theta}K = \sin\theta(\cos\theta\sin2\phi-1)\le 0$, violating Kazdan-Warner Identity. \\
\hline
$-(1 + \cos^2 \theta)^{-1}$ & \texttt{negative\_dip} & Theorem \ref{theo: rotationally symmetric solutions} (The necessary condition $f > 0$ somewhere is violated.) \\
\hline
\end{tabularx}
\caption{Table of prescribers which are unsolvable.}
\label{tab:known_noexist_general}
\end{table}
\medskip

\begin{table}[ht!]
\begin{tabularx}{\textwidth}{|p{0.35\textwidth}|X|}
\hline
\rowcolor{blue!30}
\textbf{Formula} & \textbf{Name}\\
\Xhline{2pt}

\rowcolor{blue!15}
\multicolumn{2}{|c|}{\textbf{Unknown existence}}\\
\Xhline{2pt}
\hline
$x^2-y^2+z$ & \texttt{x2\_minus\_y2\_plus\_z} \\
\hline
$x^2+2y^2+3z^2+x+y$ & \texttt{x2\_plus\_2yz\_plus\_3x2\_plus\_x\_plus\_y} \\
\hline
$x^2+yz+x$ & \texttt{xz\_plus\_yz\_plus\_x} \\
\hline
$x^2+yz+x+1$ & \texttt{x2\_plus\_yz\_plus\_x\_plus\_1} \\
\hline
$xy+x+y$ & \texttt{xy\_plus\_x\_plus\_y} \\
\hline
$y^2+xz+z-1$ & \texttt{yz\_plus\_xz\_plus\_z\_minus\_1} \\
\hline
\end{tabularx}
\caption{Collection of general prescribers with unknown realisability.}
\label{tab:unknown_exist_general}
\end{table}
\newpage

\paragraph{Known validity prescribers}\mbox{}\\
For prescribers whose validity is known analytically, the network behaviour is fully consistent with theoretical expectations. As an example, the prescribers \texttt{prop\_a}, \texttt{prop\_b}, and \texttt{prop\_c}, arising directly from Proposition~\ref{prop: Classification zonal spherical harmonics}, all yield terminal losses that lie well below the empirical cut-off threshold. These losses are comparable in magnitude to those observed for realisable spherical harmonics, suggesting that the network successfully constructs conformal factors whose induced curvature closely matches the prescribed function. The prescribers included here span a slew of qualitative behaviours starting simply with rotationally symmetric profiles to more complicated non-linear functions. This indicates the observed convergence is not restricted to a narrow functional class. The results are also supported visually for the \texttt{egg} and \texttt{sinusoidal} functions in Figure \ref{fig: general function plots}.  

On the other hand, prescribers that are known to violate necessary conditions, such as \texttt{negative\_dip} (strictly negative), and \texttt{nodal\_crossing} (derivative changes sign only at points where the function itself is consistently negative) produce losses of higher order than the known realisable prescribers. This mirrors the behaviour observed in Section~\ref{sub:results_spherical_harmonics} and reinforces the interpretation that elevated loss values correspond to genuine geometric obstructions, rather than to deficiencies in network expressivity or optimisation. 

Figure \ref{fig:general_gauss_bonnet} shows the Gauss-Bonnet violations for the more general class of functions; these results are reassuringly consistent with the findings in the spherical harmonic case. The \texttt{nodal\_crossing} and \texttt{monotonic\_exp} prescribers (which are closest to loss threshold) result in metrics with a Gauss-Bonnet violation 3 orders of magnitude larger than the metrics of known realisable prescribers. This provides definitive proof that the terminal metrics are not valid metrics on the sphere. Figure \ref{fig: general function plots} further supports this idea, as it shows that the learnt metric of  \texttt{nodal\_crossing} and \texttt{negative\_dip} has some pathological behaviour. 

\paragraph{Unknown validity prescribers}\mbox{}\\
The most interesting regime is formed by prescribers whose validity is \textit{not} known analytically. Beyond illustrating a clear separation between valid and invalid prescribers, Figure \ref{fig:general_loss_scatter} reveals a notable clustering phenomenon: prescribers believed to admit solutions not only fall below the threshold, but do so with losses of comparable magnitude. This clustering suggests that, once a solution exists, the difficulty of optimisation is largely insensitive to the detailed functional form of the prescriber. In contrast, prescribers without solutions populate a broad band of higher losses, with no indication of convergence toward the realisable regime even under extensive hyperparameter variation. 

A more granular inspection of the loss values in Figure~\ref{fig:general_loss_scatter} allows for tentative existence statements to be made. The \texttt{x2\_plus\_yz\_plus\_x\_plus\_1} prescriber achieve a minimum loss that lies comfortably below the empirical threshold from the cases where existence is known. Moreover, the loss is comparable in magnitude to those obtained for prescribers known to admit solutions. This close alignment suggests that, within the resolution afforded by the present network, this function likely admits conformal metrics whose Gaussian curvature realises the prescribed profile. For the \texttt{x2\_plus\_yz\_plus\_x\_plus\_1} prescriber, the Gauss-Bonnet test provides further reassurance with a relative error of 0.07\% deviation from the expected value. 

On the other hand, the best performing model for the \texttt{x2\_plus\_yz\_plus\_x} prescriber has a conformal loss which sits just below the threshold; alone, this is not compelling evidence that the prescriber yields a valid metric. The statistical error of other runs encroaches the region occupied by the \texttt{x2\_plus\_yz\_plus\_x} prescriber. Moreover, the Gauss-Bonnet test presents a large error, 5 orders of magnitude larger than the known existence cases, suggesting the Nirenberg architecture did not find a valid metric on the sphere. This is further supported by the visualization in Figure \ref{fig: general function plots} that suggest that the converged metric explodes at a point. From the ensemble of runs, it is unlikely the \texttt{x2\_plus\_yz\_plus\_x} prescriber yields a metric within the expressivity limits of the network which is able to solve the Nirenberg PDE. 

In contrast, a number of polynomial prescribers exhibit losses that lie significantly \emph{above} the empirical threshold, with no indication of convergence toward the realisable regime across the ensemble of hyperparameter configurations. Notably, the prescribers \texttt{y2\_plus\_xz\_plus\_z\_minus\_1}, \texttt{x2\_plus\_2y2\_plus\_3z2\_plus\_x\_plus\_y}, \texttt{xy\_plus\_x\_plus\_y}, and \texttt{x2\_minus\_y2\_plus\_z} all produce minimum losses that exceed the threshold by several orders of magnitude. Since Figure \ref{fig:general_loss_scatter} shows minimum losses and the sweep was performed over a series of model architectures, the elevated values for these prescribers are not attributable to optimisation noise or insufficient expressivity of the network. Figure \ref{fig:general_gauss_bonnet} shows each of these functions possesses a Gauss-Bonnet error orders of magnitude higher than the gap, further reaffirming that the terminal metric associated with the aforementioned prescribers does not present as a solution to the the Nirenberg PDE. In the case of \texttt{xy\_plus\_x\_plus\_y}, the visualization in Figure~\ref{fig: general function plots} provides additional evidence that the network has converged to a geometrically non-realisable solution. 

The persistent failure of the network to approximate these prescribers at a level comparable to known solutions provides compelling numerical evidence against the existence of conformal metrics realising the aforementioned curvature functions. While such results cannot constitute a proof of non-existence, the consistency of the outcome across multiple architectures and initialisations suggests that there may not exist a realisable metric for these cases. As in the spherical harmonic setting, the sharp contrast in loss magnitude relative to the calibrated threshold supports the interpretation of the neural network loss as a meaningful heuristic indicator of solvability for the Nirenberg problem, and as a means of identifying promising candidates for future analytical study.

\begin{figure}[t!]
    \centering
    \includegraphics[width=1.0\linewidth]{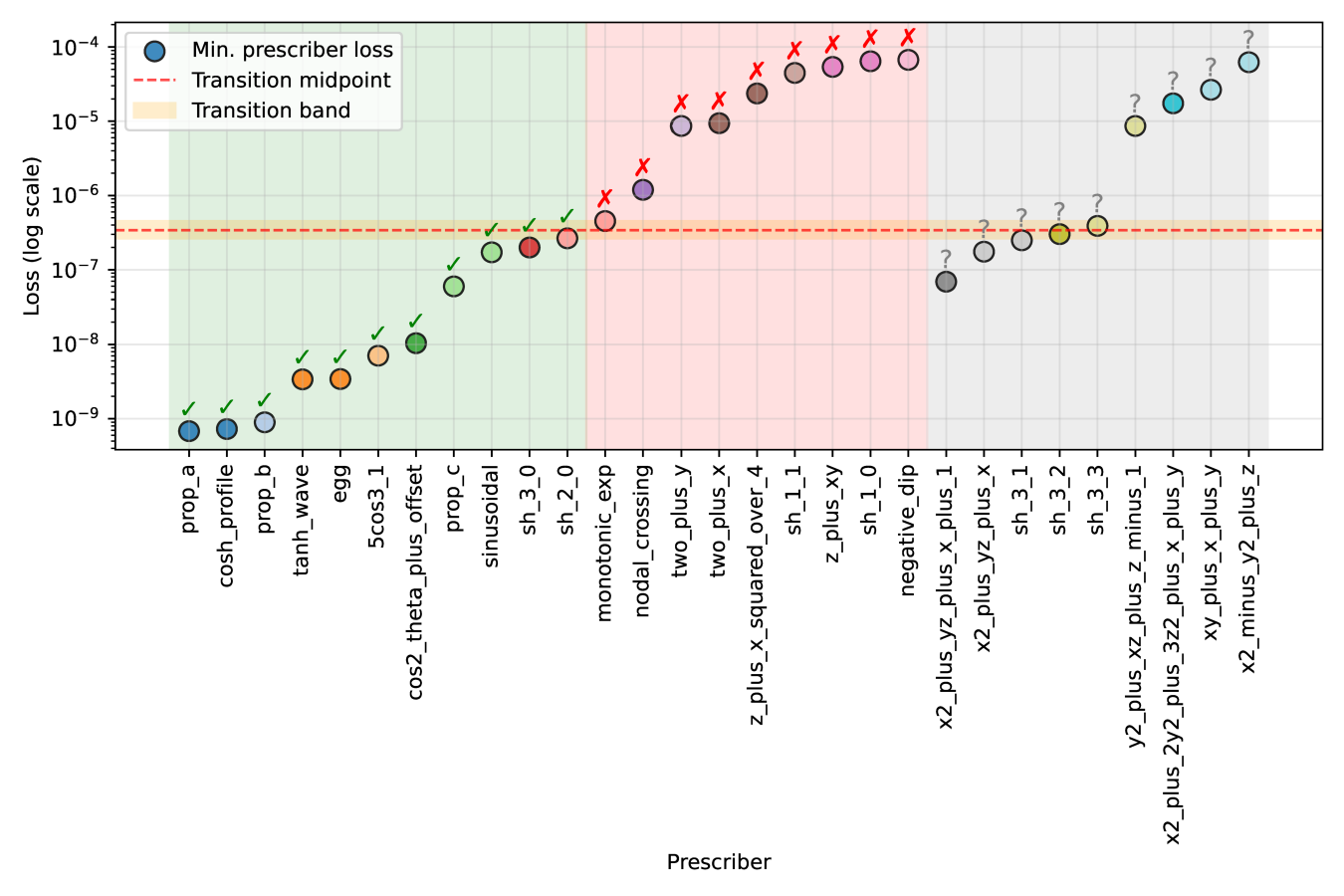}
    \caption{Minimum losses for general prescriber functions, including spherical harmonics (Tables \ref{tab:spherical_harmonics},  
    \ref{tab:known_exist_general}, \ref{tab:known_noexist_general}, and \ref{tab:unknown_exist_general}). Known realisable and unrealisable solutions are shown as ticks and crosses respectively. Unknown solutions are represented by `?'. The dashed line represents the empirically predicted transition of realisability.}
    \label{fig:general_loss_scatter}
\end{figure}

\begin{figure} [t!]
    \centering
    \includegraphics[width=0.99\linewidth]{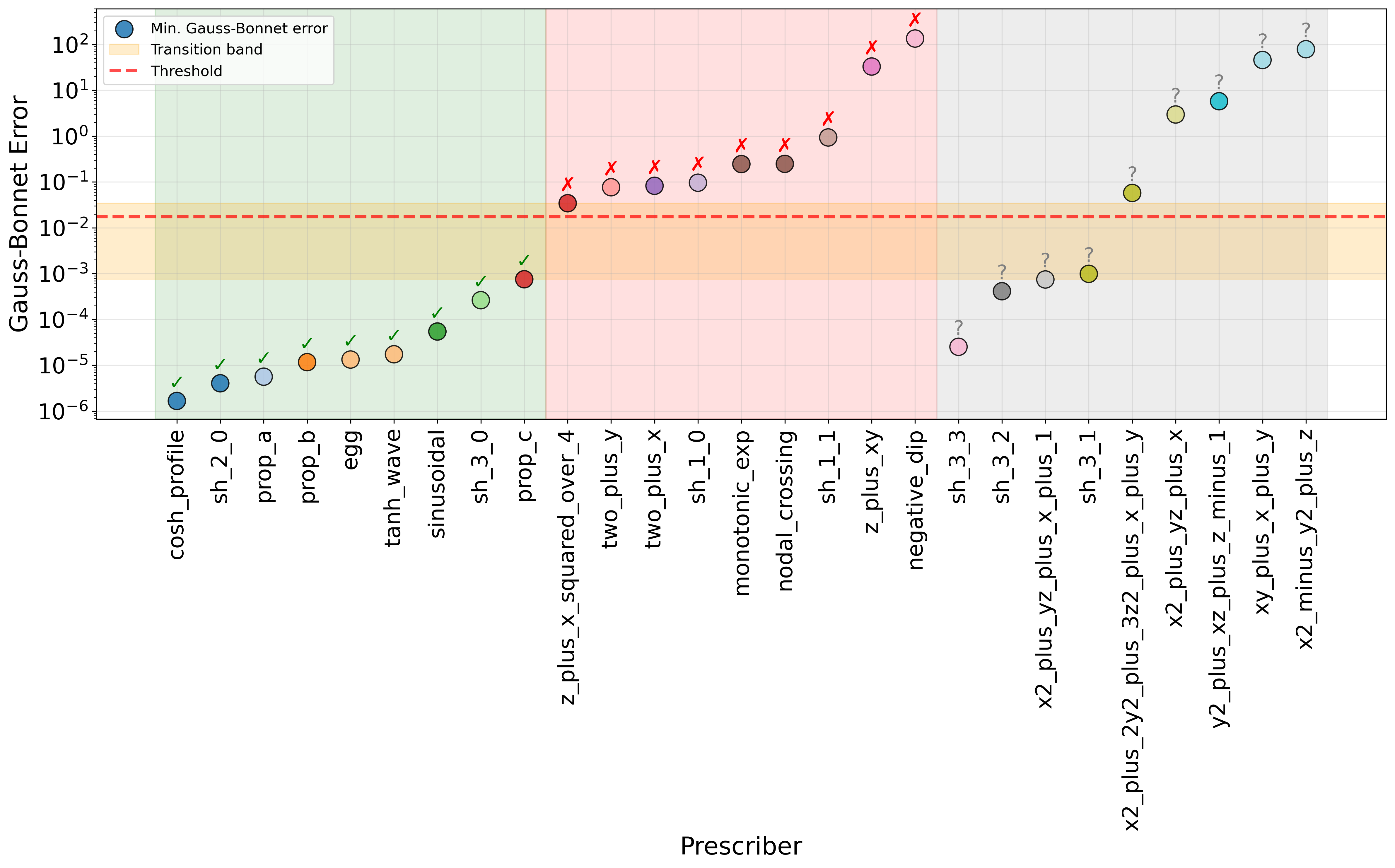}
    \caption{Gauss--Bonnet error for the general prescriber functions, including spherical harmonics (Tables \ref{tab:spherical_harmonics},  
    \ref{tab:known_exist_general}, \ref{tab:known_noexist_general}, and \ref{tab:unknown_exist_general}). These refer to the same runs as those in Figure \ref{fig:general_loss_scatter}, the style and notation is also the same.}
    \label{fig:general_gauss_bonnet}
\end{figure}

\begin{figure*}[ht!]
    \centering

    \noindent
    \colorbox{myLightgreen}{%
        \parbox{\dimexpr\textwidth-2\fboxsep}{%
            \centering\bfseries Egg
        }%
    }
    \vspace{1.0em}
    \begin{subfigure}[t]{0.22\textwidth}
        \centering
        \includegraphics[width=\linewidth]{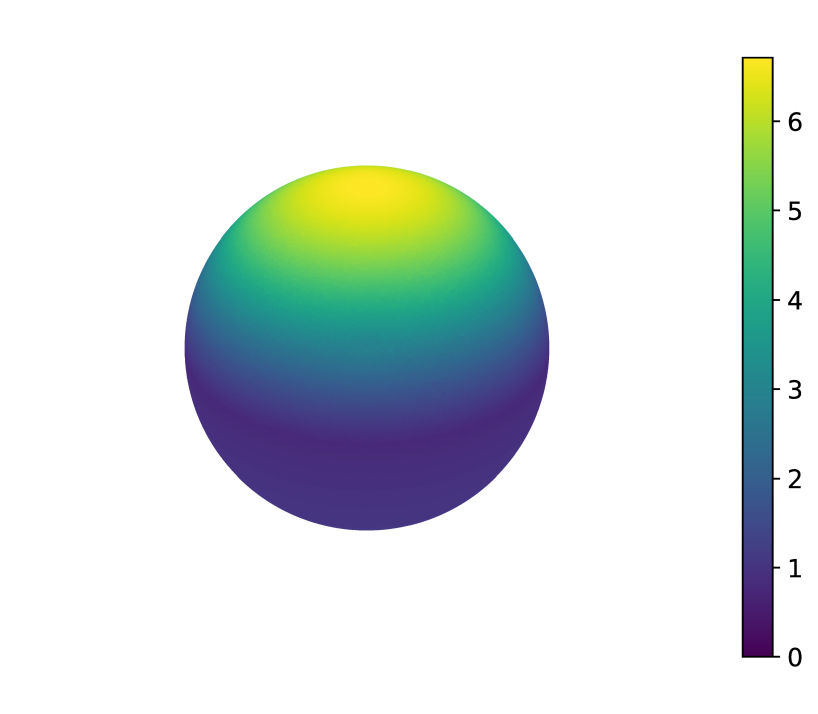}
        \caption{\texttt{Metric $g_{00}$}}
    \end{subfigure}
    \hfill
    \begin{subfigure}[t]{0.22\textwidth}
        \centering
        \includegraphics[width=\linewidth]{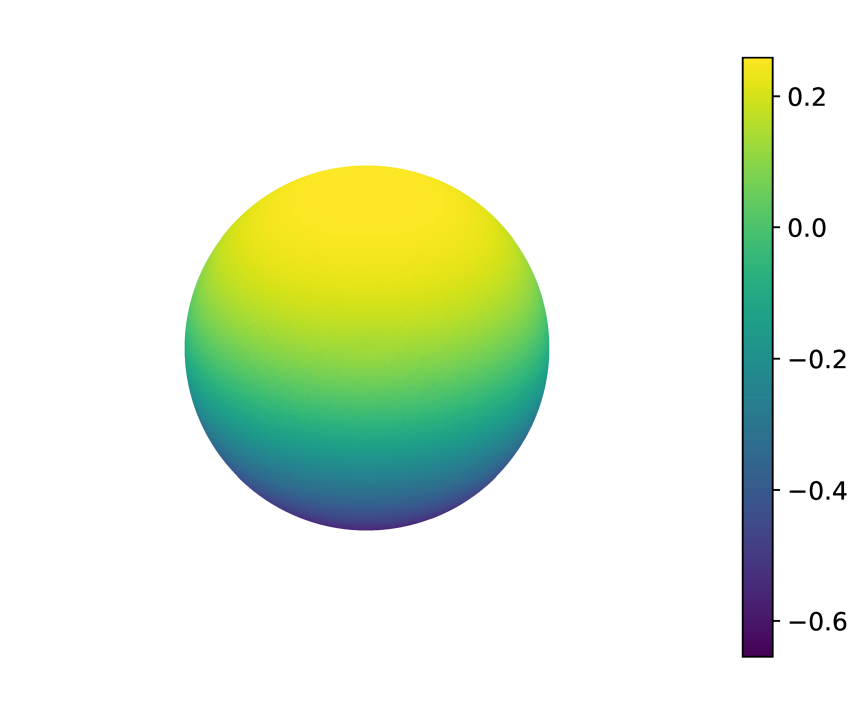}
        \caption{\texttt{u}}
    \end{subfigure}
    \hfill
    \begin{subfigure}[t]{0.22\textwidth}
        \centering
        \includegraphics[width=\linewidth]{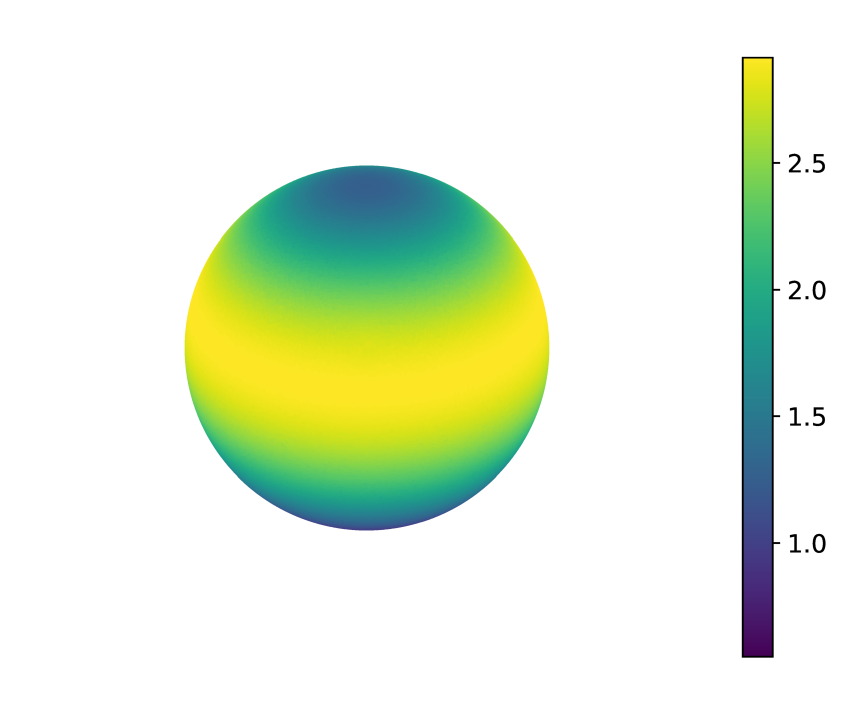}
        \caption{\texttt{Predicted R}}
    \end{subfigure}
        \hfill
    \begin{subfigure}[t]{0.22\textwidth}
        \centering
        \includegraphics[width=\linewidth]{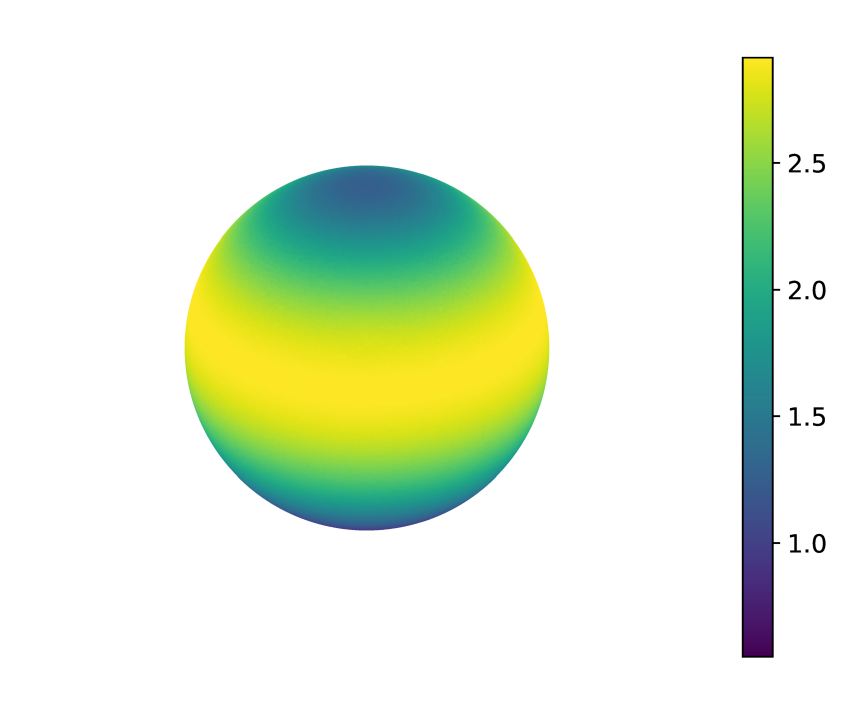}
        \caption{\texttt{Target R}}
    \end{subfigure}
    
    \vspace{1em}

    \noindent
    \colorbox{myLightgreen}{%
        \parbox{\dimexpr\textwidth-2\fboxsep}{%
            \centering\bfseries Sinusoidal
        }%
    }
    \vspace{1.0em}
        \begin{subfigure}[t]{0.22\textwidth}
        \centering
        \includegraphics[width=\linewidth]{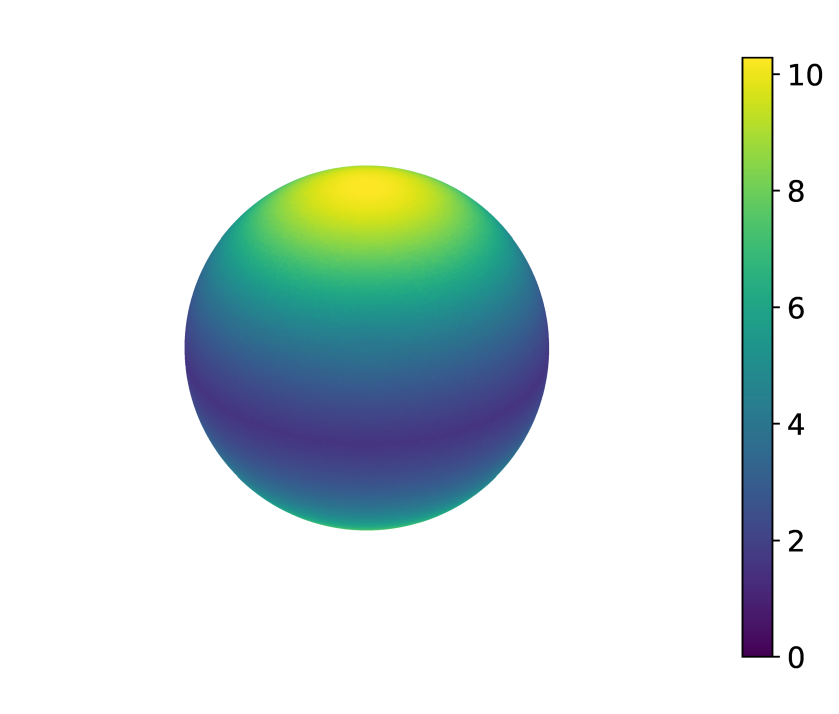}
        \caption{\texttt{Metric $g_{00}$}}
    \end{subfigure}
    \hfill
    \begin{subfigure}[t]{0.22\textwidth}
        \centering
        \includegraphics[width=\linewidth]{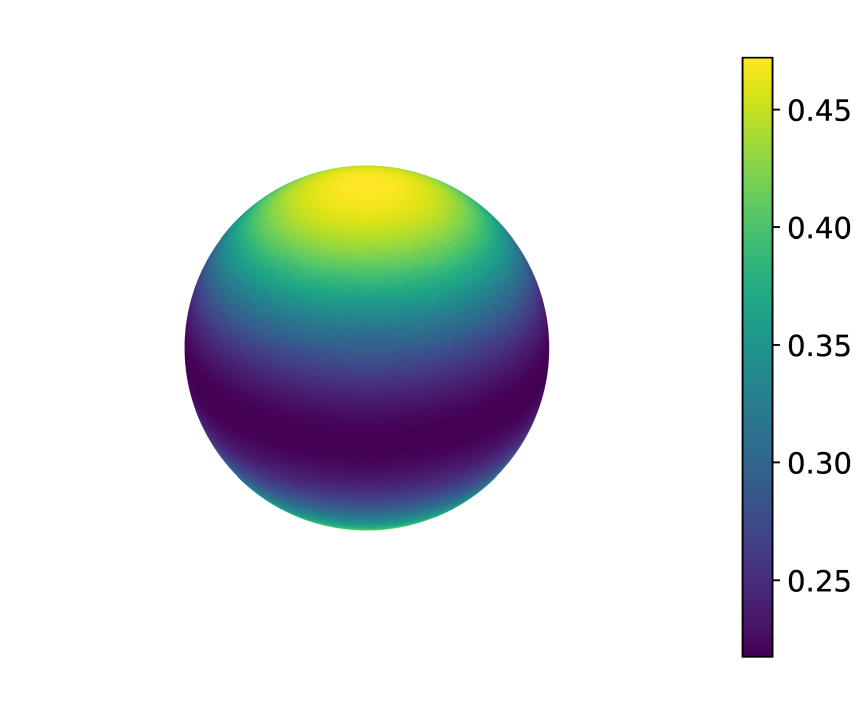}
        \caption{\texttt{u}}
    \end{subfigure}
    \hfill
    \begin{subfigure}[t]{0.22\textwidth}
        \centering
        \includegraphics[width=\linewidth]{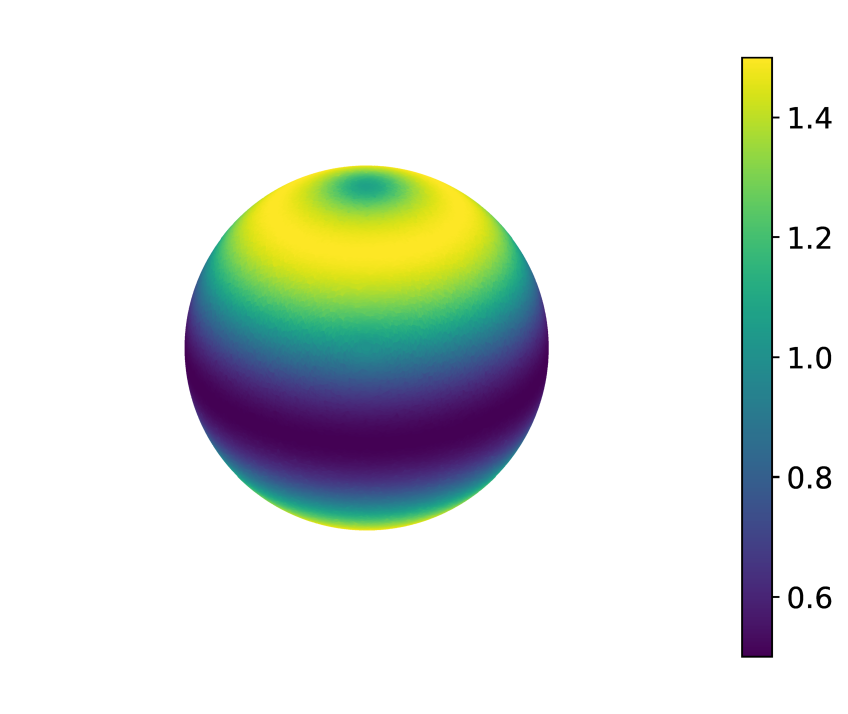}
        \caption{\texttt{Predicted R}}
    \end{subfigure}
        \hfill
    \begin{subfigure}[t]{0.22\textwidth}
        \centering
        \includegraphics[width=\linewidth]{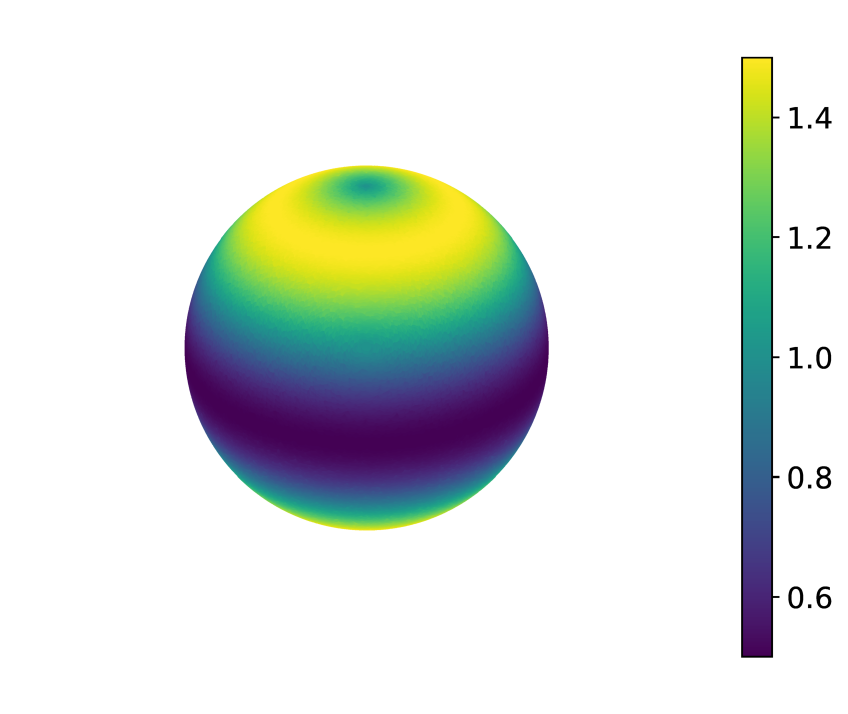}
        \caption{\texttt{Target R}}
    \end{subfigure}
\end{figure*}

\begin{figure*}[ht!]
    \centering

      \vspace{1em}
    \noindent
    \colorbox{red!15}{%
        \parbox{\dimexpr\textwidth-2\fboxsep}{%
            \centering\bfseries  \texttt{negative\_dip}
        }%
    }

    \vspace{1.0em}

        \begin{subfigure}[t]{0.22\textwidth}
        \centering
        \includegraphics[width=\linewidth]{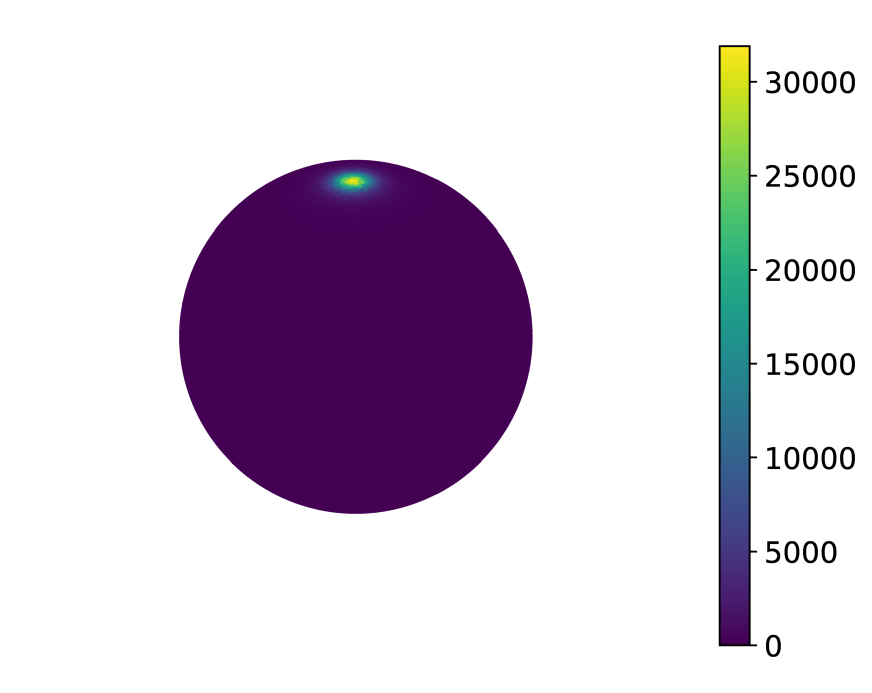}
        \caption{\texttt{Metric $g_{00}$}}
    \end{subfigure}
    \hfill
    \begin{subfigure}[t]{0.22\textwidth}
        \centering
        \includegraphics[width=\linewidth]{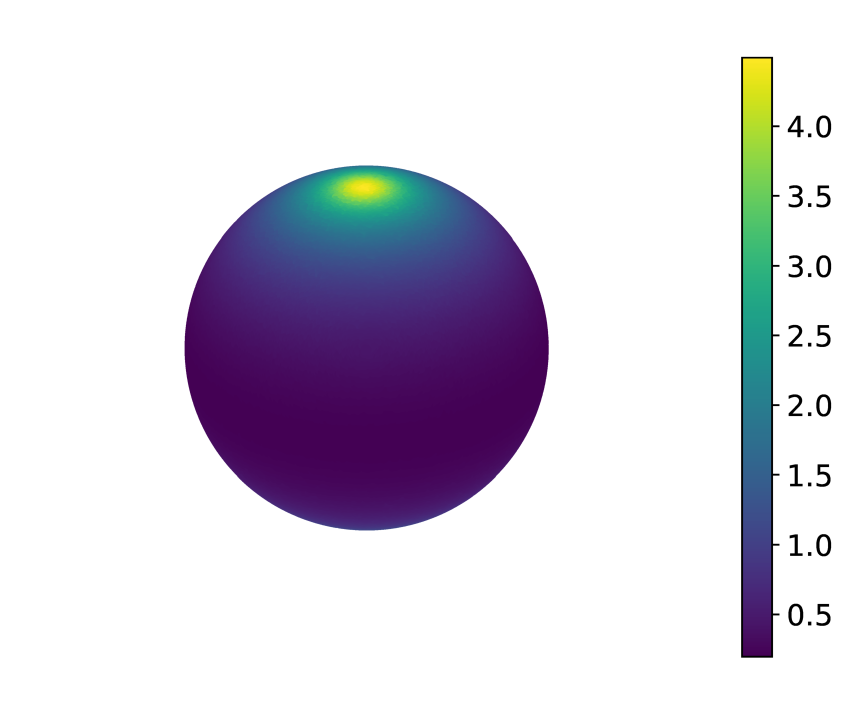}
        \caption{\texttt{u}}
    \end{subfigure}
    \hfill
    \begin{subfigure}[t]{0.22\textwidth}
        \centering
        \includegraphics[width=\linewidth]{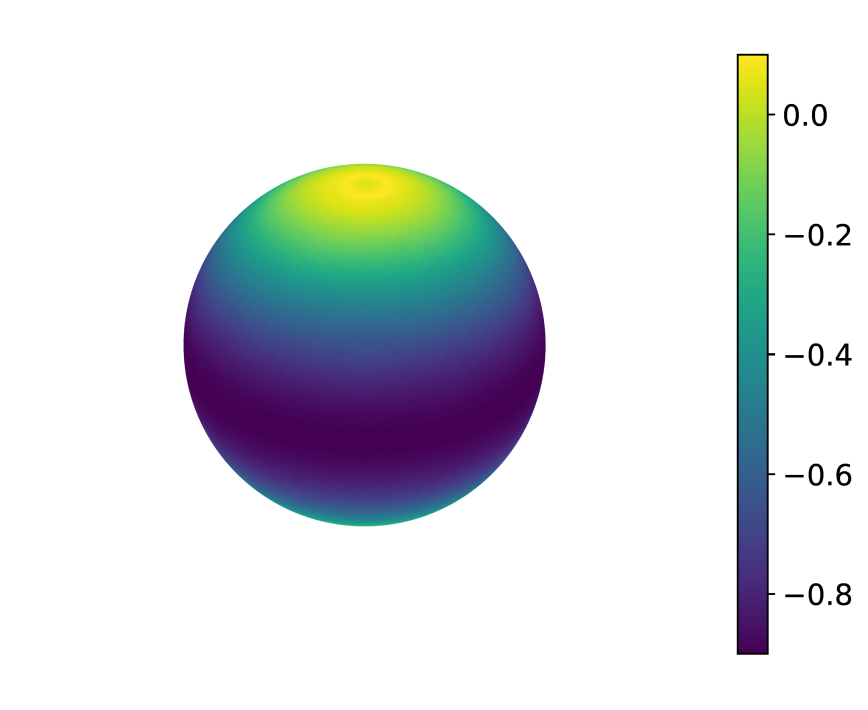}
        \caption{\texttt{Predicted R}}
    \end{subfigure}
        \hfill
    \begin{subfigure}[t]{0.22\textwidth}
        \centering
        \includegraphics[width=\linewidth]{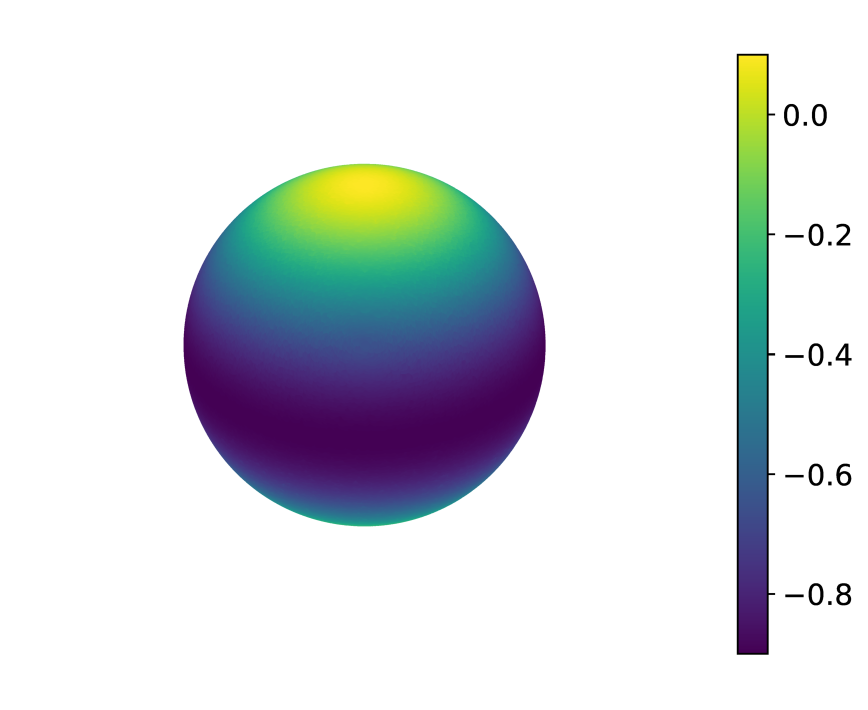}
        \caption{\texttt{Target R}}
    \end{subfigure}

      \vspace{1em}

    \noindent
    \colorbox{red!15}{%
        \parbox{\dimexpr\textwidth-2\fboxsep}{%
            \centering\bfseries \texttt{nodal\_crossing}
        }%
    }
    \vspace{1.0em}

        \begin{subfigure}[t]{0.22\textwidth}
        \centering
        \includegraphics[width=\linewidth]{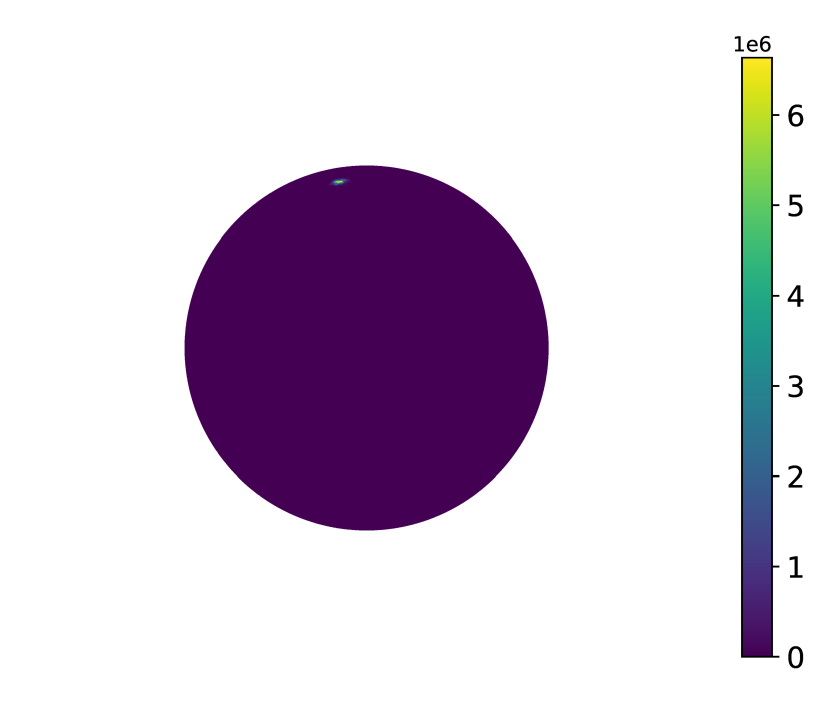}
        \caption{\texttt{Metric $g_{00}$}}
    \end{subfigure}
    \hfill
    \begin{subfigure}[t]{0.22\textwidth}
        \centering
        \includegraphics[width=\linewidth]{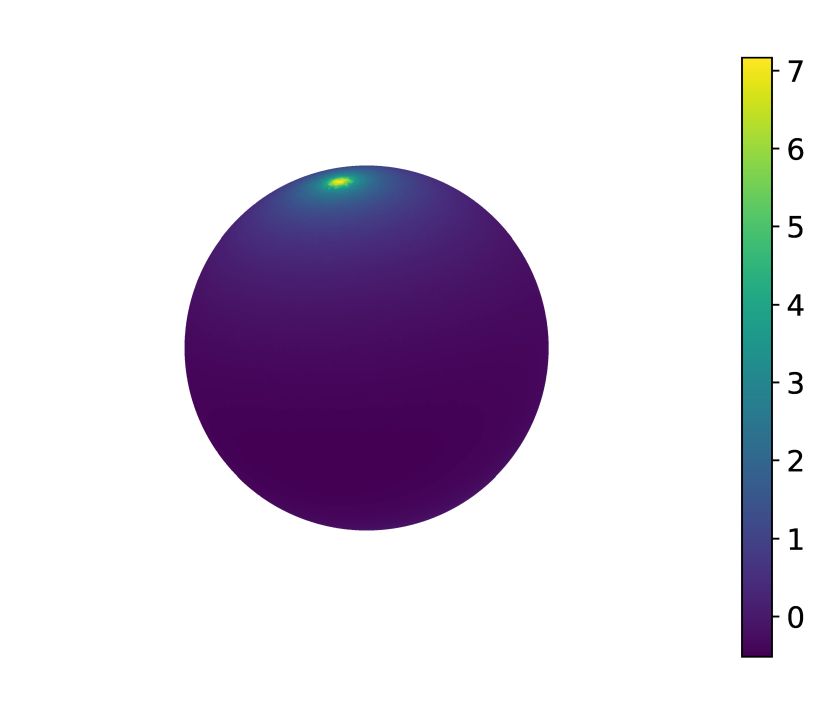}
        \caption{\texttt{u}}
    \end{subfigure}
    \hfill
    \begin{subfigure}[t]{0.22\textwidth}
        \centering
        \includegraphics[width=\linewidth]{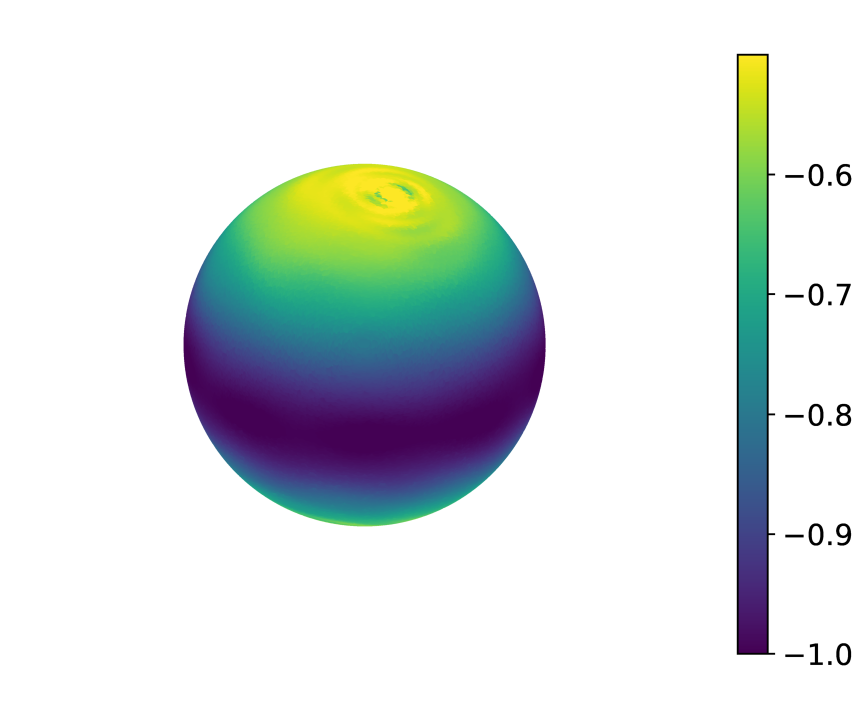}
        \caption{\texttt{Predicted R}}
    \end{subfigure}
        \hfill
    \begin{subfigure}[t]{0.22\textwidth}
        \centering
        \includegraphics[width=\linewidth]{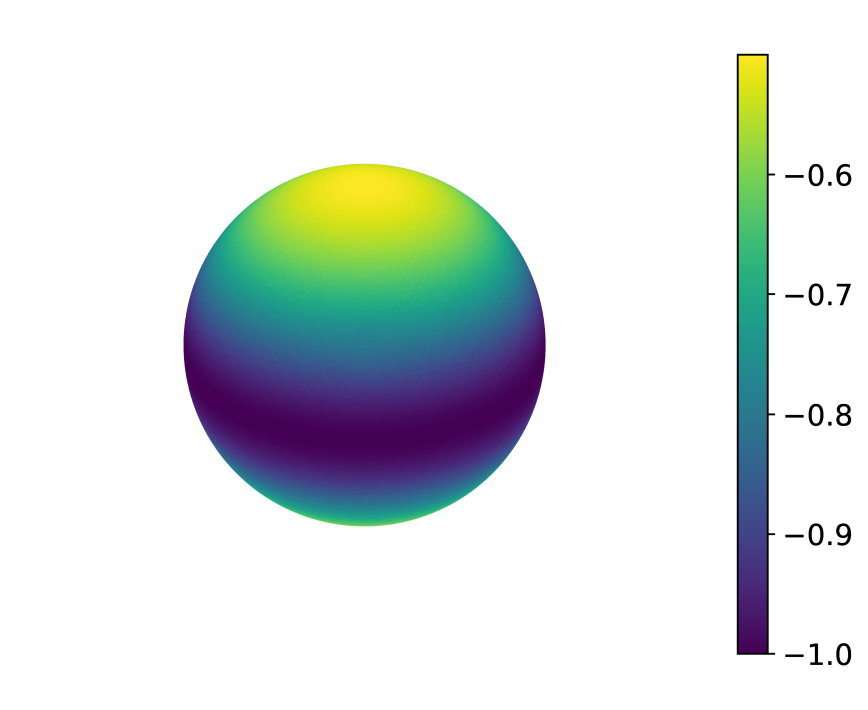}
        \caption{\texttt{Target R}}
    \end{subfigure}

    \vspace{1.0em}

    \noindent
    \colorbox{myLightgrey}{%
        \parbox{\dimexpr\textwidth-2\fboxsep}{%
            \centering\bfseries \texttt{x2\_plus\_yz\_plus\_x}
        }%
    }    
    \vspace{1.0em}
        \begin{subfigure}[t]{0.22\textwidth}
        \centering
        \includegraphics[width=\linewidth]{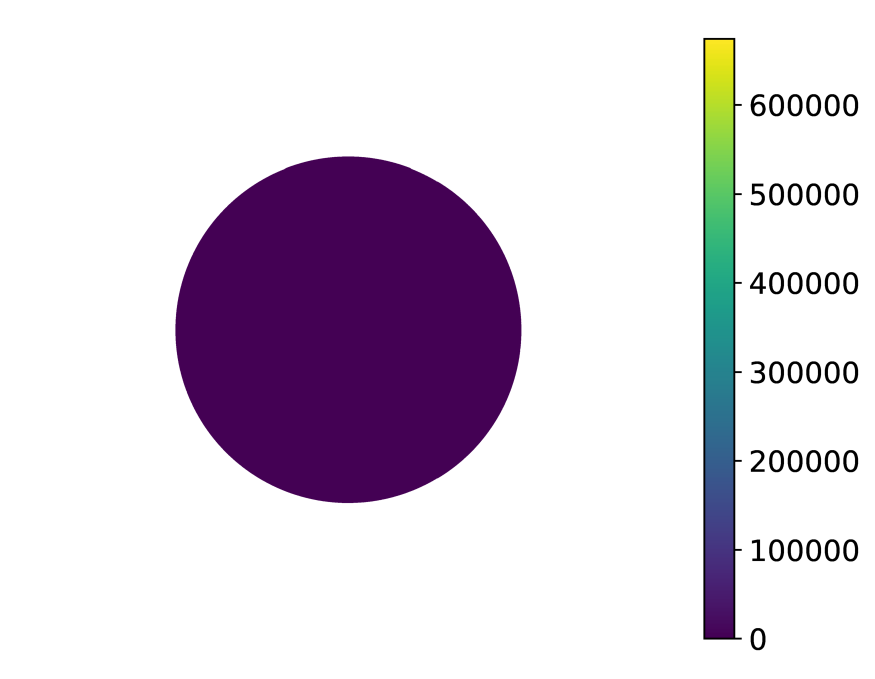}
        \caption{\texttt{Metric $g_{00}$}}
    \end{subfigure}
    \hfill
    \begin{subfigure}[t]{0.22\textwidth}
        \centering
        \includegraphics[width=\linewidth]{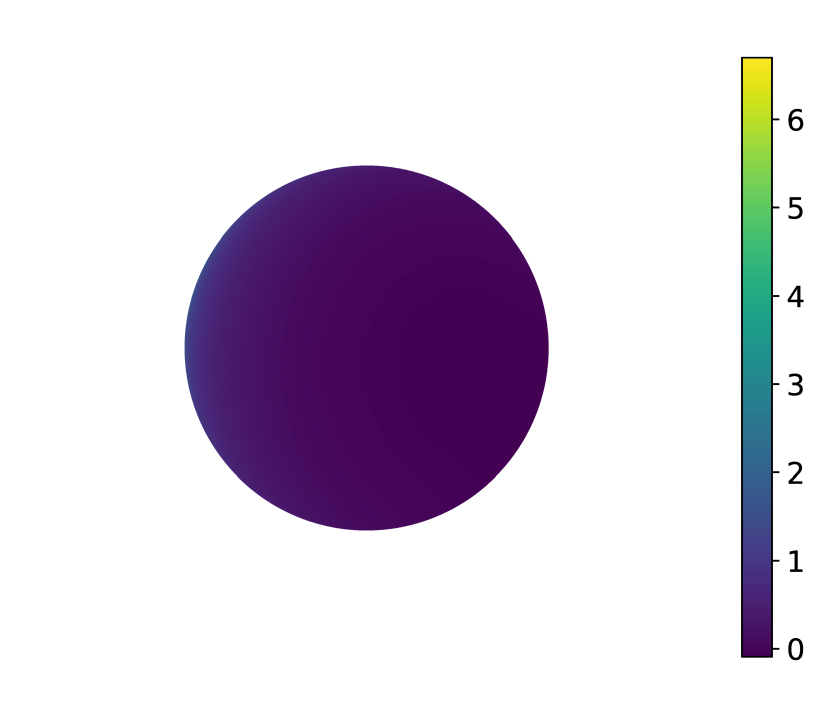}
        \caption{\texttt{u}}
    \end{subfigure}
    \hfill
    \begin{subfigure}[t]{0.22\textwidth}
        \centering
        \includegraphics[width=\linewidth]{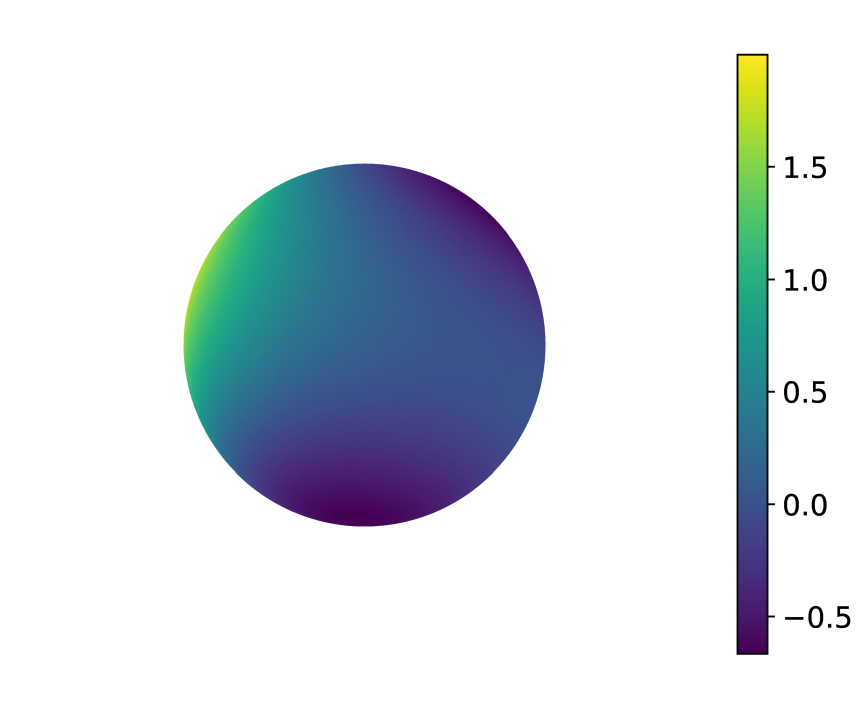}
        \caption{\texttt{Predicted R}}
    \end{subfigure}
        \hfill
    \begin{subfigure}[t]{0.22\textwidth}
        \centering
        \includegraphics[width=\linewidth]{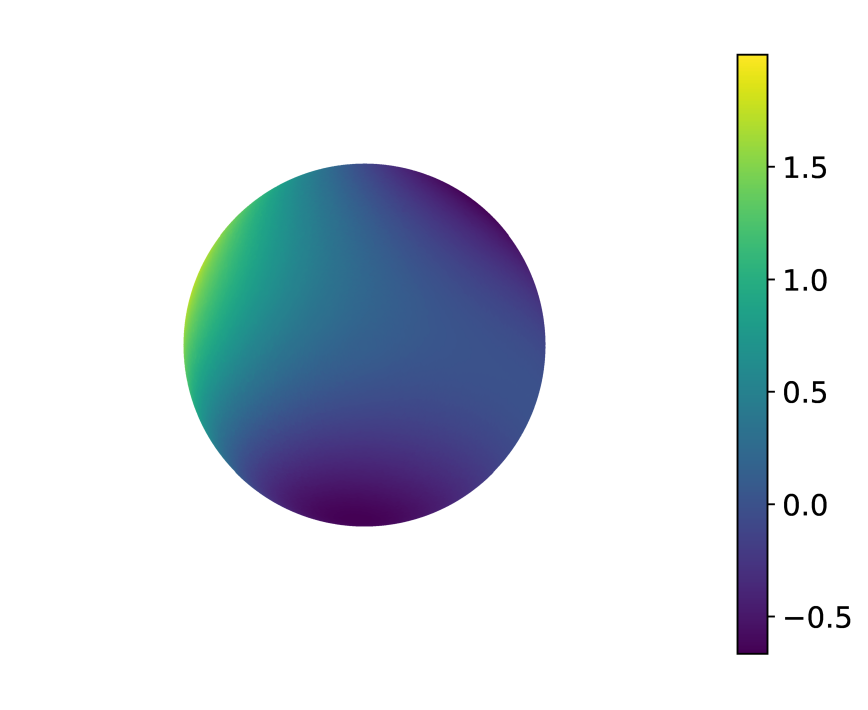}
        \caption{\texttt{Target R}}
    \end{subfigure}
   
      \vspace{1em}

    \noindent
    \colorbox{myLightgrey}{%
        \parbox{\dimexpr\textwidth-2\fboxsep}{%
            \centering\bfseries \texttt{xy\_plus\_x\_plus\_y}
        }%
    }    
    \vspace{1.0em}
        \begin{subfigure}[t]{0.22\textwidth}
        \centering
        \includegraphics[width=\linewidth]{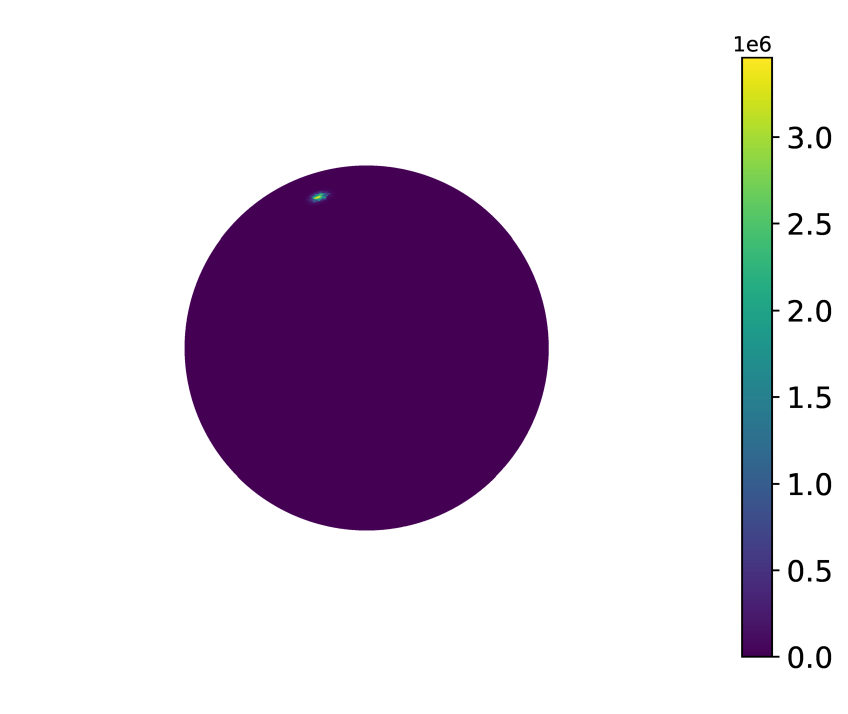}
        \caption{\texttt{Metric $g_{00}$}}
    \end{subfigure}
    \hfill
    \begin{subfigure}[t]{0.22\textwidth}
        \centering
        \includegraphics[width=\linewidth]{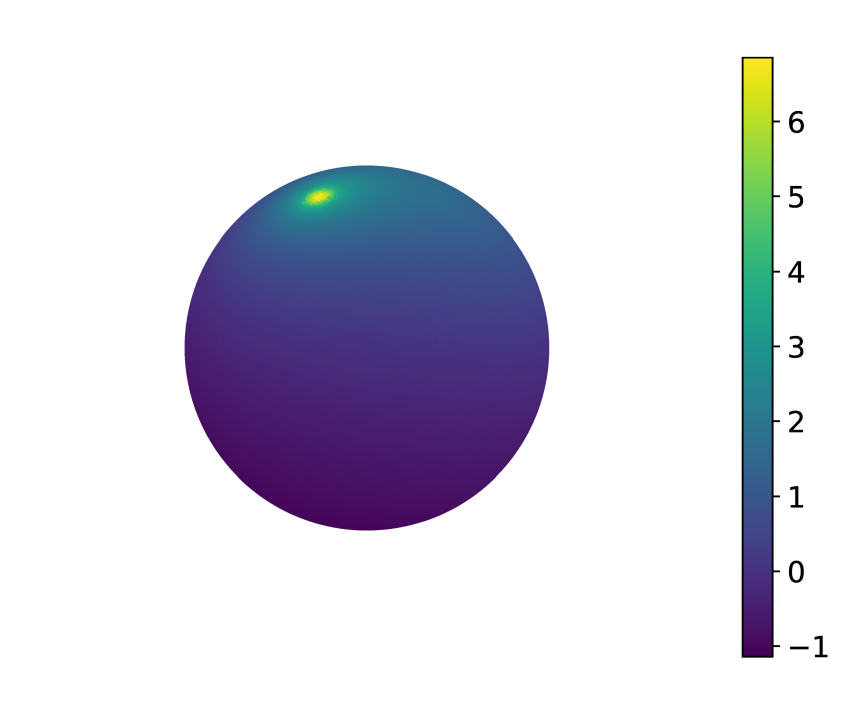}
        \caption{\texttt{u}}
    \end{subfigure}
    \hfill
    \begin{subfigure}[t]{0.22\textwidth}
        \centering
        \includegraphics[width=\linewidth]{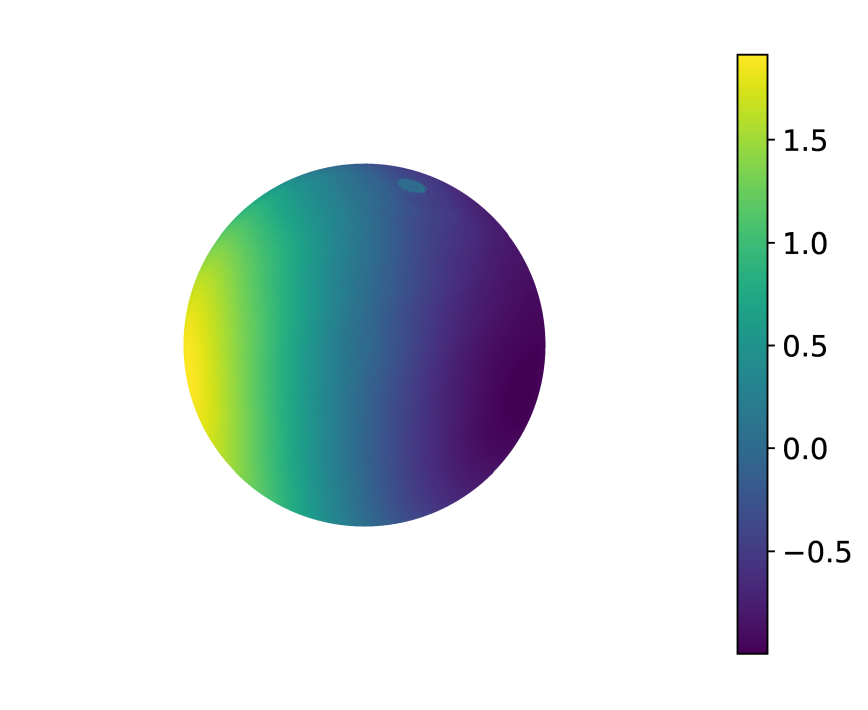}
        \caption{\texttt{Predicted R}}
    \end{subfigure}
        \hfill
    \begin{subfigure}[t]{0.22\textwidth}
        \centering
        \includegraphics[width=\linewidth]{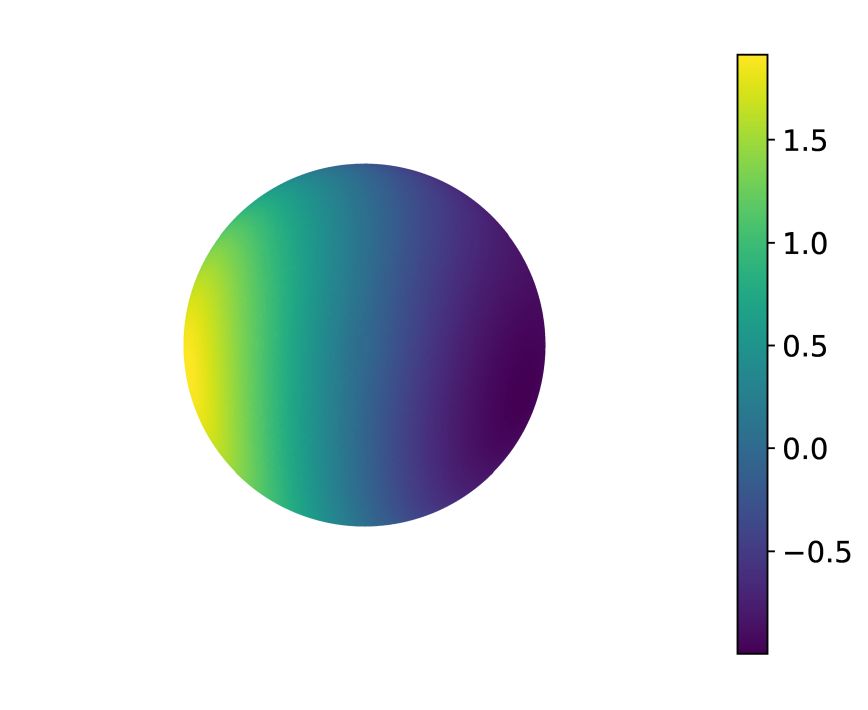}
        \caption{\texttt{Target R}}
    \end{subfigure}
    \caption{Metric $g_{00}$, conformal factor $u$, and predicted and target curvature $R$.} 
    \label{fig: general function plots}
\end{figure*}

\clearpage
\subsection{Interpretable Harmonic Expansion Approximations}
\label{sub:sh_expansion}
The result of the Nirenberg NN architecture training is models which approximate $u$, i.e. the conformal factor that yields the desired prescribed curvatures.
Distilling mathematical insight from these trained models is, however, still challenging. 

In this subsection, the spherical harmonic expansion ansatz for $u$, studied in Section \ref{sec:spectralpair}, is truncated and fitted to these trained NN models.
This truncated ansatz takes the form
\begin{equation}
    \tilde{u} = \sum_{\ell\le L}\sum_{m=-\ell}^{\ell} \frac{c_{\ell, m}}{\ell(\ell + 1)}Y_{\ell, m}\;,
\end{equation}
up to some chosen maximum order $L$.
Each $\tilde{u}$ is a smooth conformal factor on $S^2$ and hence produces an explicit smooth intrinsic metric 
$g_{\tilde{u}} = e^{2\tilde{u}}g_{\mathrm{round}}$.
Therefore, fitting a $\tilde{u}$ ansatz to a trained NN model provides an interpretable approximation of the learnt $u$, for which geometric quantities such as area, Laplace--Beltrami operators, and Gaussian curvature can be evaluated analytically and globally. 

The method works via first generating a new sample on $S^2$, for consistency 20,000 points were used. 
The trained NN model to be approximated with this ansatz is then imported, and the $\tilde{u}$ ansatz is created for a chosen $L$.
In these investigations $L = 4$ was used, which corresponds to 24 $c_{\ell,m}$ coefficients; these coefficients are then the parameters of $\tilde{u}$ that are to be optimised in the fitting process.
Then, over 500 epochs (with a 50 epoch early stopping patience), each batch of sample points has its $u$ value predicted by both the NN model, and the current form of the $\tilde{u}$ ansatz.
The MSE difference\footnote{The codebase also provides functionality for direct learning of these ansatz coefficients using the same NN loss in Equation~\eqref{eq:nn_loss}, without the intermediate NN. Loss performances were comparable with longer compute times, and additionally were worse than the NN learning losses.} in the $u$ predictions form the loss function for the fitting, with a minor L2 regularisation component.
This loss is minimised in a supervised manner using the Adam optimiser, with a cosine learning rate scheduler.

When training is completed, the set of 24 $c_{\ell, m}$ coefficients provide a directly interpretable approximation of the true $u$ which produces the desired prescribed curvature.
Additionally, since this learnt ansatz is of the exact form as Proposition \ref{prop: spectral pairs}, the prescribed curvature of the $\tilde{u}$ approximation, $\tilde{R}$, can be read off directly from the learnt coefficients.
Using this Spectral Pairs method of efficiently computing $\tilde{R}$ for the learnt $\tilde{u}$ ansatz, $\tilde{R}$ can be compared to the $R$ value of the original prescriber used to train the NN, on a new test sample.
This comparison produces another MSE measure at the level of the curvature, which quantifies how similar this learnt $\tilde{u}$ ansatz is to producing this desired curvature.

To first validate the harmonic expansion ansatz fitting, the NN models were fitted to the three Spectral Pair prescribers, \{\texttt{prop\_a}, \texttt{prop\_b}, \texttt{prop\_c}\}.
These are known to have valid and complete descriptions in terms of finitely many non-zero $c_{\ell,m}$ coefficients, due to their Spectral Pair construction process.
Then, this fitting was run for two other studied prescribers, \{\texttt{sh\_3\_2}, \texttt{x2\_plus\_yz\_plus\_x\_plus1}\}, whose realisability as scalar curvature on $S^2$ is unknown, yet the Nirenberg NN results in Section \ref{sub:results_spherical_harmonics} and Section \ref{sub:results_general_funcs} indicate are likely to be realisable, yet with unknown $u$ conformal factors.

Table \ref{tab:shexpansion_data}, shows the results of these fitting runs on a new test sample of 2000 points.
Listing the MSE fitting loss between the ansatz predicted $u$ value and that predicted by the NN model, then the MSE differences of the $\tilde{R}$ curvature to the original prescriber, and the final fitted $c_{\ell, m}$ coefficients.
The values of $R$ and $\tilde{R}$ at the test sample points are also plotted over the $S^2$ fundamental domain, for each original prescriber and fitted $\tilde{u}$ ansatz, in Figures \ref{fig:shexpansion_Rs_props} and \ref{fig:shexpansion_Rs_unknown} respectively.

\begin{table}[ht!]
\addtolength{\leftskip}{-1cm}
\addtolength{\rightskip}{-1cm}
\centering
\begin{tabular}{|cc|ccccc|}
\hline
\multicolumn{2}{|c|}{\multirow{2}{*}{Measure}} & \multicolumn{5}{c|}{Prescriber}  \\ \cline{3-7} 
\multicolumn{2}{|c|}{} & \multicolumn{1}{c|}{\texttt{prop\_a}} & \multicolumn{1}{c|}{\texttt{prop\_b}} & \multicolumn{1}{c|}{\texttt{prop\_c}} & \multicolumn{1}{c|}{\texttt{sh\_3\_2}} & \begin{tabular}[c]{@{}c@{}}\texttt{x2\_plus\_yz}\\ \texttt{\_plus\_x\_plus1}\end{tabular} \\ \hline

\multicolumn{2}{|c|}{\begin{tabular}[c]{@{}c@{}}Supervised \\ Loss ($u$)\end{tabular}}   
& \multicolumn{1}{c|}{$1.00 \times 10^{-9}$} 
& \multicolumn{1}{c|}{$5.02 \times 10^{-8}$} 
& \multicolumn{1}{c|}{$2.61 \times 10^{-5}$} 
& \multicolumn{1}{c|}{$1.25 \times 10^0$}
& $8.19 \times 10^{-2}$\\ \hline

\multicolumn{2}{|c|}{\begin{tabular}[c]{@{}c@{}}Curvature\\ MSE ($R$)\end{tabular}} 
& \multicolumn{1}{c|}{$7.57 \times 10^{-10}$} 
& \multicolumn{1}{c|}{$9.25 \times 10^{-7}$} 
& \multicolumn{1}{c|}{$2.24 \times 10^{-1}$} 
& \multicolumn{1}{c|}{$1.36 \times 10^{1}$} 
& $7.31 \times 10^0$ \\ \hline 
\multicolumn{1}{|c|}{\multirow{24}{*}{\rotatebox{90}{Learnt Coefficients}}}
& $c_{1,-1}$ & \multicolumn{1}{c|}{0} & \multicolumn{1}{c|}{-0.018} & \multicolumn{1}{c|}{-0.119} & \multicolumn{1}{c|}{0} & 0.133 \\ \cline{2-7}
\multicolumn{1}{|c|}{} & $c_{1,0}$  & \multicolumn{1}{c|}{0} & \multicolumn{1}{c|}{0.038}  & \multicolumn{1}{c|}{0} & \multicolumn{1}{c|}{-1.263} & -0.109 \\ \cline{2-7}
\multicolumn{1}{|c|}{} & $c_{1,1}$  & \multicolumn{1}{c|}{0} & \multicolumn{1}{c|}{\textbf{0.121}}  & \multicolumn{1}{c|}{\textbf{2.02}} & \multicolumn{1}{c|}{1.705}  & 4.526 \\ \cline{2-7}
\multicolumn{1}{|c|}{} & $c_{2,-2}$ & \multicolumn{1}{c|}{0} & \multicolumn{1}{c|}{0}      & \multicolumn{1}{c|}{-0.275} & \multicolumn{1}{c|}{0} & 0.199 \\ \cline{2-7}
\multicolumn{1}{|c|}{} & $c_{2,-1}$ & \multicolumn{1}{c|}{0} & \multicolumn{1}{c|}{0}      & \multicolumn{1}{c|}{0} & \multicolumn{1}{c|}{-0.127} & -0.314 \\ \cline{2-7}
\multicolumn{1}{|c|}{} & $c_{2,0}$  & \multicolumn{1}{c|}{\textbf{1.000}} & \multicolumn{1}{c|}{0} & \multicolumn{1}{c|}{0} & \multicolumn{1}{c|}{0} & -2.335 \\ \cline{2-7}
\multicolumn{1}{|c|}{} & $c_{2,1}$  & \multicolumn{1}{c|}{0} & \multicolumn{1}{c|}{0.002} & \multicolumn{1}{c|}{0} & \multicolumn{1}{c|}{2.352} & -0.227 \\ \cline{2-7}
\multicolumn{1}{|c|}{} & $c_{2,2}$  & \multicolumn{1}{c|}{0} & \multicolumn{1}{c|}{\textbf{0.054}} & \multicolumn{1}{c|}{\textbf{6.014}} & \multicolumn{1}{c|}{-1.532} & 4.078 \\ \cline{2-7}
\multicolumn{1}{|c|}{} & $c_{3,-3}$ & \multicolumn{1}{c|}{0} & \multicolumn{1}{c|}{0} & \multicolumn{1}{c|}{0.380} & \multicolumn{1}{c|}{0} & 0.246 \\ \cline{2-7}
\multicolumn{1}{|c|}{} & $c_{3,-2}$ & \multicolumn{1}{c|}{0} & \multicolumn{1}{c|}{0} & \multicolumn{1}{c|}{0} & \multicolumn{1}{c|}{0} & -0.310 \\ \cline{2-7}
\multicolumn{1}{|c|}{} & $c_{3,-1}$ & \multicolumn{1}{c|}{0} & \multicolumn{1}{c|}{0} & \multicolumn{1}{c|}{0} & \multicolumn{1}{c|}{0} & 0 \\ \cline{2-7}
\multicolumn{1}{|c|}{} & $c_{3,0}$  & \multicolumn{1}{c|}{0} & \multicolumn{1}{c|}{0} & \multicolumn{1}{c|}{0} & \multicolumn{1}{c|}{1.091} & 0.053 \\ \cline{2-7}
\multicolumn{1}{|c|}{} & $c_{3,1}$  & \multicolumn{1}{c|}{0} & \multicolumn{1}{c|}{\textbf{0.028}} & \multicolumn{1}{c|}{\textbf{11.775}} & \multicolumn{1}{c|}{0.870} & -2.641 \\ \cline{2-7}
\multicolumn{1}{|c|}{} & $c_{3,2}$  & \multicolumn{1}{c|}{0} & \multicolumn{1}{c|}{0} & \multicolumn{1}{c|}{0} & \multicolumn{1}{c|}{8.715} & -0.242 \\ \cline{2-7}
\multicolumn{1}{|c|}{} & $c_{3,3}$  & \multicolumn{1}{c|}{0} & \multicolumn{1}{c|}{0} & \multicolumn{1}{c|}{0} & \multicolumn{1}{c|}{0.808} & 3.324 \\ \cline{2-7}
\multicolumn{1}{|c|}{} & $c_{4,-4}$ & \multicolumn{1}{c|}{0} & \multicolumn{1}{c|}{0} & \multicolumn{1}{c|}{-0.277} & \multicolumn{1}{c|}{-0.091} & 0.284 \\ \cline{2-7}
\multicolumn{1}{|c|}{} & $c_{4,-3}$ & \multicolumn{1}{c|}{0} & \multicolumn{1}{c|}{0} & \multicolumn{1}{c|}{0} & \multicolumn{1}{c|}{-0.170} & -0.235 \\ \cline{2-7}
\multicolumn{1}{|c|}{} & $c_{4,-2}$ & \multicolumn{1}{c|}{0} & \multicolumn{1}{c|}{0} & \multicolumn{1}{c|}{0.322} & \multicolumn{1}{c|}{0.122} & -0.107 \\ \cline{2-7}
\multicolumn{1}{|c|}{} & $c_{4,-1}$ & \multicolumn{1}{c|}{0} & \multicolumn{1}{c|}{0} & \multicolumn{1}{c|}{0} & \multicolumn{1}{c|}{-0.196} & 0.166 \\ \cline{2-7}
\multicolumn{1}{|c|}{} & $c_{4,0}$  & \multicolumn{1}{c|}{0} & \multicolumn{1}{c|}{0} & \multicolumn{1}{c|}{0} & \multicolumn{1}{c|}{-3.145} & 1.471 \\ \cline{2-7}
\multicolumn{1}{|c|}{} & $c_{4,1}$  & \multicolumn{1}{c|}{0} & \multicolumn{1}{c|}{0} & \multicolumn{1}{c|}{0} & \multicolumn{1}{c|}{-0.603} & 0.081 \\ \cline{2-7}
\multicolumn{1}{|c|}{} & $c_{4,2}$  & \multicolumn{1}{c|}{0} & \multicolumn{1}{c|}{0} & \multicolumn{1}{c|}{0} & \multicolumn{1}{c|}{-2.050} & -1.941 \\ \cline{2-7}
\multicolumn{1}{|c|}{} & $c_{4,3}$  & \multicolumn{1}{c|}{0} & \multicolumn{1}{c|}{0} & \multicolumn{1}{c|}{0} & \multicolumn{1}{c|}{2.466} & -0.111 \\ \cline{2-7}
\multicolumn{1}{|c|}{} & $c_{4,4}$  & \multicolumn{1}{c|}{0} & \multicolumn{1}{c|}{0} & \multicolumn{1}{c|}{0} & \multicolumn{1}{c|}{1.825} & 2.650 \\ \hline
\end{tabular}
\caption{The learnt spherical harmonic expansion data for each considered prescribed curvature: \{\texttt{prop\_a}, \texttt{prop\_b}, \texttt{prop\_c}, \texttt{sh\_3\_2}, \texttt{x2\_plus\_yz\_plus\_x\_plus1}\}. Data first shows the expansion fitting MSE loss between the $u$ predicted by the pretrained NN model and that by the spherical harmonic expansion approximation being learnt. Then MSE between the true $R$ curvature of the prescriber and that for the learnt expansion using SpectralPair. Finally the learnt coefficients are given, where those which were <1\% of the maximum coefficient were set to 0, and those which are non-zero in the \texttt{prop\_} prescribers definitions are bolded.}
\label{tab:shexpansion_data}
\end{table}

For the Spectral Pair prescribers (\texttt{prop\_*}), the supervised losses are all low. Each has a sparse selection of relatively non-negligible learnt coefficients, and those which take non-zero values in the prescriber definition are bolded.
Each of these learnt bolded values are significant at their $\ell$ orders, and match the set prescribed values as given in Table \ref{tab:known_exist_general}.
This validates the expected behaviour of this process, where the true $u$ coefficients have been reproduced for the learnt $\tilde{u}$; especially for \texttt{prop\_a} which matches down to numerical precision.
Furthermore, Figure \ref{fig:shexpansion_Rs_props} corroborates these successful ansatz approximations, where the $R$ and $\tilde{R}$ distributions visually match. 

For the `unknown' prescribers, the supervised $u$ losses and curvature $R$ differences were not as strong, likely limited by the small $L$; although all saw reductions throughout training and demonstrated learning. 
Additionally, Figure \ref{fig:shexpansion_Rs_unknown} supports this observed learning, where there is clear visual similarity between $R$ and $\tilde{R}$.
The equivalent learnt coefficients in the Table also have varying scales, notably not all are used, providing some first order (or fourth order) interpretable information on the true form of the conformal factor $u$ functions that produce these likely realisable prescribed curvatures.

\begin{figure*}[ht!]
    \centering
    \begin{subfigure}[t]{0.48\textwidth}
        \centering
        \includegraphics[width=0.9\linewidth]{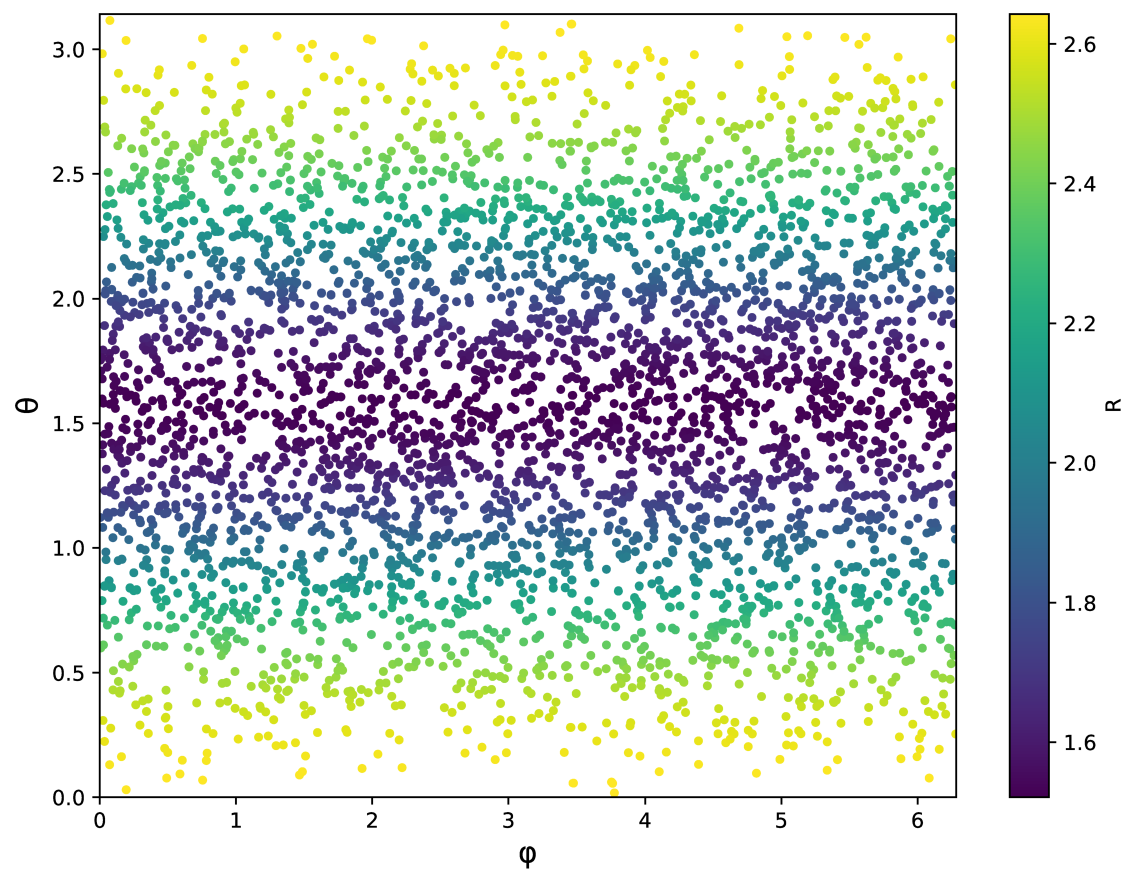}
        \caption{\texttt{prop\_a} $R$}
        \label{fig:shR_pa_p}
    \end{subfigure}
    \hfill
    \begin{subfigure}[t]{0.48\textwidth}
        \centering
        \includegraphics[width=0.9\linewidth]{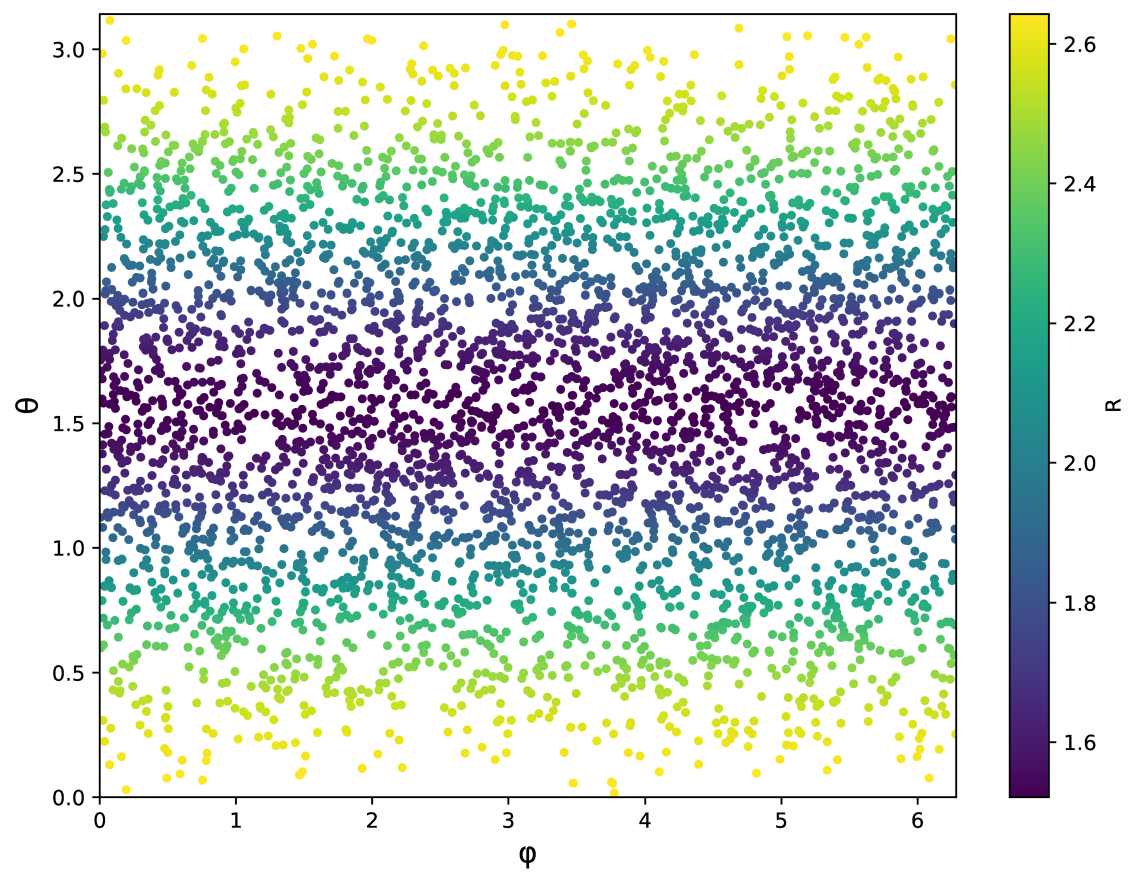}
        \caption{\texttt{prop\_a} $\tilde{R}$}
        \label{fig:shR_pa_sp}
    \end{subfigure}

    \vspace{0.5cm}

    \begin{subfigure}[t]{0.48\textwidth}
        \centering
        \includegraphics[width=0.9\linewidth]{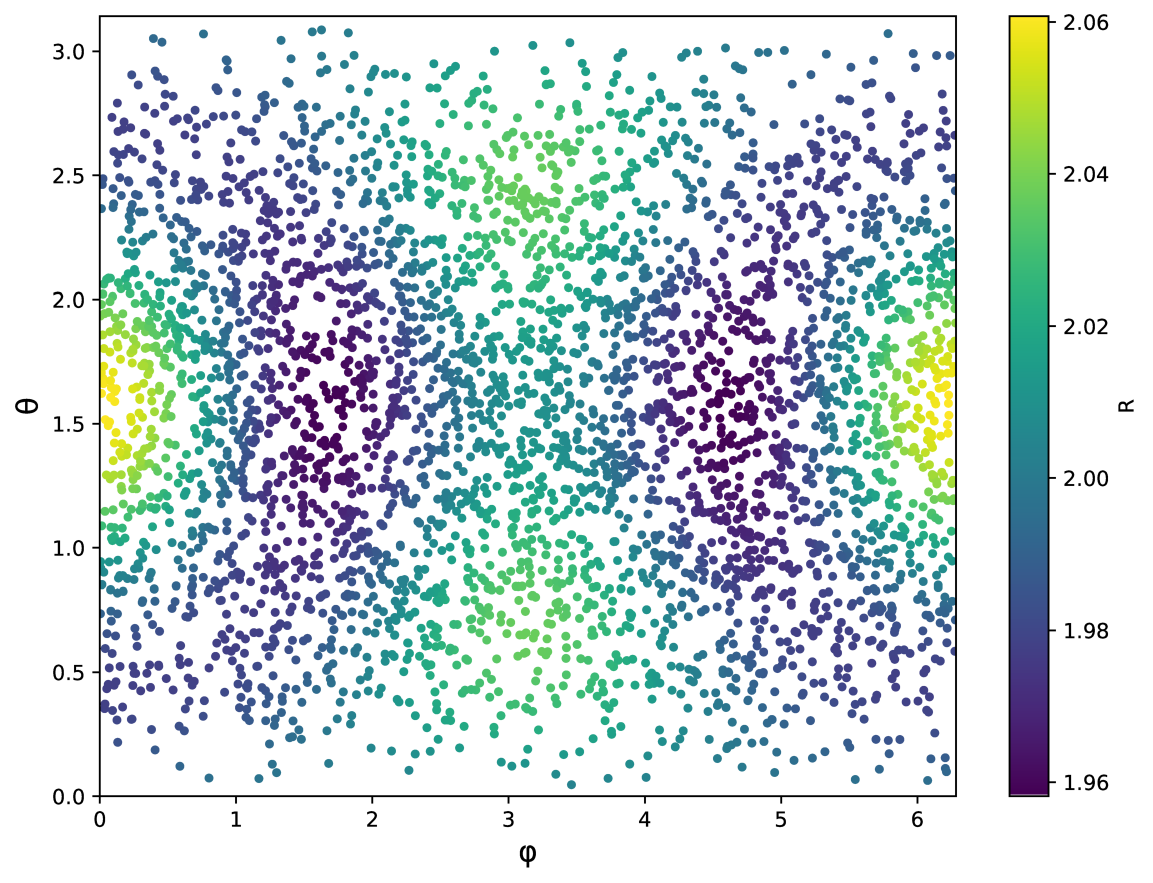}
        \caption{\texttt{prop\_b} $R$}
        \label{fig:shR_pb_p}
    \end{subfigure}
    \hfill
    \begin{subfigure}[t]{0.48\textwidth}
        \centering
        \includegraphics[width=0.9\linewidth]{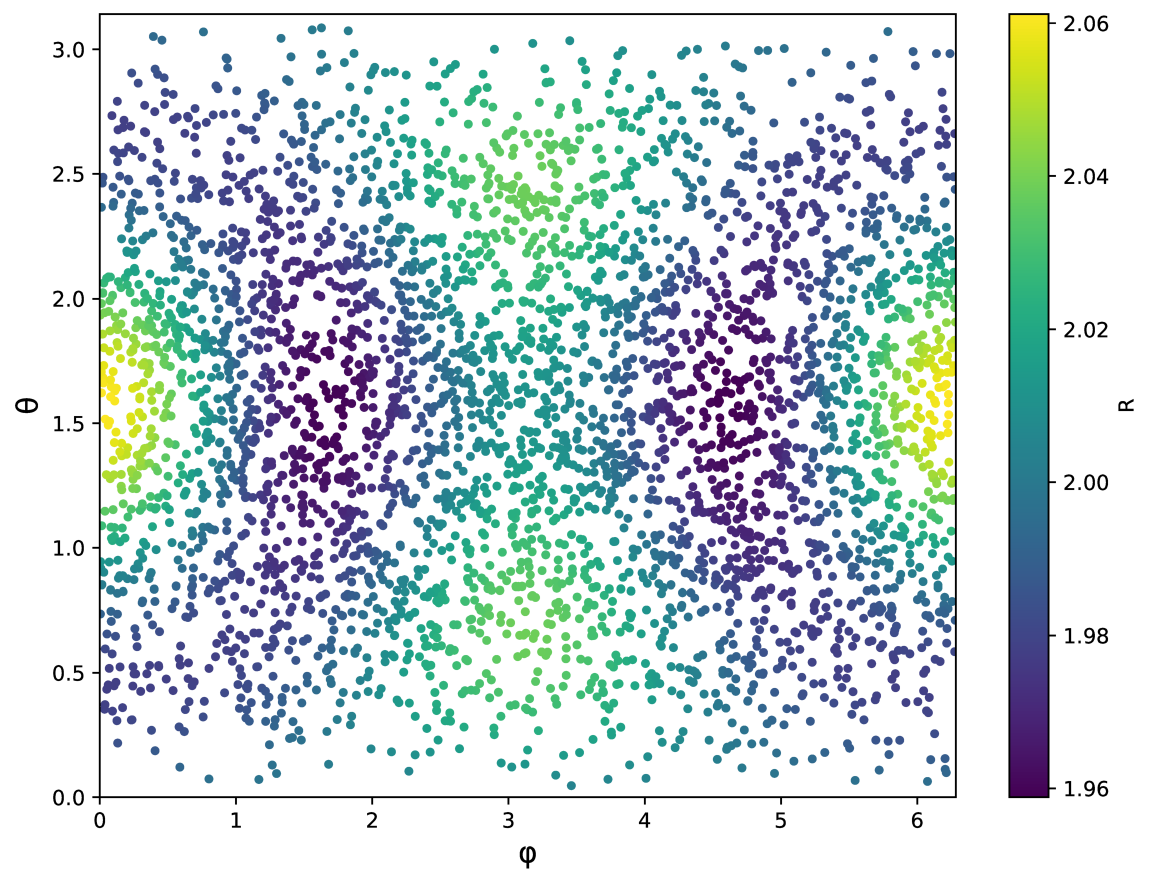}
        \caption{\texttt{prop\_b} $\tilde{R}$}
        \label{fig:shR_pb_sp}
    \end{subfigure}
    \hfill

    \vspace{0.5cm}

    \begin{subfigure}[t]{0.48\textwidth}
        \centering
        \includegraphics[width=0.9\linewidth]{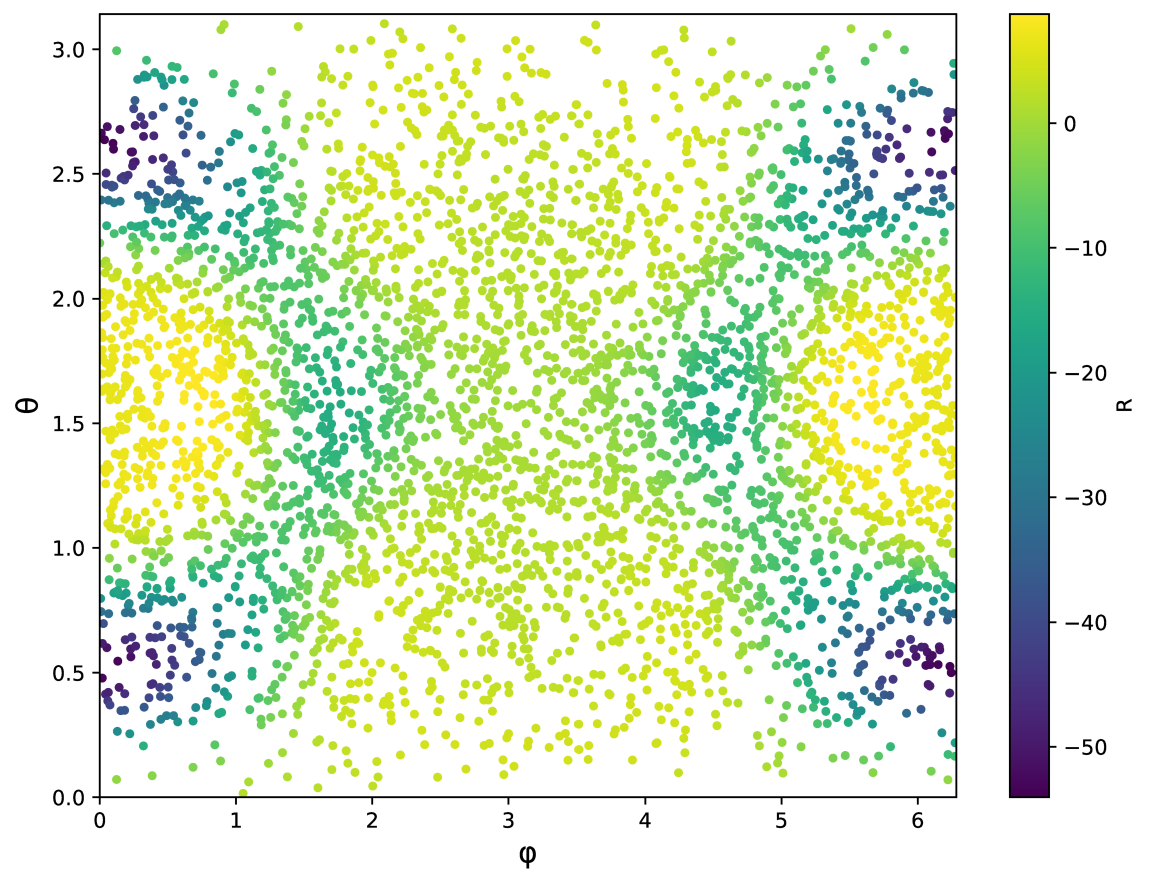}
        \caption{\texttt{prop\_c} $R$}
        \label{fig:shR_pc_p}
    \end{subfigure}
    \hfill
    \begin{subfigure}[t]{0.48\textwidth}
        \centering
        \includegraphics[width=0.9\linewidth]{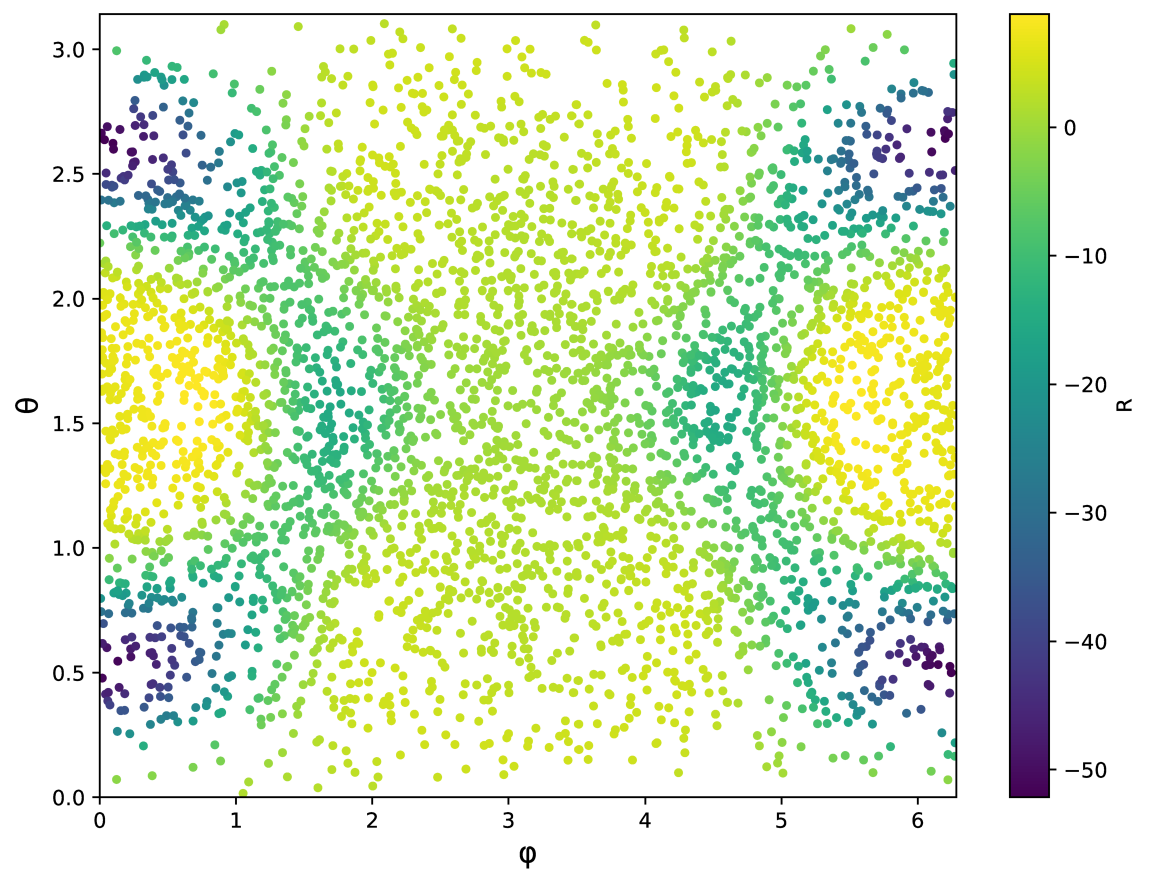}
        \caption{\texttt{prop\_c} $\tilde{R}$}
        \label{fig:shR_pc_sp}
    \end{subfigure}
    \caption{Plots of computed curvature over the fundamental domain of the sphere $(\theta, \phi)$, for spherical harmonic expansions fitted up to order $L = 4$ on trained NN models, for the Spectral Pair prescribers: \{\texttt{prop\_a}, \texttt{prop\_b}, \texttt{prop\_c}\}. The first column shows $R$ computed directly from the original prescriber function, whilst the second $\tilde{R}$ from the fitted expansion ansatz $\tilde{u}$.}
    \label{fig:shexpansion_Rs_props}
\end{figure*}

\begin{figure*}[ht!]
    \centering
    \begin{subfigure}[t]{0.48\textwidth}
        \centering
        \includegraphics[width=0.9\linewidth]{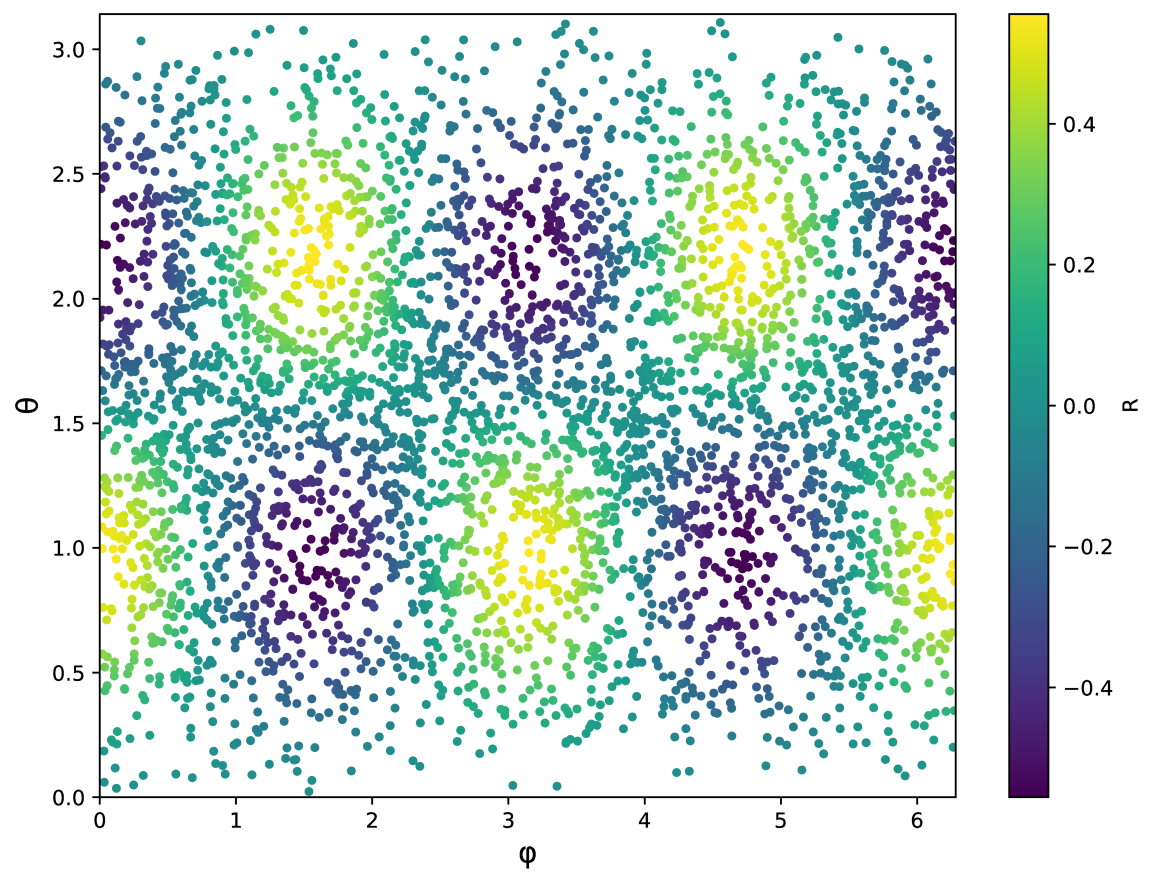}
        \caption{\texttt{sh\_3\_2} $R$}
        \label{fig:shR_sh_p}
    \end{subfigure}
    \hfill
    \begin{subfigure}[t]{0.48\textwidth}
        \centering
        \includegraphics[width=0.9\linewidth]{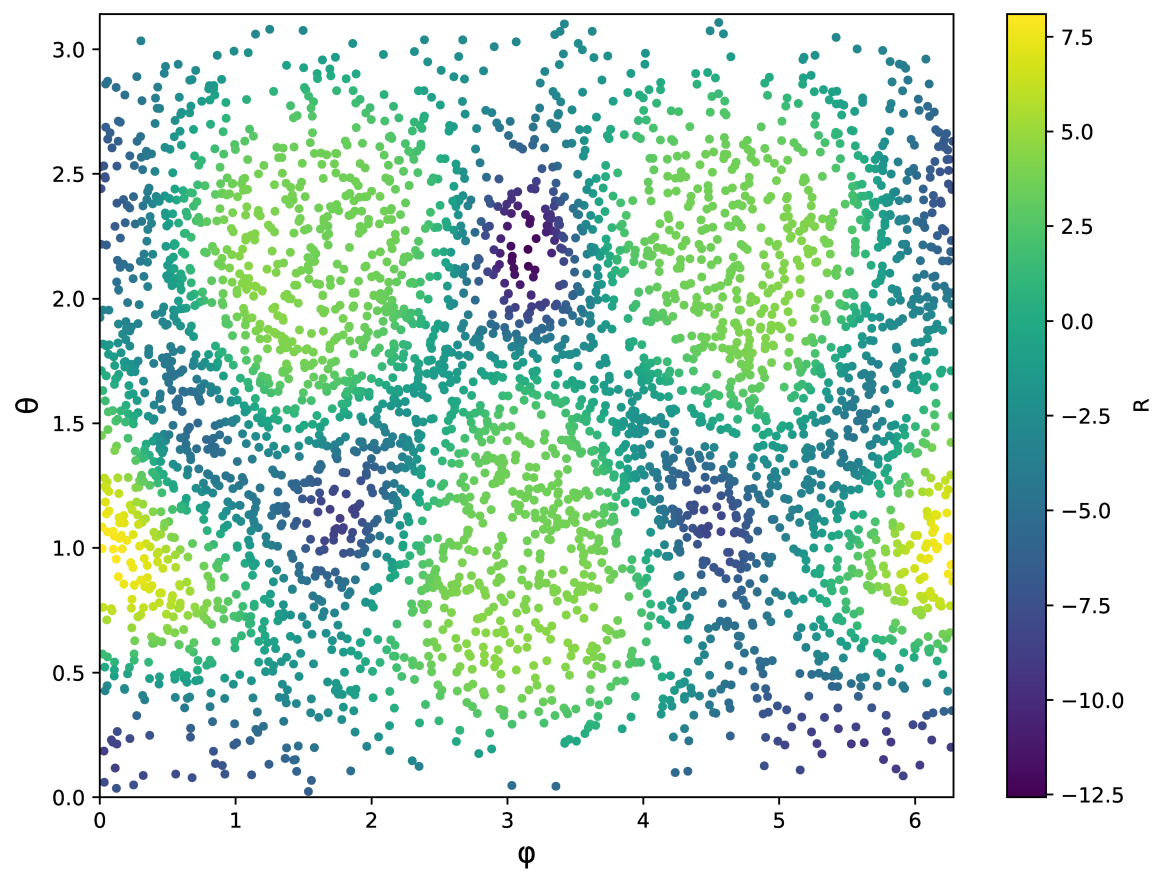}
        \caption{\texttt{sh\_3\_2} $\tilde{R}$}
        \label{fig:shR_sh_sp}
    \end{subfigure}

    \vspace{0.5cm}

    \begin{subfigure}[t]{0.48\textwidth}
        \centering
        \includegraphics[width=0.9\linewidth]{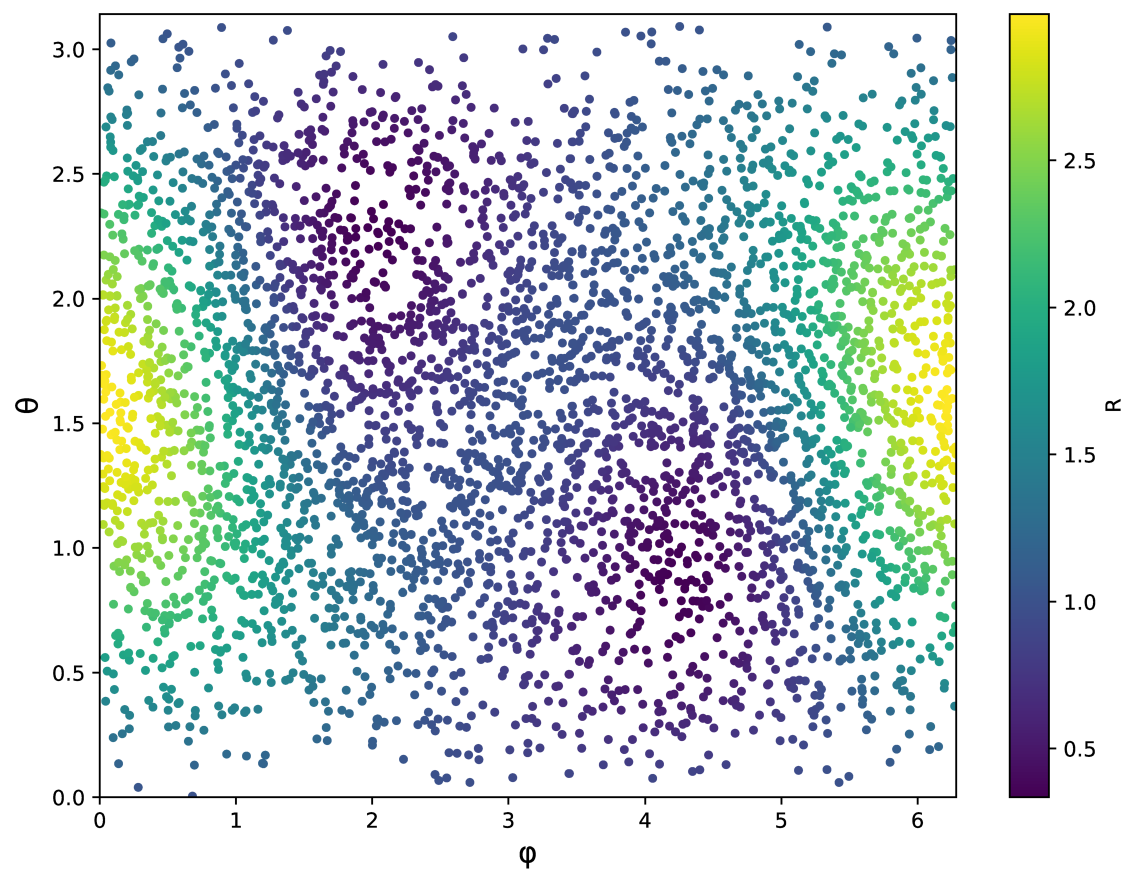}
        \caption{\texttt{x2\_plus\_yz\_plus\_x\_plus1} $R$}
        \label{fig:shR_xy_p}
    \end{subfigure}
    \hfill
    \begin{subfigure}[t]{0.48\textwidth}
        \centering
        \includegraphics[width=0.9\linewidth]{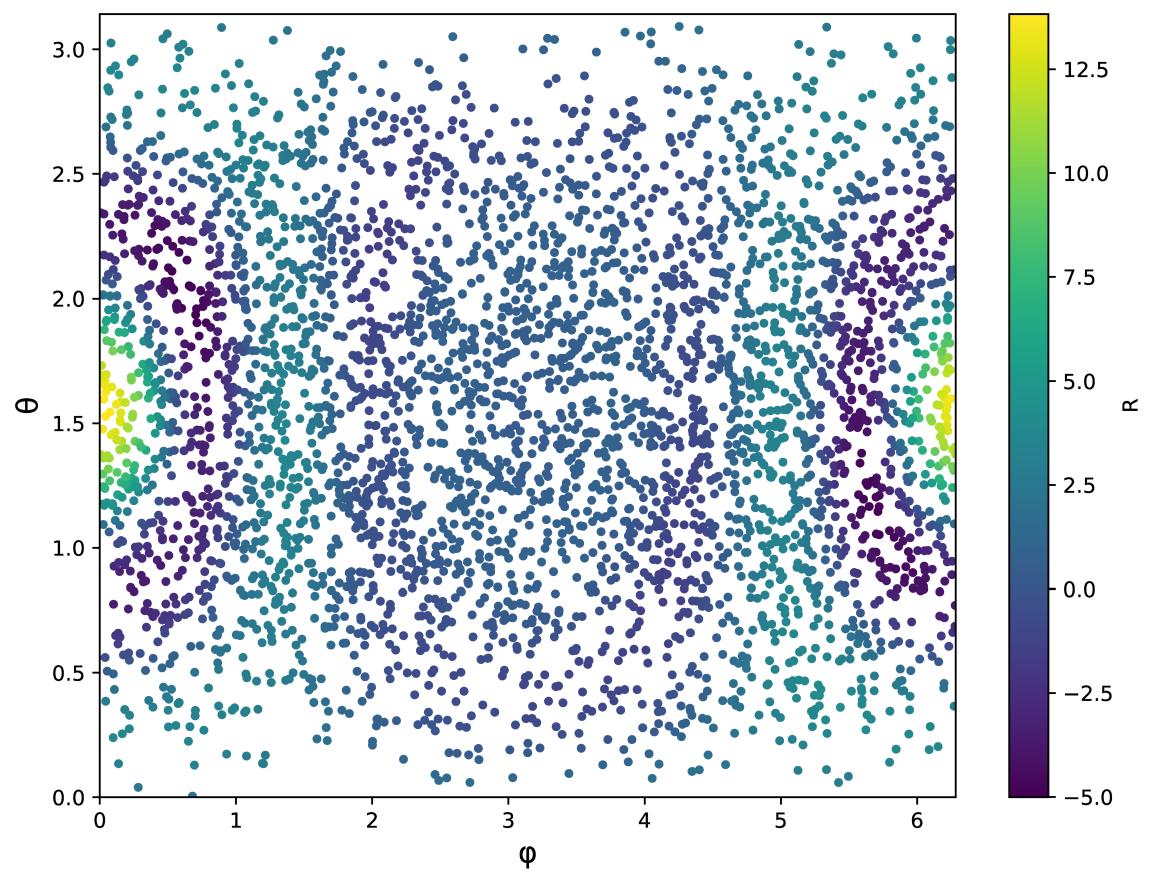}
        \caption{\texttt{x2\_plus\_yz\_plus\_x\_plus1} $\tilde{R}$}
        \label{fig:shR_xy_sp}
    \end{subfigure}
    \caption{Plots of computed curvature over the fundamental domain of the sphere $(\theta, \phi)$, for spherical harmonic expansions fitted up to order $L = 4$ on trained NN models, for selected prescribers with unknown realisability: \{\texttt{sh\_3\_2}, \texttt{x2\_plus\_yz\_plus\_x\_plus1}\}. The first column shows $R$ computed directly from the original prescriber function, whilst the second $\tilde{R}$ from the fitted expansion ansatz $\tilde{u}$.}
    \label{fig:shexpansion_Rs_unknown}
\end{figure*}

\clearpage

\section{Conclusions and Outlook}

\label{sec: conclusions}

This work presents a neural-architecture-based framework for numerical investigation of the Nirenberg problem of prescribed Gaussian curvature on $S^2$. By learning the conformal factor directly through a physics-informed loss function, the approach avoids traditional mesh-based discretisations and instead exploits automatic differentiation together with global geometric constraints to approximate solutions of the non-linear curvature equation.

The numerical experiments demonstrate that the network consistently distinguishes between prescribed curvature functions that are known to admit solutions and those that are known to be obstructed. In particular, the attained training loss, supplemented by independent diagnostic checks such as the Gauss-Bonnet identity, serves as a robust empirical indicator of solvability. When applied to families of curvature functions whose analytical status remains unresolved, the method exhibits a clear separation in achievable loss values (when augmented with the Gauss-Bonnet validation): functions that are likely to admit solutions converge to losses comparable with known solvable cases, while obstructed functions remain separated by several orders of magnitude. Although such observations do not constitute proofs, they provide strong numerical evidence to aid and prioritise future analytical investigations, potentially guiding towards fruitful conjectures and/or results. Moreover, the Nirenberg Neural Network shows its suitability for inclusion in a numerical pipeline leading to a (possibly) computer-assisted proof in the following aspects. Firstly, the loss achieved for some of the prescribers is only a few orders of magnitude from being at the level required by computer-assisted proof, whilst the error on the round metric already meets these requirements. Secondly, the investigation of neural network interpretability, through a fitted spherical harmonics expansion, suggests that amenable closed-form expansions of the conformal factor might be distilled from the model, providing analytical formulae over which one has full control.

More broadly, these results suggest that neural networks may function as useful exploratory tools in geometric analysis, complementing rather than replacing classical techniques. This perspective is particularly natural for the Nirenberg problem, where local partial differential equation behaviour is tightly coupled to global topological and variational constraints. The ability of the network to implicitly respect these structures highlights the potential of machine-learning-based methods for probing the global landscape of non-linear geometric equations.

Several directions for further work arise naturally. From an analytical perspective, it would be of interest to investigate whether the empirical loss thresholds observed here admit interpretations in terms of degree-theoretic or variational obstructions. From a computational viewpoint, natural extensions include higher-dimensional prescribed curvature problems, alternative background metrics, and related geometric equations such as the Yamabe or $\sigma_k$-curvature problems. The incorporation of symmetry-aware or equivariant architectures may also improve both efficiency and interpretability. In terms of improving the current set-up, there are also a number of natural steps ahead. For example, integrating understanding of the group Conf($S^2$) and its role in the Nirenberg problem in the codebase might lead to significant improvements. Moreover, the investigation can be naturally scaled-up in two aspects. On one side an intention to process more data; incorporating more test functions, gaining a better understanding on the regime of applicability of the method; and collecting further corroboration. On the other side, plans to make use of more compute, i.e.~larger networks, longer training, and more Fourier features. These improvements might lead to an increase in performance, yielding results suitable for computer-assisted proofs.

It is hoped that this package will help researchers develop a deeper understanding of the Nirenberg problem, as well as generate and test conjectures in a systematic, computational way. Further, this work hopes to encourage greater interaction between geometric analysis and machine learning, and to promote the data–driven methods at the interface of geometry, analysis, and computation.

\section*{Data Availability Statement}

The models and code used to generate the data and results in this work are available on GitHub at: \url{https://github.com/xand-stapleton/nirenberg-neural-network}. 

\section*{Acknowledgements}
AGS, EH, and TSG acknowledge support from Pierre Andurand over various stages of this research, and wish to thank David Berman for useful discussions on earlier incarnations of this work. TSG is supported by the 2024 Max Planck-Humboldt Research Award bestowed on Geordie Williamson and Catharina Stroppel by the Max Planck Society and the Alexander von Humboldt Foundation.
EH is supported by S\~ao Paulo Research Foundation (FAPESP) grant 2024/18994-7.  
EH thanks Daniel Platt and Andre Lukas for discussions and suggestions surrounding the learnt spherical harmonic expansions.
YF thanks Michael Douglas for helpful discussions.
MEC would like to thank Erik Bekkers for helpful discussions. MEC is supported by the Vidi research programme through the project SIGN (file number VI.Vidi.233.220), which is partly financed by the Dutch Research Council (NWO) under grant \url{https://doi.org/10.61686/PKQGZ71565}. 
All authors also wish to thank Jordan Marajh for his insight and discussions during the early stages of this work.  

This research utilised Queen Mary's Apocrita HPC facility \cite{apocrita}, supported by QMUL Research-IT.

\appendix

\section{Embedding Visualisations}
In the Nirenberg problem, the $S^2$ choice sets the topology of space, but the geometry is defined by the learnt metric which expresses the desired scalar curvature.

For the realisable, and likely realisable, prescribed scalar curvature functions the trained NN architectures model the respective $u$ conformal factor, and hence provide a respective metric $g = e^{2u}g_{round}$ with this curvature.
This metric can be used to compute local distances on the space which can inform on an embedding visualisation for the distortion of the round sphere ($g_{round}$) caused by this learnt conformal factor ($u$).

Here, visualisations for the learnt metrics with these prescribers are embedded in a geometry-preserving manner.
Firstly, a mesh is generated on the sphere, then the chosen trained NN model is imported and used to predict pairwise distances on this mesh.
These pairwise distances are then passed to the Multidimensional Scaling (MDS) method, an unsupervised machine learning tehnique \cite{borg2005modern}. 
MDS computes a low-dimensional embedding of data by finding coordinates whose pairwise Euclidean distances best approximate a given matrix of dissimilarities, minimizing a stress function. 
This produces a resulting embedding which preserves the relative geometry of the original distance relationships as closely as possible in a least-squares sense.
Outputting an embedding for the mesh which can be visualised in 3D.

Figure's \ref{fig:embeddings_1} and \ref{fig:embeddings_2} show these MDS embeddings of the learnt metrics for each considered prescribed curvature.
Each are realisable, with the final 3 of Figure \ref{fig:embeddings_2} the likely realisable prescribed curvatures as suggested by the Nirenberg NN results.
These embeddings provide some unique visual insight into these learnt geometries.

\begin{figure*}[t]
    \centering
    \begin{subfigure}[t]{0.32\textwidth}
        \centering
        \includegraphics[width=\linewidth]{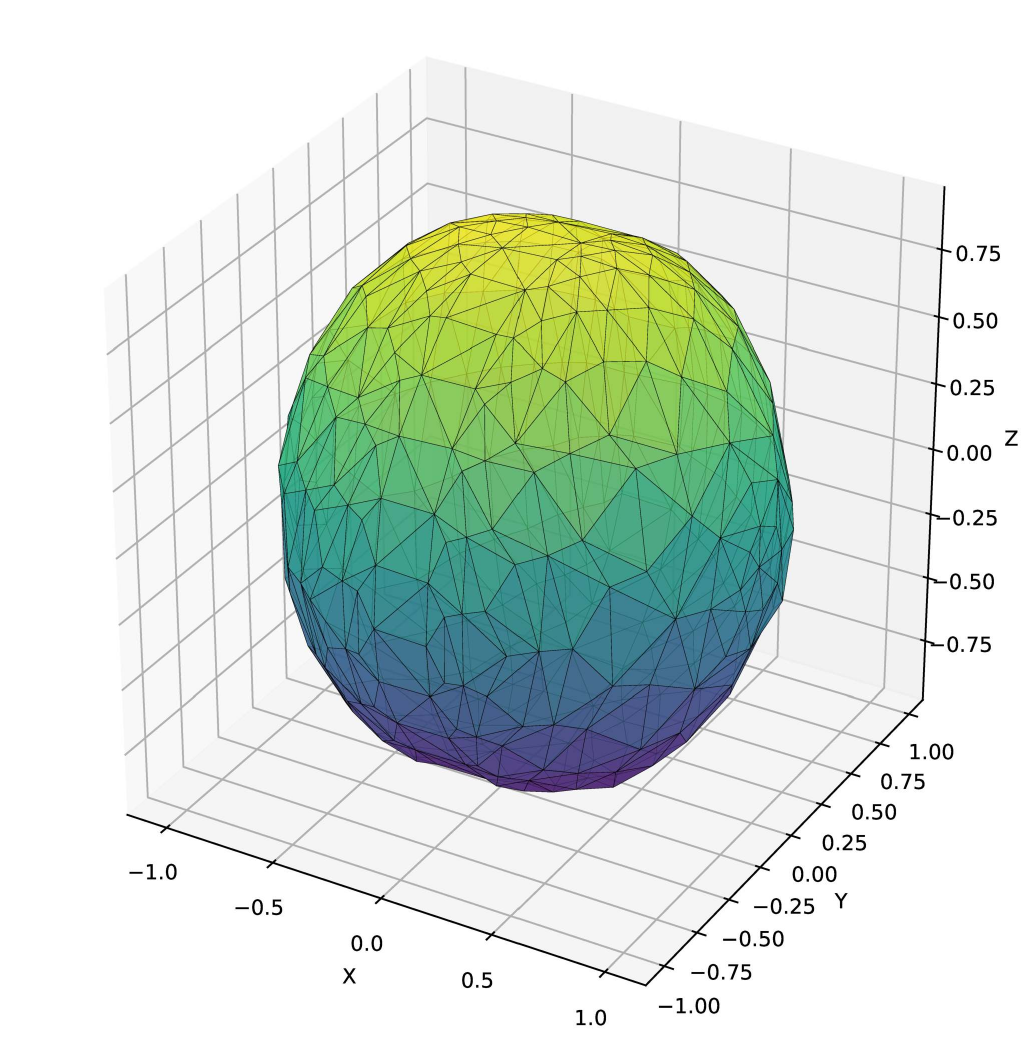}
        \caption{\texttt{prop\_a}}
        \label{fig:embed_pa}
    \end{subfigure}
    \hfill
    \begin{subfigure}[t]{0.32\textwidth}
        \centering
        \includegraphics[width=\linewidth]{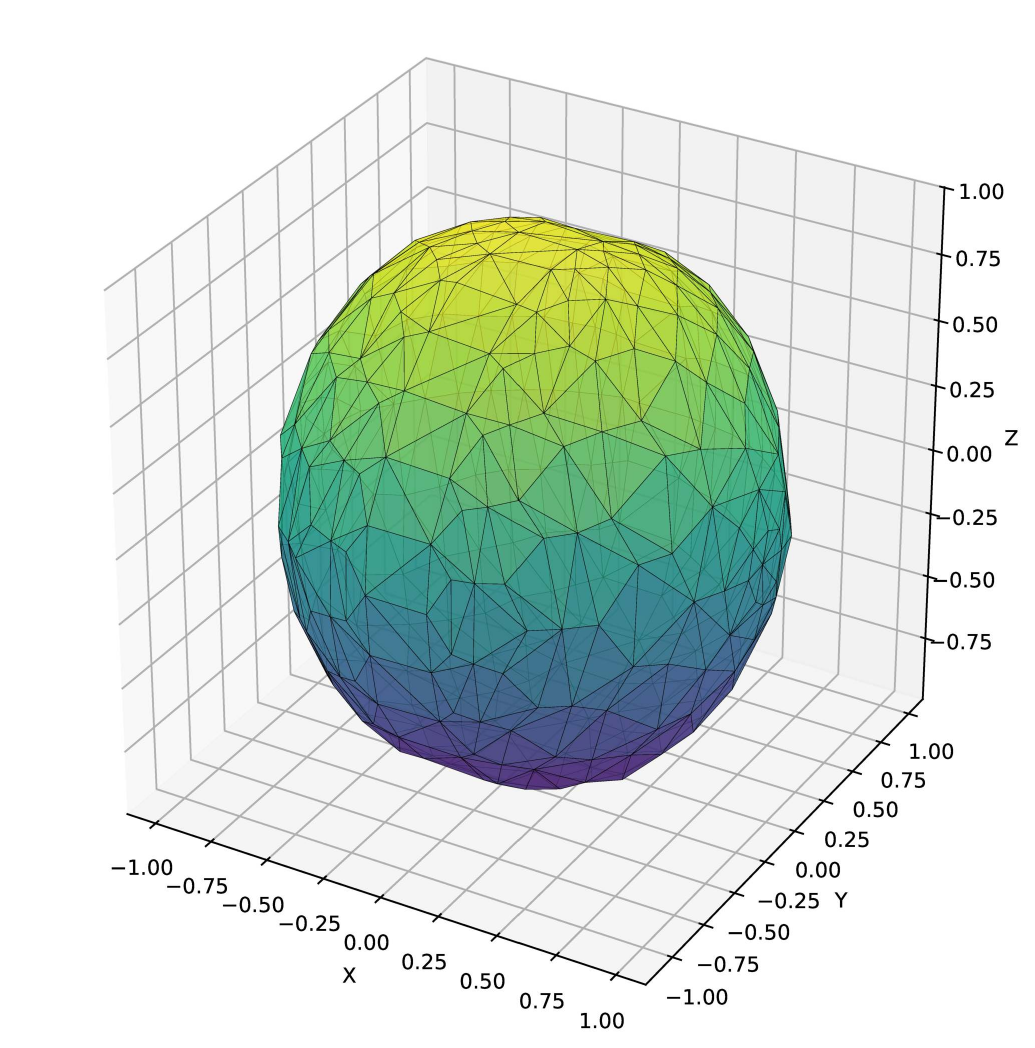}
        \caption{\texttt{prop\_b}}
        \label{fig:embed_pb}
    \end{subfigure}
    \hfill
    \begin{subfigure}[t]{0.32\textwidth}
        \centering
        \includegraphics[width=\linewidth]{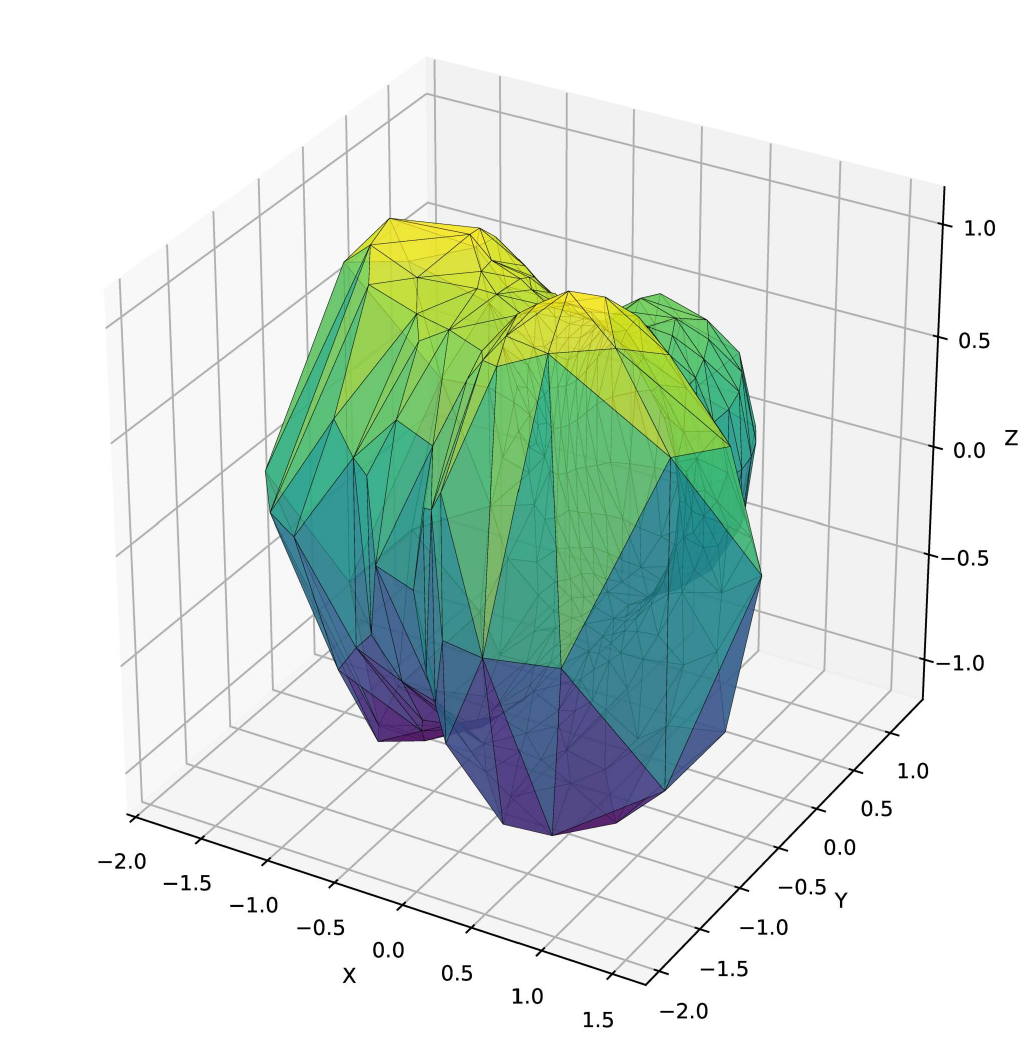}
        \caption{\texttt{prop\_c}}
        \label{fig:embed_pc}
    \end{subfigure}

    \vspace{0.5em}

    \begin{subfigure}[t]{0.32\textwidth}
        \centering
        \includegraphics[width=\linewidth]{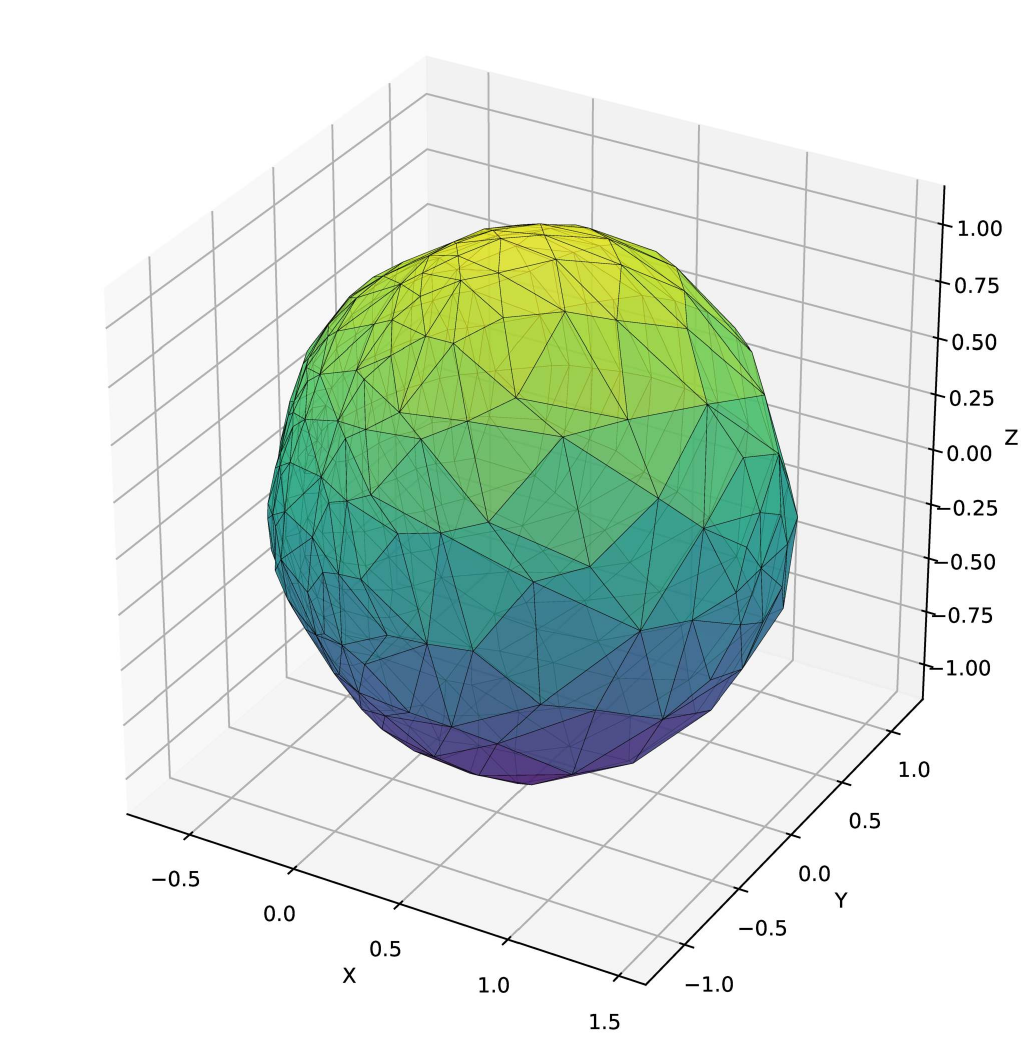}
        \caption{\texttt{egg}}
        \label{fig:embed_egg}
    \end{subfigure}
    \hfill
    \begin{subfigure}[t]{0.32\textwidth}
        \centering
        \includegraphics[width=\linewidth]{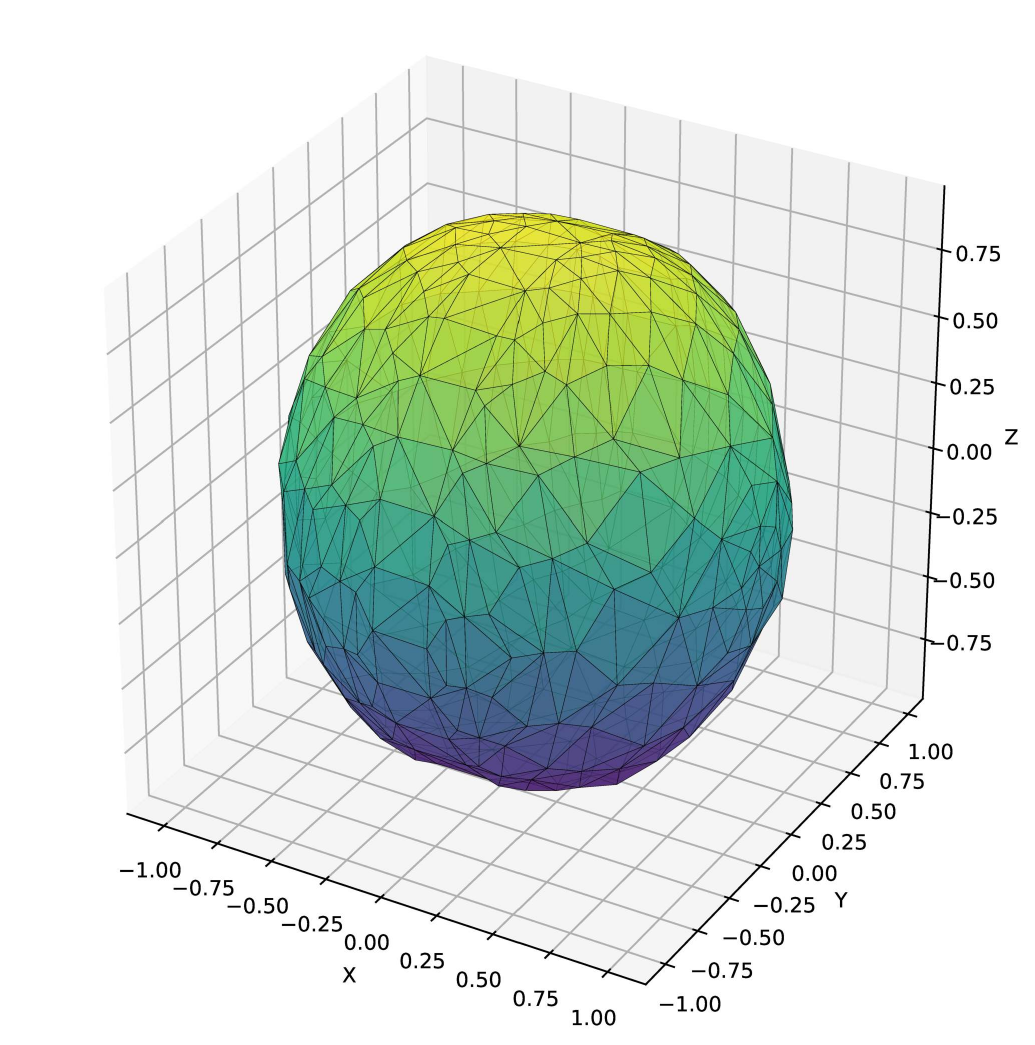}
        \caption{\texttt{cosh\_profile}}
        \label{fig:embed_cosh}
    \end{subfigure}
    \hfill
    \begin{subfigure}[t]{0.32\textwidth}
        \centering
        \includegraphics[width=\linewidth]{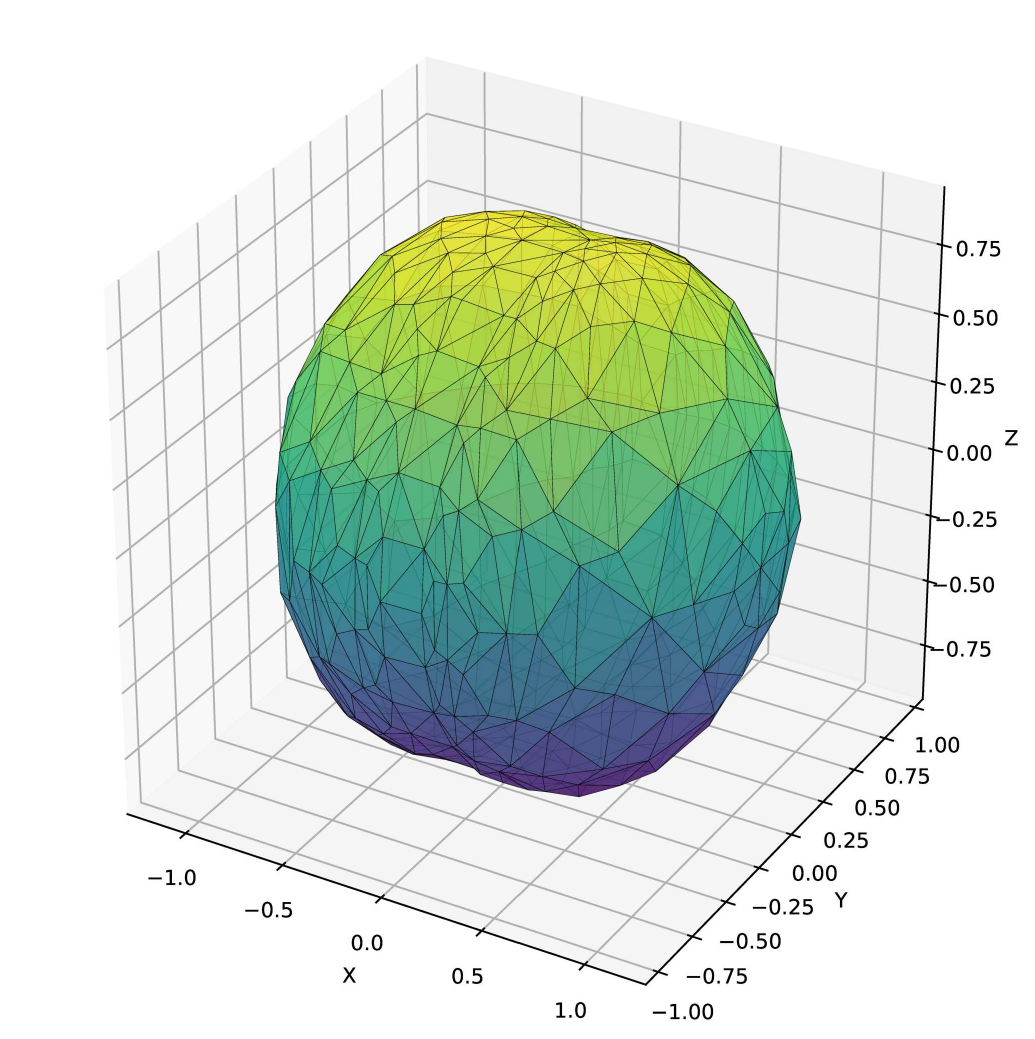}
        \caption{\texttt{tanh\_wave}}
        \label{fig:embed_tanh}
    \end{subfigure}
    \caption{MDS embeddings of the learnt metrics for each respective prescribed curvature. All 6 here are known to be realisable.}
    \label{fig:embeddings_1}
\end{figure*}

\begin{figure*}[t]
    \centering
    \begin{subfigure}[t]{0.32\textwidth}
        \centering
        \includegraphics[width=\linewidth]{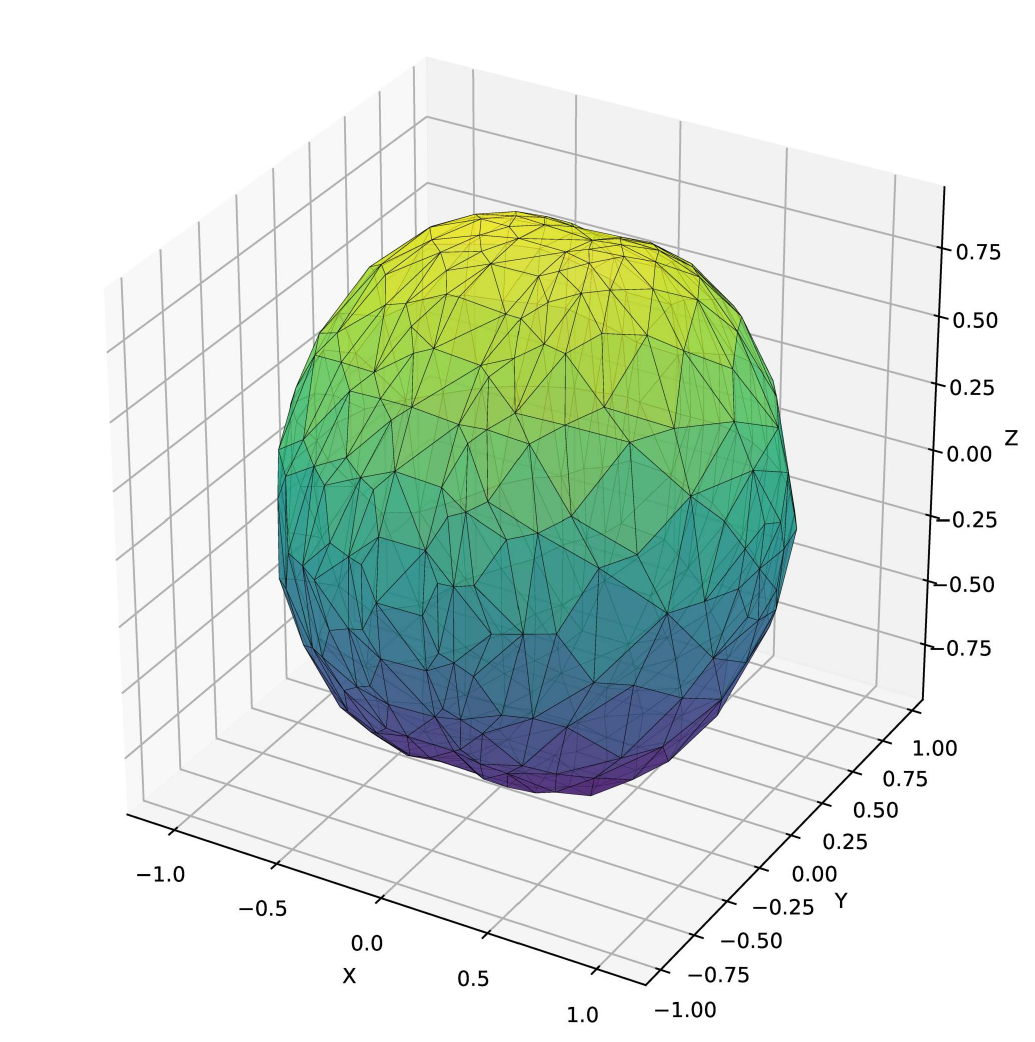}
        \caption{\texttt{sinusoidal}}
        \label{fig:embed_sinusoidal}
    \end{subfigure}
    \hfill
    \begin{subfigure}[t]{0.32\textwidth}
        \centering
        \includegraphics[width=\linewidth]{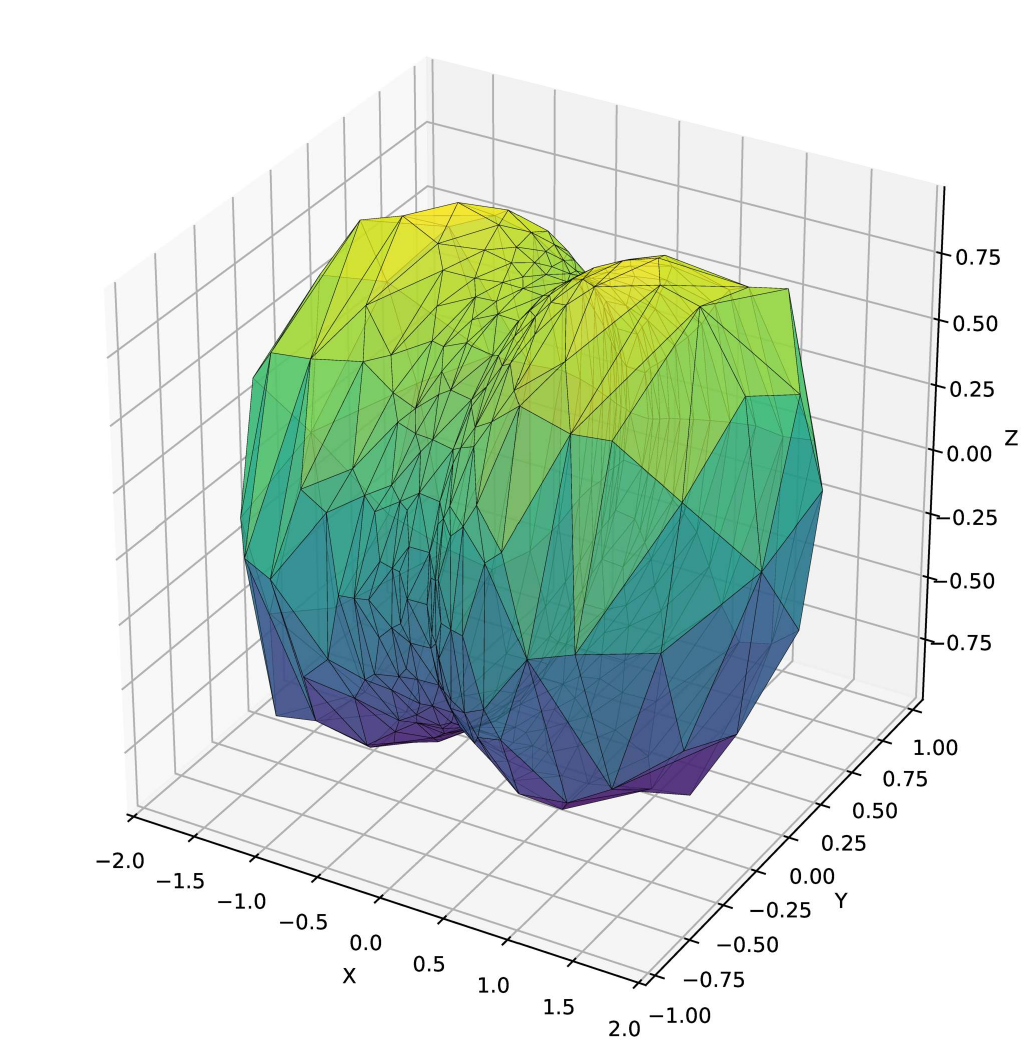}
        \caption{\texttt{sh\_2\_0}}
        \label{fig:embed_sh20}
    \end{subfigure}
    \hfill
    \begin{subfigure}[t]{0.32\textwidth}
        \centering
        \includegraphics[width=\linewidth]{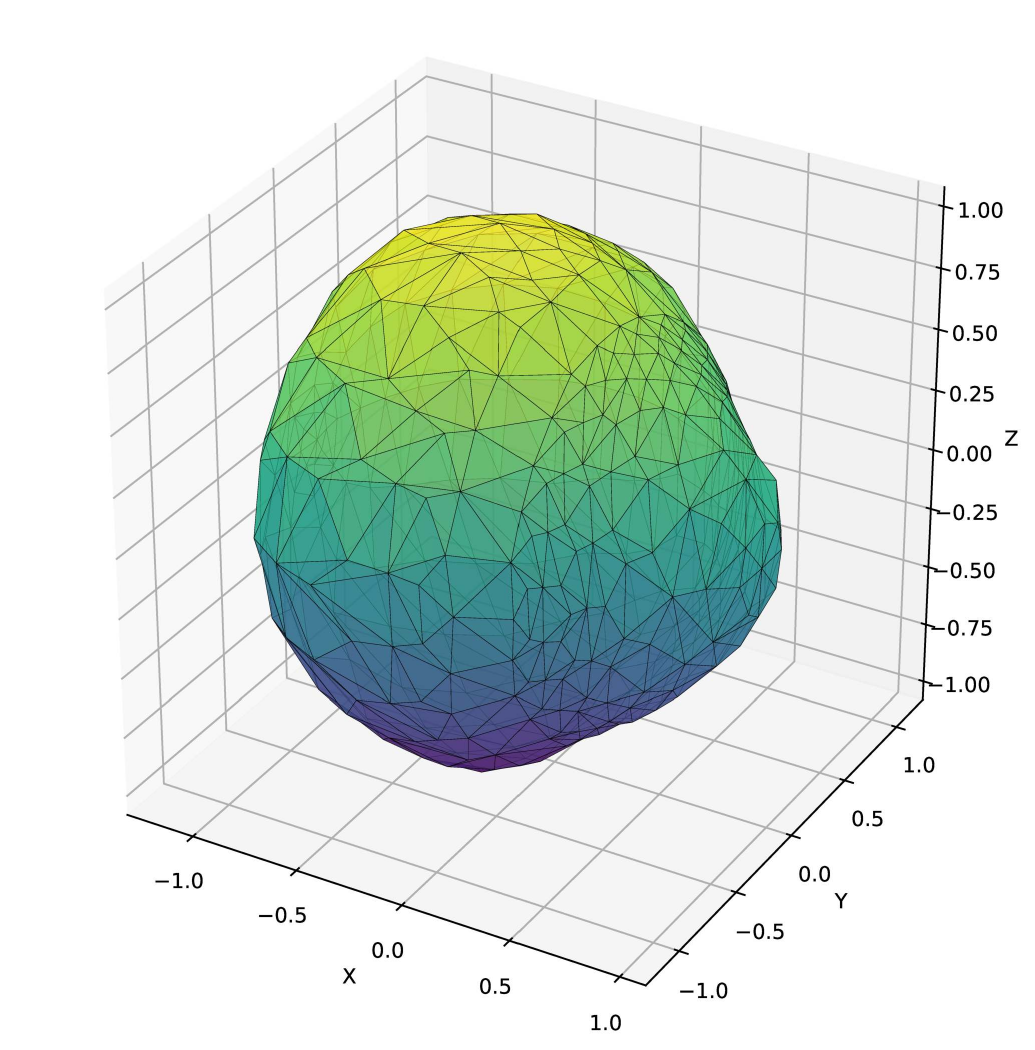}
        \caption{\texttt{sh\_3\_0}}
    \end{subfigure}

    \vspace{0.5em}

    \begin{subfigure}[t]{0.32\textwidth}
        \centering
        \includegraphics[width=\linewidth]{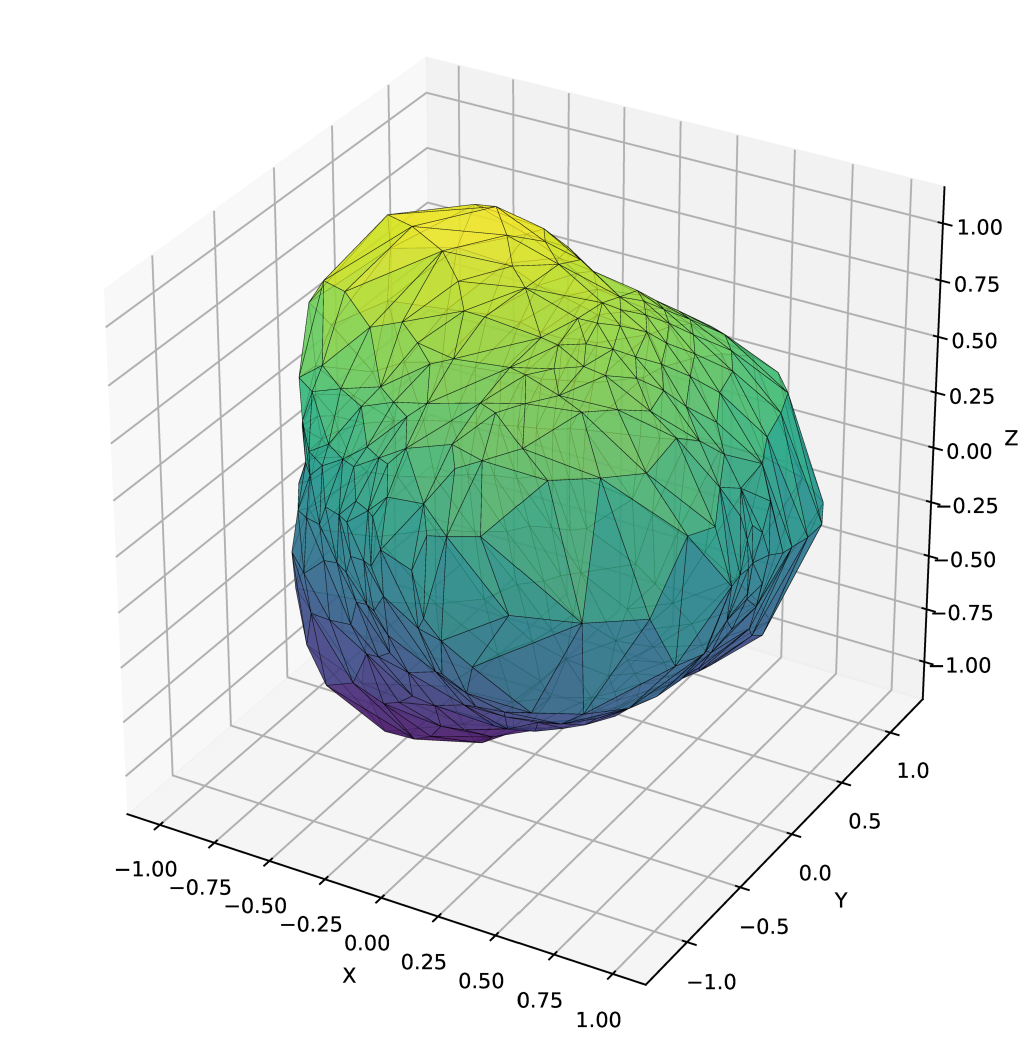}
        \caption{\texttt{sh\_3\_2}}
        \label{fig:embed_sh32}
    \end{subfigure}
    \hfill
    \begin{subfigure}[t]{0.32\textwidth}
        \centering
        \includegraphics[width=\linewidth]{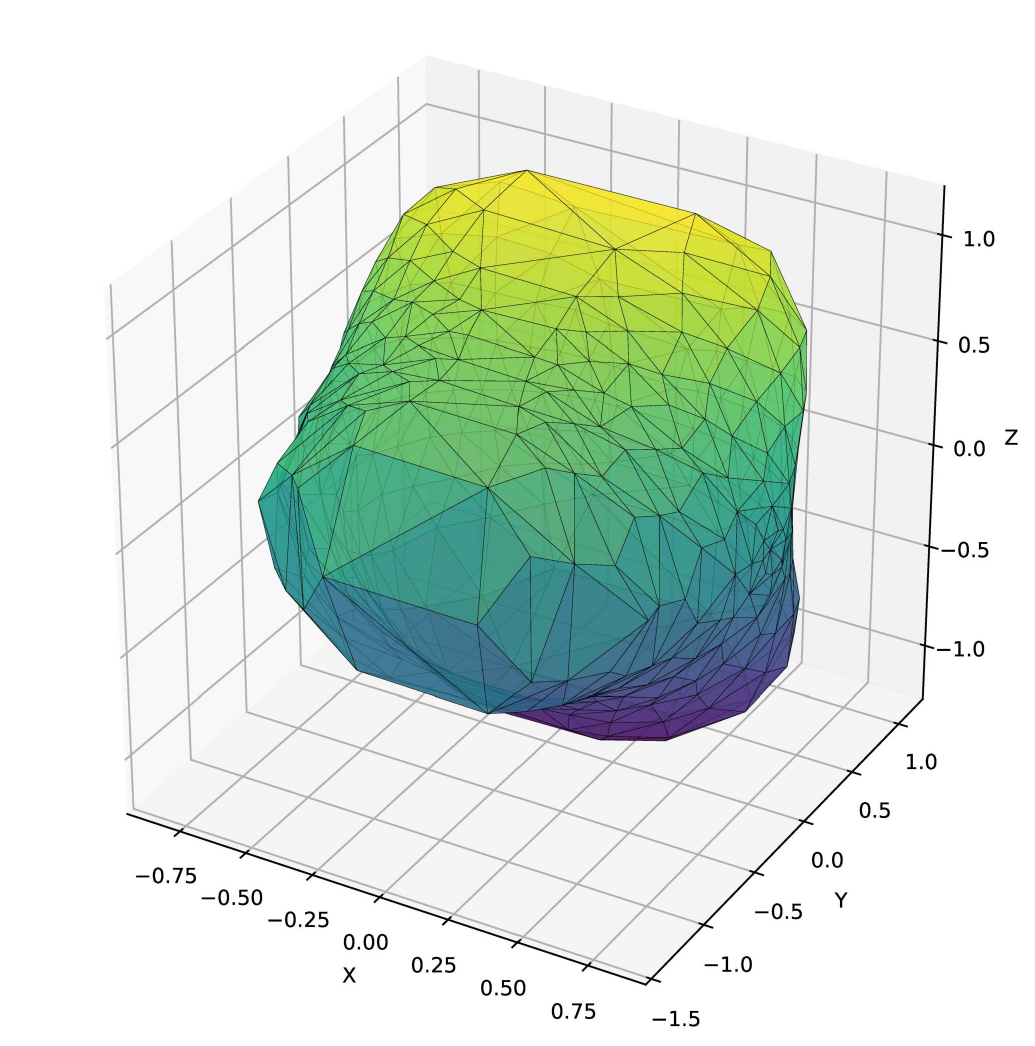}
        \caption{\texttt{sh\_3\_3}}
        \label{fig:embed_sh33}
    \end{subfigure}
    \hfill
    \begin{subfigure}[t]{0.32\textwidth}
        \centering
        \includegraphics[width=\linewidth]{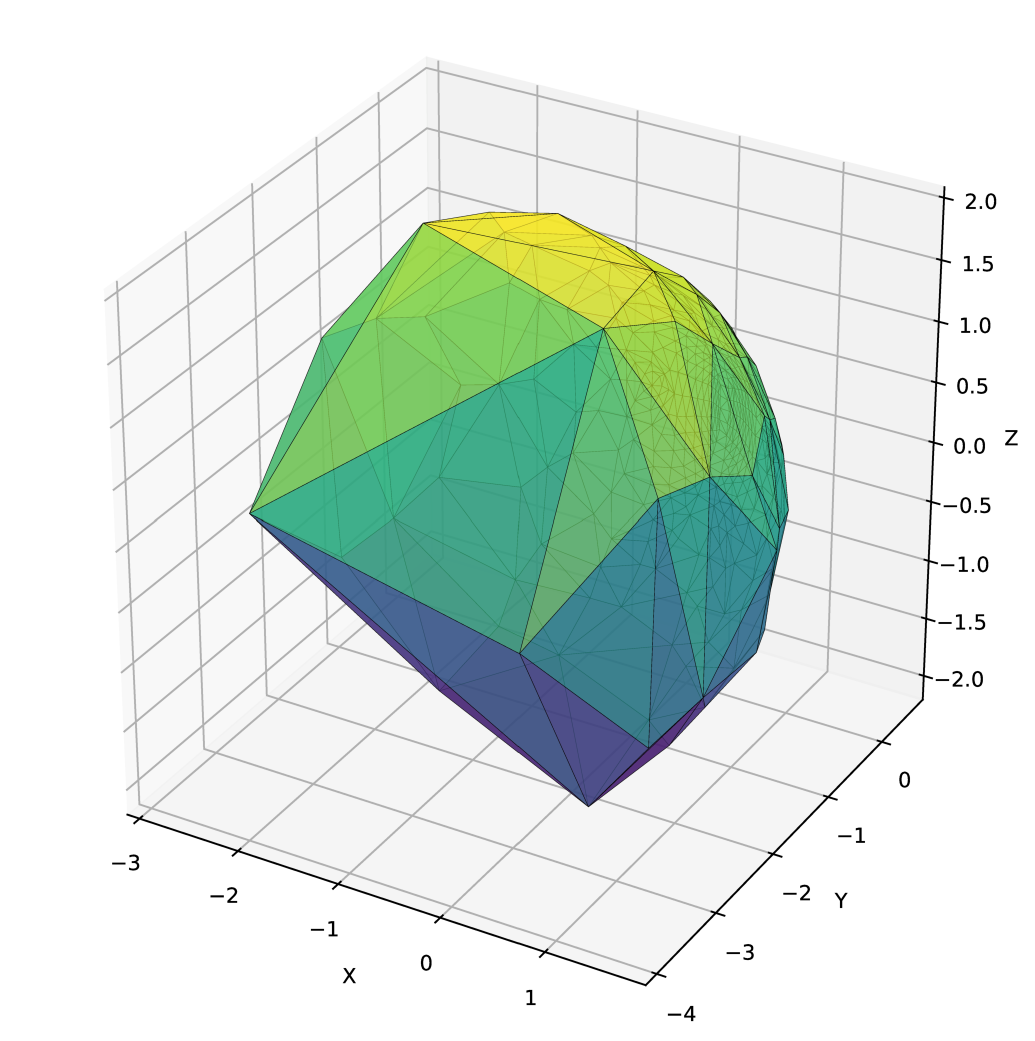}
        \caption{\texttt{x2\_plus\_yz\_plus\_x\_plus1}}
        \label{fig:embed_xyform}
    \end{subfigure}
    \caption{MDS embeddings of the learnt metrics for each respective prescribed curvature. The top row are known to be realisable, the bottom row are unknown but are likely realisable due to the NN learning results.}
    \label{fig:embeddings_2}
\end{figure*}

\newpage
\section{Network Training}

\subsection{Training Pipeline}
\label{sub:training_pipeline}

The training workflow begins with configuration specification through YAML files that define all hyperparameters including network architecture (layer widths, depths, activation functions, random Fourier feature settings), data generation parameters (number of patches, sampling density, radial offset), loss function weights, and optimization settings (learning rate, schedule, batch size, number of epochs). This configuration-based approach enables systematic exploration of the hyperparameter space and reproducible experiments.

A sampler generates training points on the sphere as described earlier, evaluating the prescribed curvature at each point and computing the normalisation constant. These data are assembled into batches using TensorFlow's data pipeline infrastructure, with shuffling to decorrelate consecutive batches and prefetching to overlap data loading with computation. The dataset is sized to provide tens of thousands of training points per epoch, ensuring dense coverage of the sphere.

Training proceeds using the Adam optimizer, with a cosine annealing scheduler, which adapts learning rates for each parameter based on first and second moment estimates of the gradients, and the index of the training step. The learning rate schedule reduces the learning rate while allowing periodic increases to escape local minima. Batch sizes typically range from 100 to 200 points sampled on the sphere. During training, metrics are logged to Weights \& Biases for experiment tracking, including the loss value, learning rate, and any auxiliary quantities of interest. Checkpoints are saved periodically, allowing training to be resumed or analysed at intermediate stages.

Upon completion, the final trained model is serialised and saved along with its configuration. This allows the trained conformal factor to be loaded later for evaluation, visualisation, or fine-tuning on modified objective functions.

\subsection{Training Parameters}
\label{sub:training_parameters}

Table \ref{tab:training_params} provides an overview of the parameters used to train the Nirenberg Neural Network.

\begin{table}[h]
\centering
\renewcommand{\arraystretch}{1.15}
\begin{tabular}{ll}
\toprule

\multicolumn{2}{c}{\textbf{Data Configuration}} \\
\midrule
Number of Patches & 2 \\
Number of Samples & 20\,000 \\
Batch Size & 200 \\

\addlinespace[6pt]
\multicolumn{2}{c}{\textbf{Network Architecture}} \\
\midrule
Hidden Units per Layer & 64 \\
Number of Layers & 6 \\
Activation Function & SiLU \\
Use Bias & True \\
Number of RFFs & 32 or 128\footnote{Equal numbers of both RFFs were used}\\
RFF Bandwidth ($\sigma$) & 1.0 \\

\addlinespace[6pt]
\multicolumn{2}{c}{\textbf{Optimization}} \\
\midrule
Epochs & 150 \\
Learning Rate & $1 \times 10^{-3}$ \\

\addlinespace[6pt]
\multicolumn{2}{c}{\textbf{Scheduler (Exponential Decay)}} \\
\midrule
Decay Steps & 2000 \\
Decay Rate & 0.5 \\
Staircase Decay & False \\
\bottomrule
\end{tabular}
\caption{Training configuration.}
\label{tab:training_params}
\end{table}

\clearpage
\bibliographystyle{JHEP}
\bibliography{Bibliography}{}

\end{document}